\documentclass[final,onefignum,onetabnum]{siamonline220329}

\usepackage{lipsum}
\usepackage{comment}
\usepackage{epstopdf}
\ifpdf
  \DeclareGraphicsExtensions{.pdf,.png,.jpg}
\fi

\usepackage{braket,amsfonts}

\usepackage{array}

\newsiamremark{remark}{Remark}
\newsiamremark{assumption}{Assumption}
\newsiamremark{hypothesis}{Hypothesis}
\crefname{hypothesis}{Hypothesis}{Hypotheses}
\newsiamthm{claim}{Claim}

\usepackage{algorithmic}

\usepackage{graphicx,epstopdf}

\Crefname{ALC@unique}{Line}{Lines}

\usepackage{amsopn}

\usepackage{enumitem}
\setlist[enumerate]{leftmargin=.5in}
\setlist[itemize]{leftmargin=.5in}

\usepackage{refcount}

\usepackage{afterpage}
\usepackage{mathrsfs}
\usepackage{bm}
\usepackage{mathtools}
\usepackage{appendix}
\usepackage{hyperref}
\usepackage{amsmath}
\usepackage{amssymb}
\usepackage{latexsym}
\usepackage{pdfsync}
\usepackage{subfigure}
\usepackage{color}
\usepackage{tabularx}
\usepackage[english]{babel}
\usepackage{algorithm}
\usepackage{esint}
\usepackage{environ}

\usepackage[symbol]{footmisc}

\usepackage{enumerate}

\usepackage{bbm}%
\usepackage{color}              %
\usepackage{xcolor}
\usepackage{tabularx}
\usepackage{makecell}
\usepackage{booktabs} 
\usepackage{bbold}

\newcommand{\vertiii}[1]{{\left\vert\kern-0.25ex\left\vert\kern-0.25ex\left\vert #1 
    \right\vert\kern-0.25ex\right\vert\kern-0.25ex\right\vert}}

\newcommand{\mres}{\mathbin{\vrule height 1.2ex depth 0pt width
0.13ex\vrule height 0.13ex depth 0pt width 0.9ex}}

\newcommand{\mbf}[1]{\boldsymbol{#1}}

\newcommand{\inp}[1]{\langle{#1}\rangle}

\newcommand{\norm}[1]{\left\| #1 \right\|}

\newcommand{\real}{\mathbb{R}}

\newcommand{\bv}{{\mbf{v}}}

\newcommand{\bx}{{\mbf{x}}}
\newcommand{\bX}{\mbf{X}}
\newcommand{\bbX}{\mathbb{X}}

\newcommand{\bY}{\mbf{Y}} %
\newcommand{\by}{\mbf{y}}
\newcommand{\bby}{\mathbb{Y}}
\newcommand{\bbY}{\mathbb{Y}}
\newcommand{\bbV}{\mathbb{V}}

\newcommand{\balpha}{\mbf{\alpha}}
\newcommand{\bgamma}{\mbf{\gamma}}

\newcommand{\mE}{\mathcal{E}}

\newcommand{\mH}{\mathcal{H}_{{K}}}
\newcommand{\mHe}{\mathcal{H}_{{K}^E}}
\newcommand{\mHa}{\mathcal{H}_{{K}^A}}

\newcommand{\mK}{{K}}
\newcommand{\mKe}{{K}^{E}}
\newcommand{\mKa}{{K}^{A}}

\newcommand{\bZ}{\mbf{Z}}
\newcommand{\bbZ}{\mathbb{Z}}

\newcommand{\intkernele}{{\intkernel^{E}}}
\newcommand{\intkernela}{{\intkernel^{A}}}

\newcommand{\intkernelvare}{\varphi^{E}}
\newcommand{\intkernelvara}{\varphi^{A}}

\newcommand{\bV}{\mbf{V}}

 \newcommand{\phiH}{\intkernel_{\mathcal{H}_{K^E \times K^A}}}

 \newcommand{\thetae}{{\theta^E}}
 \newcommand{\thetaa}{{\theta^A}}

\newcommand{\force}{F}
\newcommand{\forcev}{\force^{\bv}}

\newcommand{\intkernel}{\phi}

\newcommand{\bintkernel}{{\bm{\phi}}}

\newcommand{\intkernelvar}{\varphi}
\newcommand{\bintkernelvar}{{\bm{\varphi}}}

\newcommand{\rhsfo}{\mathbf{f}}

\newcommand{\cov}{\mathrm{Cov}}

\newcommand{\argmin}[1]{\underset{#1}{\operatorname{arg}\operatorname{min}}\;}

\newcommand{\JF}[2]{#2}

\renewcommand{\algorithmicrequire}{\textbf{Input:}}    
\renewcommand{\algorithmicensure}{\textbf{Output:}}

\DeclareMathAlphabet{\mathpzc}{OT1}{pzc}{m}{it}

\headers{Data-Driven Model Selection for Second-Order Particle Based Dynamics}{Jinchao Feng, Charles Kulick, Sui Tang}

\title{ Data-Driven Model Selections of Second-Order Particle Dynamics via Integrating Gaussian Processes with Low-Dimensional Interacting Structures \thanks{This work was partially supported by NSF DMS-2111303.}}

\author{Jinchao Feng\thanks{\JF{}{School of Sciences, Great Bay University, Dongguan, Guangdong, China
  (\email{jcfeng@gbu.edu.cn}).}}
\and Charles Kulick\thanks{Departments of Mathematics, University of California, Santa Barbara, Isla Vista, CA
  (\email{charles@math.ucsb.edu}).}
\and Sui Tang\thanks{Departments of Mathematics, University of California, Santa Barbara, Isla Vista, CA
  (\email{suitang@math.ucsb.edu}).}}

\usepackage{amsopn}

\begin{document}

\maketitle

\begin{abstract}

In this paper, we focus on the data-driven discovery of a general second-order particle-based model that contains many state-of-the-art models for modeling the aggregation and collective behavior of interacting agents of similar size and body type. This model takes the form of a high-dimensional system of ordinary differential equations parameterized by two interaction kernels that appraise the alignment of positions and velocities. We propose a Gaussian Process-based approach to this problem, where the unknown model parameters are marginalized by using two independent Gaussian Process (GP) priors on latent interaction kernels constrained to dynamics and observational data. This results in a nonparametric model for interacting dynamical systems that accounts for uncertainty quantification. We also develop acceleration techniques to improve scalability. Moreover, we perform a theoretical analysis to interpret the methodology and investigate the conditions under which the kernels can be recovered. We demonstrate the effectiveness of the proposed approach on various prototype systems, including the selection of the order of the systems and the types of interactions. In particular, we present applications to modeling two real-world fish motion datasets that display flocking and milling patterns up to 248 dimensions. Despite the use of small data sets, the GP-based approach learns an effective representation of the nonlinear dynamics in these spaces and outperforms competitor methods.

\end{abstract}


\section{Introduction} \label{sec:intro}

Interacting particle/agent systems are a broad spectrum of complex systems with multiple components interacting with each other and co-evolving with time.  Individual interactions yield a wide variety of collective behaviors at different scales and levels of complexity such as clustering, alignment, swarming, synchronization, or dancing equilibrium. There are numerous real-world examples of such systems, including the orbits of planets, motion of self-propelled particles, flocking of birds,  schooling of fish, aggregation of cells, consensus of opinions, and synchronization of oscillators over networks. Understanding the link between individual interactions and global-scale collective behaviors is one of the most fundamental problems in various disciplines.

Modeling interacting agents by differential equations has played a crucial role in exploring the emergence of collective behaviors from individual interactions. However, such systems are often high-dimensional and exhibit many possible dynamical couplings of components that contribute to the dynamics, making them challenging to study \JF{}{\cite{lu2019nonparametric,lu2020learning,lu2021learning,miller2020learning}}.  
  Despite these challenges, recent work has made impressive progress in developing a general physical model derived from Newton's second law that can capture a wide range of collective behaviors \JF{}{\cite{schmidt2009distilling,BKP2017,zhang2018robust,BCCCCGLOPPVZ2008,BCGMSVW2012}}. This model describes a system of $N$ agents interacting according to a set of ODEs, where each agent's motion is influenced by self-propulsion, friction, and interactions with other agents, represented by energy and alignment-based radial interaction kernels: for $i=1,\cdots,N$
\begin{equation}\label{eq:2ndOrder}
m_i\ddot\bx_i = F_i(\bx_i, \dot\bx_i, \mbf{\alpha}_i) + \sum_{i'=1}^N \frac{1}{N} \Big[\intkernel^E (||\bx_{i'} - \bx_i||)(\bx_{i'} - \bx_i) + \intkernela( ||{\bx_{i'} - \bx_i}||)(\dot\bx_{i'} - \dot\bx_i)\Big],
\end{equation} 
where $m_i \geq 0$ is the mass of the agent $i$; $\ddot\bx_i \in \mathbb{R}^d$ is the acceleration, $\dot\bx_i \in \mathbb{R}^d$ is the velocity, and $\bx_i\in \mathbb{R}^d$ is the position of agent $i$; the first term $F_i$ is a parametric function of position and velocities, modeling self-propulsion and frictions of agent $i$ with the environment {with} scalar parameters $\mbf{\alpha}_i$ describing their strength; $||\bx_j-\bx_i||$ is the Euclidean distance; {and} the 1D functions $\phi^{E}, \phi^{A}:\mathbb{R}^{+}\rightarrow \mathbb{R}$ are called the  \textit{energy} and \textit{alignment-based radial interaction kernels} respectively. The $\phi^{E}$ term describes the alignment of positions based on the difference of positions;  the $\phi^{A}$ term describes the alignment of velocities  based on the difference of velocities. We summarize the relevant notations in \cref{tab:2ndOrder_def}.

Particular examples of \eqref{eq:2ndOrder} include the first-order systems ($m_i\equiv 0, \intkernela \equiv 0$) that model clustering and aggregation of agents with application to opinion dynamics \cite{motsch2014heterophilious}, the  second-order Cucker-Smale model ($\intkernele\equiv 0$) \cite{cucker2007mathematics} {for} the flocking behavior of animals and robots, the second-order self-propelling particle model ($\intkernela\equiv 0$) that {is} shown to reproduce (double) milling, ring, escaping or swarming behaviors of biological motors \cite{d2006self}, {and the} anticipation dynamics \cite{shu2021anticipation} ($\intkernela,\intkernele \neq 0$) that describes the velocity alignment and spatial concentration of animal groups. For simplicity of description, we assume the masses of all agents are the same and equal to $\mbf{m}$.  We write the second-order model \eqref{eq:2ndOrder} in a compact form:
\begin{equation}
\mbf{m}\bZ(t) =\force_{\mbf{\alpha}}(\bY(t)) + \rhsfo_\bintkernel(\bY(t))
 \label{eq:2ndOrder_compact}
 \end{equation}
 where  $\bY(t):= \begin{bmatrix}\bX(t)\\ \bV(t) \end{bmatrix} \JF{}{\in} \mathbb{R}^{2dN}$ represents the state variable for the system, $\bZ(t) =  \dot\bV(t) = \ddot \bX(t)$, and
 $\rhsfo_\bintkernel(\bY(t)) = \rhsfo_{\intkernele, \intkernela}(\bY(t))$ represents the sum of  energy and alignment-based interactions as in  \eqref{eq:2ndOrder}.

\begin{table}[tbhp]
\small
\caption{Notations for second-order systems}
\label{tab:2ndOrder_def} 
\begin{center}
\begin{tabular}{ c | c }
\hline
Variable                    & Definition \\
\hline\hline
$N$                         & number of agents \\
\hline
$m_i$                       & mass of agent $i$ \\
\hline
$\bx_i(t)\in \real^d$ & position vector of agent $i$ at time $t$ \\
\hline
$\dot{\bx}_i(t)\in \real^d$ & velocity vector of agent $i$ at time $t$ \\
\hline 
$\ddot{\bx}_i(t)\in \real^d$ & acceleration vector of agent $i$ at time $t$ \\
\hline
$\force$         &  non-collective force   \\
\hline
$\mbf{\alpha}$         & parameters of $\force$ \\
\hline
$\intkernele, \intkernela$ & energy and alignment-based interaction kernels respectively \\
\hline
$\|\cdot\|$ & Euclidean norm in $\mathbb{R}^d$ \\
\hline
\end{tabular}    
\end{center}
\end{table}

\subsection{Data-driven model selection problem}

Recent advancements in data information technology, such as digital imaging, high-resolution lightweight GPS devices, and particle tracking methods, have allowed for the gathering of high-resolution trajectory data of individual particles in various applications. However, a significant issue that remains scarcely addressed is how to select models that match the observational data. For example, while there are many theoretical models known to reproduce flocking patterns, it is challenging to determine which one generates the pattern observed in the data. Previous theoretical and numerical studies cannot address this problem, as predetermined governing equations are needed, and the aim is often to reproduce qualitative rather than quantitative dynamics.

To address this issue, we consider the data-driven model selection problem, aiming to select possible models from a general form to match the observational data. For instance, given the motion data of a school of fish, we aim to determine whether to use first-order or second-order models and which types of interactions, such as alignment versus energy-based or both, contribute to collective patterns. These are challenging questions that practitioners typically address based on their expertise in the field. In this paper, we seek to develop data-driven methods to automate this step by considering a general model that incorporates many classical models as special cases.

Mathematically, we formulate the problem as follows. Given approximate observations of multiple trajectory data $\mathcal{D}_{M,L}:=\{ \bY^{(m)}(t_l), \bZ^{(m)}(t_l)\}_{m,l=1}^{M,L}$, where the observation time instances are denoted by $0=t_1<\dots<t_L=T$ and $m$ denotes the trial number of experiments starting from different initial conditions, the goal is to infer the interaction kernels $\bintkernel={\intkernele, \intkernela }$ as well as the unknown scalar parameters $\mbf{\alpha}$ and possibly $\mbf{m}$ from the trajectory data $\mathcal{D}_{M,L}$. Subsequently, we use the learned governing equations to make predictions about future events or simulate new datasets.

\subsection{Scalable Model Selection by Gaussian processes}

The field of data-driven model selection faces two primary practical challenges. Firstly, there is often limited information available on the parametric forms of interaction kernels, making it difficult to select a suitable approximation dictionary. Secondly, datasets may be scarce and noisy. Gaussian process (GP) based approaches in machine learning offer a solution to these challenges, as they are known for their ability to learn a rich class of nonlinear functions without making assumptions about their parametric form and for quantifying the associated uncertainty. However, the challenge of scalability to large-scale problems remains a significant hurdle for specific applications.

This paper proposes a novel approach to address these challenges by leveraging the inference power of Gaussian processes and developing efficient techniques to improve scalability. Computationally,

\begin{itemize}
    \item We propose a novel method by modeling interaction kernels as two independent Gaussian processes to learn  \eqref{eq:2ndOrder} from data with uncertainty quantification. We investigate whether types  of interaction kernels and order information (first versus second order) of the system can be learned from scarce noisy data. We conduct intensive numerical experiments  on various prototypical systems exhibiting clustering, milling, and flocking behaviors that demonstrate the effectiveness. 

    \item We propose effective acceleration techniques based on the recent progress from randomized numerical linear algebra.  
    
    \item  Our method is applied to modeling two real-world fish motion sets that display flocking and milling patterns up to 248 dimensions and outperforms competitor methods that use SINDy and feed-forward neural networks. 
    
\end{itemize}

Theoretically, 
\begin{itemize}
    \item  We derive a Representer theorem that connects the GP-based estimators with the kernel ridge regression estimators, shedding light on  the role of the hyperparameters in learning. It also provides a basis representation for the estimators of interaction kernels, which enables efficient trajectory prediction using learned models over larger time intervals. 

    \item We study the well-posedness of the inverse problem for learning interaction kernels in a statistical setting.

\end{itemize}

\subsection{Relevant works}

Integrating machine learning techniques into the data-driven discovery of dynamical systems (see e.g. \cite{bongard2007automated, schmidt2009distilling,BPK2016, RBPK2017, BKP2017,han2015robust, kang2019ident,zhang2018robust,BCCCCGLOPPVZ2008,BCGMSVW2012,raissi2018deep, raissi2018hidden, long2017pde}) has become a hot topic in scientific machine learning, as it provides powerful models to represent the complex functional data. In terms of parametric methods, one can refer to \cite{lu2021deepxde} (and references therein) for the most recent survey on deep learning techniques and \cite{BPK2016,TranWardExactRecovery,schaeffer2018extracting,boninsegna2018sparse} for sparse regression techniques.

 Gaussian process regression (GPR) is a non-parametric Bayesian machine learning technique for supervised learning with a built-in quantification of uncertainty framework.  As such, GPs have been applied to learn ODEs, SDEs, and PDEs \cite{heinonen2018learning,archambeau2007gaussian,yildiz2018learning,zhao2020state,raissi2017machine,chen2020gaussian,wang2021explicit,nonlinPDEs,lee2020coarse,akian2022learning,darcy2021learning}  and  lead to more accurate and robust models of dynamical systems. Because of the distinctive nature of dynamical data,  it necessitates novel ideas and  nontrivial efforts tailored to particular types of dynamical systems and data regimes. We model the latent interaction kernels as GPs and imbue them with the structure of our governing equations (translation and rotational invariance). This makes our work distinguishable from {most works, which model} state variables as GPs.

In the context of interaction kernel learning in interacting particle systems, least square estimators derived from the maximum likelihood method are the most frequently used, where a challenge lies in the selection of the basis to represent the interaction kernels. One can refer to \cite{lu2019nonparametric,lu2020learning,lu2021learning,miller2020learning} for the usage of  a piecewise polynomial basis. The random feature method together with sparse regression techniques is recently proposed in \cite{liu2021random}. One can also refer to the recent methodology development of interaction kernel/potential learning in mean-field systems such as \JF{}{\cite{lang2020learning,lang2021identifiability,he2022numerical,kemeth2022learning,tang2023identifiablility}}.

In particular, \cite{lu2021learning,miller2020learning} considered learning theory for heterogeneous systems and showed that learning multiple interaction kernels simultaneously is challenging and regularization is necessary.  

For scarce noisy data, our GP method leverages the underlying statistical inference power to select the best basis to represent the observational dynamics and provides effective regularization. It yields accurate recovery of the  \textit{governing equation} beyond learning only interaction kernels. It is well-known that the non-collective force also plays an important role in determining the collective behaviors. The governing equation recovery makes our method more practical than previous work that only focuses on interaction kernels. Further, we analyze the well-posedness of the inverse problems, which complements the missing analysis in \cite{miller2020learning}. 
 This work is an extension of our recent work \cite{learning2022} on a single kernel case where we assumed $\intkernela \equiv 0$ and the focus was the theoretical framework for error analysis. Here, we consider a more generalized model involving two types of kernels and consider the model selection problems. The focus is shifted to the computational aspects concerning scalability and uncertainty quantification,  and real data applications.

 \subsection{Notation and preliminaries}
\paragraph{Notation} Let $\rho$ be a Borel positive measure on $D$ dimensional Eucliean space $\mathbb{R}^{D}$.  We use $L^2(\mathbb{R}^{D};\rho;\mathbb{R}^{n})$ to denote the set of $L^2(\rho)$-integrable vector-valued functions that map $\mathbb{R}^{D}$ to $\mathbb{R}^{n}$. For a function $\mbf{f}\in L^2(\mathbb{R}^{D};\rho;\mathbb{R}^{n})$, and a vector $\bX=[\bx_1^{\top},\cdots,\bx_m^{\top}]^T \in \mathbb{R}^{mD}$ with $\bx_i \in \mathbb{R}^D$, we use the notation $\mbf{f}(\bX)$ to represent the image of the vector under the function of $\mbf{f}$ componentwisely, namely, $\mbf{f}(\bX)=[\mbf{f}(\bx_1)^{\top},\cdots,\mbf{f}(\bx_m)^{\top}]^{\top}\in \mathbb{R}^{mn}$.
Let $\mathcal{S}_1$ be  a  measurable subset of $\mathbb{R}^{m}$, the restriction of the measure $\rho$ on $\mathcal{S}_1$, denote by $\rho\mres \mathcal{S}_1$, is defined as $\rho\mres \mathcal{S}_1(\mathcal{S}_2)=\rho(\mathcal{S}_1 \cap \mathcal{S}_2)$ for any measurable subset $\mathcal{S}_2$ of $\mathbb{R}^{D}$. We used $\mathcal{N}(0,I_{d\times d})$ to denote the standard multivariate Gaussian distribution in $\mathbb{R}^d$. 

\paragraph{Preliminaries on operator algebras}

Let $\mathcal{H}_1, \mathcal{H}_2$ be  Hilbert spaces. We use $\langle\cdot,\cdot\rangle_{\mathcal{H}_1}$ to denote the inner product over $\mathcal{H}_1$, and still use $\langle\cdot,\cdot\rangle$ to denote the inner product on the Euclidean space.  We denote by $\mathcal{B}(\mathcal{H}_1,\mathcal{H}_2)$  the set of bounded linear operators mapping $\mathcal{H}_1$ to $\mathcal{H}_2$.
Let $A \in \mathcal{B}(\mathcal{H}_1,\mathcal{H}_2)$, we use $\mathrm{Im}(A)$ to denote its range and $\|A\|$ to denote its operator norm.  $A$ is a compact operator if $A$ maps bounded subsets of $\mathcal{H}_1$ to relatively compact subsets of $\mathcal{H}_2$ (subsets with compact closure in $\mathcal{H}_2$). We use $A^*: \mathcal{H}_2\rightarrow \mathcal{H}_1$ to denote the adjoint operator of $A$, that is, $\forall f \in \mathcal{H}_1$, $g\in\mathcal{H}_2$, $\langle Af,g\rangle_{\mathcal{H}_2}=\langle f,A^*g\rangle_{\mathcal{H}_1}$.

 For $d,N,M,L \in \mathbb{N}^+$, let $\mbf{w}=(\mbf{w}_{m,l,i})_{m,l,i=1}^{M,L,N},\mbf{z}=(\mbf{z}_{m,l,i})_{m,l,i=1}^{M,L,N}\in \mathbb{R}^{dNML}$ with  
$\mbf{w}_{m,l,i},\mbf{z}_{m,l,i} \in\mathbb{R}^d$, we define 
\begin{equation}\label{winnerp}
\langle \mbf{w}, \mbf{z}\rangle =\frac{1}{MLN}\sum_{m,l,i=1}^{M,L,N} \langle \mbf{w}_{m,l,i}, \mbf{z}_{m,l,i}\rangle,
\end{equation} where $\langle \mbf{w}_{m,l,i}, \mbf{z}_{m,l,i}\rangle$ is the canonical inner product on $\mathbb{R}^d$.

Then for vectors $\mbf{y} \in \mathbb{R}^{mdN}$ and functions $\mbf{g} : \mathbb{R}^{mdN} \rightarrow \mathbb{R}^{dn}$ for some $m,n \in \mathbb{N}^+$, and let $\rho$ be a measure at a Borel subset $\mathcal{Y}$ of $\mathbb{R}^{mdN}$, we have the norm:
\begin{equation}\label{normed}
\lVert \mbf{g}(\by) \rVert_{L^2(\rho)}^2 = \frac{1}{n}\int_{\mathcal{Y}} \sum_{i=1}^{n} \lVert \mbf{g}_i(\by) \rVert^2 \rho(d\by)
\end{equation} where $\mbf{g}(x)$ is componentwise denoted by $\mbf{g}_i(\by): \mathbb{R}^{mdN}\rightarrow \mathbb{R}^{d}$.

For two Borel positive measures $\rho_1, \rho_2$ defined on $\mathbb{R}^{D}$, $\rho_1$ is said to be absolutely continuous with respect to $\rho_2$, $\rho_1 \ll \rho_2$, if $\rho_1(\mathcal{S}) = 0$ for every set $\rho_2(\mathcal{S}) = 0$, $\mathcal{S} \subset \mathbb{R}^{D}$. $\rho_1$ and $\rho_2$ are called equivalent iff $\rho_1 \ll \rho_2$ and $\rho_2 \ll \rho_1$. 
The product measure $\rho_1 \times \rho_2$ is defined to be a measure on $\mathbb{R}^{2D}$ satisfying the property $(\rho_{1}\times \rho_{2})(\mathcal{S}_{1}\times \mathcal{S}_{2})=\rho_{1}(\mathcal{S}_{1})\rho_{2}(\mathcal{S}_{2})$ for all subsets $\mathcal{S}_i \subset \mathbb{R}^D$, $i=1,2$.

\JF{}{\paragraph{Preliminaries on GPs (Gaussian Processes) Prior } We say $\phi \sim \mathcal{GP}(u,K)$ to denote our prior on $\phi$. In particular, this means that for any $r \in \mathbb{R}$, the random variable $\phi(r)$ is Gaussian: $\phi(r) \sim \mathcal{N}(u(r), K(r,r))$, where $\mathcal{N}$ denotes the normal  or multivariable normal distributions, $u:\mathbb{R} \to \mathbb{R}$ is the mean function, and $K: \mathbb{R}\times\mathbb{R} \to \mathbb{R}$ is the covariance function. Similarly, for any $(r, r') \in \mathbb{R}^2$, the joint distribution of $ \begin{bmatrix} \phi(r)\\ \phi(r') \end{bmatrix}$ is multivariate Gaussian: $\begin{bmatrix} \phi(r)\\ \phi(r') \end{bmatrix} \sim \mathcal{N} \left(   \begin{bmatrix} (u(r)\\ (u(r') \end{bmatrix} ,\begin{bmatrix} K ( r, r)& K(r,r')\\ K(r',r)&K(r',r') \end{bmatrix}\right)$. This extends in a natural way to any finite set $(r_1, \dots, r_N) \in \mathbb{R}^N$.}

\paragraph{Preliminaries on RKHSs}  Let $\mathcal{D}$ be a compact subset of $\mathbb{R}^D$. We say that $K: \mathcal{D}\times \mathcal{D} \rightarrow \mathbb{R}$ is a Mercer kernel if it is continuous, symmetric, and  positive semidefinite, i.e., for any finite set of distinct points $\{x_1,\cdots,x_M\} \subset \mathcal{D},$ the matrix $(K(x_i,x_j))_{i,j=1}^{M}$ is positive semidefinite.  For $x \in \mathbb{R}^D$, $K_{x}$ is a function defined on $\mathcal{D}$ such that $K_{x}(y)=K(x,y)$, $y \in \mathcal{D}$. The Moore–Aronszajn theorem proves that there is an  RKHS $\mH$ associated with  the kernel $K$, which is defined to be the closure of the linear span of the set of functions $\{K_x:x\in \mathcal{D}\}$ with respect to the inner product $\langle \cdot,\cdot\rangle_{\mH}$ satisfying $\langle K_x, K_y\rangle_{\mH}=K(x,y)$. Let $\mK$ be a Mercer kernel that is defined on $[0, R]\times [0, R]$ and use $\mH$ to denote the RKHS associated with $\mK$. For two RKHS $\mathcal{H}_{K_1}, \mathcal{H}_{K_2}$, with $K_1, K_2: \mathcal{D}\times \mathcal{D} \rightarrow \mathbb{R}$, the product RKHS $\mathcal{H}_{K_1} \times \mathcal{H}_{K_2}$ is defined to be the closure of the linear span of the set of functions $\{(K_{1,x_1}, K_{2,x_2}): x_1,x_2\in \mathcal{D}\}$ with respect to the inner product $\langle \cdot,\cdot\rangle_{\mathcal{H}_{K_1}\times \mathcal{H}_{K_2}}$ satisfying $\langle (K_{1,x_1}, K_{2,x_2}), (K_{1,y_1}, K_{2,y_2})\rangle_{\mathcal{H}_{K_1}\times \mathcal{H}_{K_2}}=K_1(x_1,y_1) + K_2(x_2,y_2)$.

\section{Methodology}
\label{sec: Methodology}

\noindent

In this section, we propose a learning approach based on GPs for the model selection problem.

\subsection{Two independent Gaussian process priors}\label{sec:kernel} We start by modeling the interaction kernel functions $\intkernele$ and $\intkernela$ with the priors as two independent Gaussian processes
\begin{equation}
\intkernele \sim \mathcal{GP}(0,K_\thetae(r,r')), \qquad
\intkernela \sim \mathcal{GP}(0,K_\thetaa(r,r')),
\end{equation}
where $K_\thetae$, $K_\thetaa$ are covariance functions with hyperparameters $\boldsymbol{\theta} = (\thetae, \thetaa)$.  $\boldsymbol\theta$ can either be chosen by the modeler or tuned via a data-driven procedure discussed later.

 \subsection{Training of hyperparameters via maximum likelihood estimation}
\label{subsec:trainhyp}

In real-world modeling, it is possible that some other parameters such as $\mbf{\alpha}$ and noise level in the data are unknown.  In this section, we detail how to perform the estimation of these physical parameters in the governing equation via a data-driven hyperparameter tuning process induced by the Gaussian process. This flexible training procedure distinguishes the Gaussian process from other kernel-based methods \cite{tipping2001sparse,scholkopf2002learning,vapnik2013nature} and regularization-based approaches \cite{tihonov1963solution,tikhonov2013numerical,poggio1990networks}.

 We organize the training data into the vector format  $\bbY = [\bY^{(1,1)},\dots,\bY^{(M,L)}]^T \in \mathbb{R}^{dNML}$, and $\bbZ = [\bZ^{(1,1)},\dots,\bZ^{(M,L)}]^T \in \mathbb{R}^{dNML}$ where
\begin{equation}
    \bY^{(m,l)} = \bY^{(m)}(t_l), \quad  \bZ^{(m,l)} = \bZ^{(m)}(t_l).
\end{equation} 

To model the noise, we assume 
$\bbZ = [\bZ^{(1,1)}_{\sigma^2},\dots,\bZ^{(M,L)}_{\sigma^2}]^T \in \mathbb{R}^{dNML}$ where
\begin{equation}\label{gpnoise}
    \mbf{m}\bZ^{(m,l)}_{\sigma^2} = \force_{\mbf{\alpha}}(\bY^{(m,l)}) + \rhsfo_{\bintkernel}(\bY^{(m,l)}) + \epsilon^{(m,l)},
\end{equation}
with i.i.d (independent and identically distributed) noise $\epsilon^{(m,l)} \sim \mathcal{N}(0, \sigma^2 I_{dN})$ that is also independent of the Gaussian processes. Later, we will show the role of $\sigma$ in the prediction step is equivalent to the role of the regularization constant in a Tikhonov regularization problem.

Therefore, based on the properties of Gaussian processes, with the priors of $\intkernele$, $\intkernela$, we have 
\begin{equation}
    \mbf{m}\bbZ \sim \mathcal{N}(\force_{\mbf{\alpha}}(\bbY), K_{\rhsfo_{\bintkernel}}(\bbY,\bbY;\theta) + \sigma^2 I_{dNML}),
\end{equation}
with the mean vector $\force_{\mbf{\alpha}}(\bbY) = \mathrm{Vec}(\{\force_{\mbf{\alpha}}(\bY^{(m,l)})\}_{m,l=1}^{M,L})\in \mathbb{R}^{dNML}$, and $K_{\rhsfo_{\bintkernel}} (\bbY,\bbY;\theta)\in \mathbb{R}^{dNML \times dNML}$ is the covariance matrix between $\rhsfo_{\bintkernel}(\bbY)$ and $\rhsfo_{\bintkernel}(\bbY)$, which can be computed elementwise based on the covariance functions $K_\thetae$, $K_\thetaa$, see Appendix \cref{learningapproach} for detailed formulas.

Thus, for the hyperparameters $\mbf{\alpha}$, $
\mbf{\theta}$, and $\sigma$, we can train by maximizing the probability of the observational data, which is equivalent to minimizing the negative log marginal likelihood (NLML) (see Chapter 4 in \cite{williams2006gaussian})
\begin{align}
    -\log p(\mbf{m}\bbZ \vert \bbY,\mbf{\alpha},\mbf{\theta},\sigma^2) &= \frac{1}{2} (\mbf{m}\bbZ - \force_{\mbf{\alpha}}(\bbY))^T(K_{\rhsfo_{\bintkernel}}(\bbY,\bbY;\mbf{\theta}) + \sigma^2I_{dNML})^{-1}(\mbf{m}\bbZ - \force_{\mbf{\alpha}}(\bbY))\notag\\ & \qquad +\frac{1}{2}\log\vert K_{\rhsfo_{\bintkernel}}(\bbY,\bbY;\mbf{\theta})+\sigma^2 I_{dNML}\vert + \frac{dNML}{2} \log 2\pi.
\label{eq:likelihood}
\end{align}
Note here the marginal likelihood does not simply favor the models that fit the training data best, but induces an automatic trade-off between data-fit and model complexity.
To solve for the hyperparameters $(\mbf{\alpha}, \mbf{\theta},\sigma)$, we can apply the conjugate gradient (CG) optimization (see Chapter 5 in \cite{williams2006gaussian}) to minimize the negative log marginal likelihood. More details are shown in Appendix \cref{learningapproach}.

\begin{table}[tbhp]
\caption{Notations for covariances}
\label{tab:methodology_def} 
\begin{center}
\resizebox{0.8\textwidth}{!}{%
\begin{tabular}{ c | c }
\hline
Variable                    & Definition \\
\hline\hline
$K_{\mbf{\theta}}(\cdot,\cdot)$                         &  covariance kernel function  with parameters $\mbf{\theta}$ \\
\hline
$K_{\thetae}(\cdot,\cdot),K_{\thetaa}(\cdot,\cdot)  $                      &  covariance kernels for modeling $\intkernele$, $\intkernela$\\
\hline
$K_{\rhsfo_{\intkernel}}(\cdot,\cdot)$  & \makecell{ covariance matrix between $\rhsfo_{\intkernel}(\cdot)$ and $\rhsfo_{\intkernel}(\cdot)$} \\
\hline
\makecell{$K_{\rhsfo_\intkernel,\intkernele}(\cdot,\cdot):= K_{\intkernele,\rhsfo_\intkernel}(\cdot,\cdot)^T$} &   covariance matrix between $\rhsfo_\bintkernel(\cdot)$ and $\intkernele(\cdot)$ \\
\hline 
\makecell{$K_{\rhsfo_\intkernel,\intkernela}(\cdot,\cdot):=K_{\intkernela,\rhsfo_\intkernel}(\cdot,\cdot)^T$} &   covariance matrix between $\rhsfo_\bintkernel(\cdot)$ and  $\intkernela(\cdot)$\\
\hline
\end{tabular} }
\end{center}
\end{table}

\subsubsection{Parameter \texorpdfstring{$\bf{m}$} -- Model selection of the order for the dynamical system}
\label{subsec: MS_order} 

When modeling real-world dynamics, sometimes we are not sure whether to use first-order or second-order systems. From the parameter estimation perspective, it is equivalent to determining if the mass of particles is equal to zero. We can train $\mbf{m}$ via minimizing \eqref{eq:likelihood}.  If the estimation of $\mbf{m}$ is close to zero, we can consider identifying the dynamics as a first-order system. One can refer to \cref{ex: ODS}.

\subsection{Learning interaction kernels}

\begin{table}[tbhp]
\caption{Notations for second-order systems}
\label{tab:2ndOrder_vecdef} 
\centering
\small{
\small{\begin{tabular}{ c | c }
\hline
Variable                    & Definition \\
\hline\hline
$\bX \in \real^{dN}$ &  vectorization of position vectors $(\bx_i)_{i=1}^{N}$\\
\hline 
$\bV \in \real^{dN}$ & vectorization of velocity vectors $(\bv_i)_{i=1}^{N} = (\dot{\bx}_i)_{i=1}^{N}$\\
\hline 
$\bY \in \real^{2dN} $ & $\bY = (\bX, \bV)^T$\\
\hline
$\bZ \in \real^{dN} $ & vectorization of $(\ddot{\bx}_i)_{i=1}^{N}$\\
\hline
$\mbf{r}^\bx_{ij}, \mbf{r}^{\bx'}_{ij}\in \real^d$ & $\bX(t)_j-\bX(t)_i$, $\bX(t')_j-\bX(t')_i$  \\
\hline
$\mbf{r}^\bv_{ij}, \mbf{r}^{\bv'}_{ij}\in \real^d$ & $\bV(t)_j-\bV(t)_i$, $\bV(t')_j-\bV(t')_i$  \\
\hline
$r^\bx_{ij},r^{\bx'}_{ij}\in \real^+$ & $r^\bx_{ik}=\|\mbf{r}^\bx_{ik}\|,  r^{\bx'}_{ij}=\|\mbf{r}^{\bx'}_{ij}\|$ \\
\hline
$\rhsfo_{\intkernele},\rhsfo_{\intkernela}$ & \makecell{ energy and alignment-based interaction force field }\\
\hline
$\rhsfo_{\bintkernel}$ & interaction force  field with  $\bintkernel = (\intkernele,\intkernela)$ \\
\hline
\end{tabular}}  
}
\end{table}

We plug the estimators of hyperparameters obtained in \cref{subsec:trainhyp} into the system and assume they are known. In this subsection, we show how to learn interaction kernels.  For \JF{}{any $r^\ast \in \mathbb{R}$ and the corresponding values of the kernel functions, } $\intkernel^{\mathrm{type}}(r^*)$, $\mathrm{type}=E \text{ or } A$, since we have
\begin{equation}
    \begin{bmatrix}
    \mbf{m}\mathbb{Z}-F_{\mathbf{\alpha}}(\mathbb{Y})\\
    \intkernel^{\mathrm{type}}(r^\ast)
    \end{bmatrix}
    \sim \mathcal{N} \left( 0,
    \begin{bmatrix}
    K_{\rhsfo_{\bintkernel}}(\bbY, \bbY)+\sigma^2I_{dNML} & K_{\rhsfo_\bintkernel,\intkernel^{\mathrm{type}}}(\bbY, r^\ast)\\
    K_{\intkernel^{\mathrm{type}},\rhsfo_\bintkernel}(r^\ast, \bbY) & K_{\theta^{\mathrm{type}}}(r^\ast,r^\ast)
    \end{bmatrix}
    \right),
\end{equation} 
where $K_{\rhsfo_\bintkernel,\intkernel^{\mathrm{type}}}(\bbY, r^*) = K_{\intkernel^{\mathrm{type}},\rhsfo_\intkernel}(r^*,\bbY)^T$ denotes the covariance matrix between $\rhsfo_{\bintkernel}(\bbY)$ and $\intkernel^{\mathrm{type}}(r^*)$. Conditioning on $\rhsfo_{\bintkernel}(\bbY)$, we obtain the posterior/predictive distribution for \JF{}{the kernel function value at $r^\ast$,} $\intkernel^{\mathrm{type}}(r^\ast)$ \JF{}{(see \cref{lemma: conditioning Gaussian} in Appendix for detailed derivation)}, i.e.
\begin{equation}
    p(\intkernel^{\mathrm{type}}(r^\ast)\vert \bbY,\bbZ,r^\ast) \sim \mathcal{N}(\bar{\intkernel}^{\mathrm{type}},var(\bar{\intkernel}^{\mathrm{type}})),
\end{equation}
where
\begin{equation}
    \bar{\intkernel}^{\mathrm{type}} = K_{\intkernel^{\mathrm{type}},\rhsfo_\bintkernel}(r^\ast,\bbY)(K_{\rhsfo_{\bintkernel}}(\bbY,\bbY)+\sigma^2I_{dNML})^{-1}(\mbf{m}\bbZ - \force_{\mbf{\alpha}}(\bbY)),
\label{eq:estimated phi}    
\end{equation}
\begin{equation}
    var(\bar{\intkernel}^{\mathrm{type}}) = K_{\theta^{\mathrm{type}}}(r^\ast,r^\ast) - K_{\intkernel^{\mathrm{type}},\rhsfo_\bintkernel}(r^\ast,\bbY)(K_{\rhsfo_{\bintkernel}}(\bbY,\bbY)+\sigma^2I_{dNML})^{-1}K_{\rhsfo_\bintkernel,\intkernel^{\mathrm{type}}}(\bbY,r^\ast).
    \label{eq:estimated var phi}
\end{equation}
The posterior variance $var(\bar{\intkernel}^{\mathrm{type}})$ can be
used as a good indicator for the uncertainty of the estimation $\bar{\intkernel}^{\mathrm{type}}$ based on our Bayesian approach.\\

\begin{algorithm}[!htb]
\algorithmicrequire\ $(\bbY, \bbZ)$ (training data), $r^\ast$ (test point), $K_{\mbf{\theta}}$ (covariance functions), $\rhsfo_{\bintkernel}$ (interaction function), $\force_{\mbf{\alpha}}$ (force function)
\begin{algorithmic} [1]
\STATE $(\hat\balpha,\hat{\mbf{\theta}},\hat\sigma^2) = \argmin{\balpha,\mbf{\theta},\sigma^2} -\log p(\mbf{m}\bbZ\vert \bbY,\mbf{\alpha},\mbf{\theta},\sigma^2)$\\
\hfill\COMMENT{solve for parameters by minimizing NLML \eqref{eq:likelihood} using CG}
\STATE $L := \textrm{cholesky}(K_{\rhsfo_{\bintkernel}}(\bbY,\bbY) + \hat\sigma^2I)$
\STATE $\gamma:= L^T\backslash(L\backslash (\mbf{m}\bbZ - \force_{\mbf{\hat\alpha}}(\bbY)))$
\STATE $K_E^\ast := K_{\rhsfo_\bintkernel, \intkernele}(\bby, r^\ast)$\\ $K_A^\ast := K_{\rhsfo_\bintkernel, \intkernela}(\bbY, r^\ast)$ \hfill\COMMENT{compute covariances between $\rhsfo_{\bintkernel}(\bbY)$ and $\intkernele(r^\ast)$, $\intkernela(r^\ast)$}
\STATE $\bar{\intkernel}^{E\ast} : = (K_E^\ast)^T \gamma$\\
$\bar{\intkernel}^{A\ast} : = (K_A^\ast)^T \gamma$
\hfill\COMMENT{predictive mean \cref{eq:estimated phi}}
\STATE $\bv_E = L \backslash K_E^\ast$,
$\bv_A = L \backslash K_A^\ast$ 
\STATE $var(\intkernel^{E\ast}) := K_{\hat\theta^E}(r^\ast,r^\ast) -  \bv_E^T \bv_E$\\
$var(\intkernel^{A\ast}) := K_{\hat\theta^A}(r^\ast,r^\ast) -  \bv_A^T \bv_A$
\hfill\COMMENT{predictive variance \cref{eq:estimated var phi}}
\end{algorithmic}
\algorithmicensure\ $\bar{\intkernel}^{E\ast},\bar{\intkernel}^{A\ast}$ (mean), $var(\intkernel^{E\ast}),var(\intkernel^{A\ast})$ (variance)
\caption{{\bf Learning kernels} \label{Algorithm:learning}}

\end{algorithm}

\subsection{Prediction of trajectories and its uncertainty quantification} 
\label{subsec:traj_uq}
We use the posterior mean estimators of $\bintkernel$ in trajectory prediction by performing numerical simulations of 
 the equations
\begin{equation}
 \mbf{m}\hat\bZ(t) =\force_{\mbf{\hat\alpha}}(\bY(t)) + \hat\rhsfo_{\bar{\bintkernel}}(\bY(t)).
 \end{equation}

 We can also perform uncertainty quantification for the trajectory prediction via the uncertainty band of $\mbf{\hat\phi}$. We adopted a Monte Carlo method, where we used   $\mbf{\phi}$ sampled from the posterior distribution in each simulation. Then the predictions of the trajectories are given by the mean of the trajectories' samples and the  uncertainty band of each trajectory is given by the standard deviation, with the results of experiments shown in \cref{sec:numerics}. Another possible alternative is to use step-wise uncertainty quantification based on the numerical integrator scheme such as the one-step Euler method. In this case, it is easy to compute the variance of the solution from the posterior distribution of $\bintkernel$, since the vector field $\hat\rhsfo_{\bar{\bintkernel}}(\bY(t))$ is a linear combination of $\hat{\bintkernel}$ by its definition, which suggests it also follows a Gaussian distribution and the uncertainty band can be derived from its covariance matrix.

\subsection{Acceleration of the Computation} \label{accelex}  While the full GP methods described above yield extremely accurate predictions in our empirical examples, a well-known limitation is the computational complexity; calculating the log determinant of $K_{\rhsfo_{\bintkernel}}$ and inverting the kernel matrices in the maximum likelihood estimation and prediction steps scales cubically with the matrix dimension, which is $\mathcal{O}((NdML)^3)$. Therefore, the naive approach can quickly become infeasible for large-scale problems. Below, we describe our integrated approach to the scalable estimation of hyperparameters in maximum likelihood estimation and scalable kernel prediction.

\subsubsection{Efficient Hyperparameter Optimization}

There are many recent advancements in accelerating the hyperparameter learning computations in the full GP methods for regression tasks. Our problem, however, presents many numerical difficulties that dampen runtime gains from traditional computational methods and must be addressed:

\begin{itemize}

\item  \textbf{Lack of sparsity}. Many classical acceleration techniques  rely on the sparsity of the kernel matrix $K_{\rhsfo_{\bintkernel}}$. As our kernel depends on pairwise distance and our modeling is nonlocal, we do not have a sparse kernel matrix in our formulation. Our method must be able to operate on dense $K_{\rhsfo_{\bintkernel}}$.

\item  \textbf{Extreme ill-conditioning and higher accuracy requirements}. The $L^2$ condition number of a matrix is the ratio of its maximum and minimum singular values. When much larger than $1$, the condition number indicates that a matrix is nearly singular, and thus accuracy-reducing errors in computation will occur. For many problems, such as those addressed in \cref{accelNumerical}, the kernel matrix $K_{\rhsfo_{\bintkernel}}$ has observed $L^2$ condition number above $10^{15}$. These extremely high condition numbers result in slow and inaccurate computation when using traditional methods. Our problem is also an inverse problem while learning our hyperparameters for $K_{\rhsfo_{\bintkernel}}$ (see \cref{subsec:representerthm}). This is very sensitive to perturbations, especially as our optimization problem for the hyperparameters is generally not convex. {We must carefully balance the tradeoff between computational time and accuracy.}

\end{itemize}

 We empirically observed the \textit{approximately low-rank structure} of $K_{\rhsfo_{\bintkernel}}$ in various examples. This motivated us to adapt two main classes of algorithms  in \cite{wenger22} for acceleration (see pseudocode and additional details in Appendix \cref{secA1}):

\begin{itemize}
    \item  \textbf{Preconditioned conjugate gradient (PCG) algorithm}. The PCG algorithm allows us to avoid explicit computation of the inverse matrix $(K_{\rhsfo_{\bintkernel}}(\bbY,\bbY;\theta) + \sigma^2I)^{-1}$ in both MLE and prediction, as well as compute the coefficients needed in the stochastic Lanczos quadrature below. Using preconditioners, a classical numerical technique to lower condition numbers, is a necessity for variance reduction. In addition, we must maintain a low error tolerance for PCG to preserve our accuracy throughout learning. Finding effective preconditioners that are suitable to the unique structure of our kernel matrices is a challenge. We propose using the Random Gaussian Nystrom preconditioner \cite{randomnyst} to ensure favorable tradeoffs in running time and accuracy. In our practical implementation, this preconditioner outperformed other low-rank approximation preconditioners and has low construction and inversion costs, {see \cref{accelNumerical}}.

    \item  \textbf{Stochastic trace estimation} for log determinant acceleration. We utilize the identity:
    
    $$\log \det(K_{\rhsfo_{\bintkernel}}(\bbY,\bbY;\theta) + \sigma^2I) = \log \det (P) + \log \det (P^{-\frac 1 2} (K_{\rhsfo_{\bintkernel}}(\bbY,\bbY;\theta) + \sigma^2I) P^{- \frac 1 2})$$ 

    When $P$ is chosen to be a preconditioner, this identity can prove highly useful. The Random Gaussian Nystrom preconditioner allows us to efficiently compute $\log \det (P)$, and for the remainder, we use the recently developed variance reduced Hutchinson's Estimator \cite{hutch++} combined with stochastic Lanczos quadrature \cite{wenger22}. %

\end{itemize}

\paragraph{Analysis of new computational complexity}
PCG can reduce explicit inversion complexity from $\mathcal{O}((NdML)^3)$ to $\mathcal{O}(t(NdML)^2)$, where $t$ is the number of iterations. The stochastic Lanczos quadrature improves log determinant complexity from $\mathcal{O}((NdML)^3)$ to $\mathcal{O}( t \ell (NdML)^2) + \mathcal{O}(\log \det (P))$, where $t$ is both the number of eigenvalues and the number of iterations, $\ell$ is the number of runs of stochastic Lanczos, and $\mathcal{O}(\log \det (P))$ is the complexity of computing the log determinant of the preconditioner. In practice, we chose $\ell,t << NdML$. This lowers the theoretical complexity of these steps to the quadratic $\mathcal{O}(t \ell (NdML)^2)$. For the Random Gaussian Nystrom preconditioner $P$ with rank $r$, we have construction in $\mathcal{O}(r^2 (NdML) + r^3)$ time, inversion in $\mathcal{O}(r^3)$ time and log determinant in $\mathcal{O}(r)$ time.

\section{Theoretical analysis}

In this section, we are concerned with two theoretical problems regarding learning interaction kernels in the prediction step. The first one is to understand the role of hyperparameters in  the prediction step of the Gaussian process, i.e., $\theta^E, \theta^A$, and the Gaussian noise $\sigma$. The second one is to study well-posedness as an inverse problem. 

As in the prediction step, interaction kernels are the only unknown terms in the equations. We make the following simplification on the form of equations to avoid unnecessary technical hurdles:
\begin{align}\label{secondorder}
 \ddot\bX(t)&=\rhsfo_{\bintkernel}(\bY(t)) = \rhsfo_{\intkernele}(\bX(t)) + \rhsfo_{\intkernela}(\bY(t)),
 \end{align}
where the masses of the agents are assumed to be one and non-collective forces are assumed to be zero. Our analysis can be  extended to general second-order systems \eqref{eq:2ndOrder} with known mass and non-collective force terms with slight modifications.

\subsection{The Representer theorem}
\label{subsec:representerthm}
In the classical regression setting \cite{williams2006gaussian}, there is an interesting link between GP regression and kernel ridge regression (KRR), where the posterior mean can be viewed as a KRR estimator to solve a regularized least square empirical risk functional. In our setting, we have noisy functional observations of the interaction kernels, i.e.,   the  $\{r_{\bbX_M}, r_{\bbV_M}, \bbZ_{\sigma^2,M}\}$ instead of the pairs $\{r_{\bbX_M}, \intkernele(r_{\bbX_M}), \intkernela(r_{\bbX_M})\}$, where \JF{}{$r_{\bbX_M}, r_{\bbV_M} \in \mathbb{R}^{MLN^2}$} are the sets contains all the pairwise distances in $\bbX_{M}$, and $\bbV_{M}$, i.e. 
\begin{equation}
\label{eq:pairdistances}
r_{\bbX_M} = \{ r_{ij}^{\bX^{(m,l)}}\}_{i,j,m,l=1}^{N,N,M,L}, \quad r_{\bbV_M} = \{ r_{ij}^{\bV^{(m,l)}}\}_{i,j,m,l=1}^{N,N,M,L}, \quad
\bbZ_{\sigma^2,M} = \{\bZ^{(m,l)}_{\sigma^2}\}_{m,l=1}^{M,L},
\end{equation}
so we face an inverse problem here, instead of a classical regression problem. Thanks to the linearity of the inverse problem, we can still derive a Representer theorem \cite{owhadi2019operator} that helps clarify the role of the hyperparameters. 

\begin{assumption}\label{assump1}
We assume that $K^E$ and $K^A$ are two Mercer kernels defined on $[0,R]\times [0,R]$ for some $R>0$. The true interaction functions $\intkernele \in \mHe$, $\intkernela \in \mHa$, and $$\kappa^2_{E}=\mathrm{sup}_{r\in [0,R]} {\mKe}(r,r) <\infty,$$
$$\kappa^2_{A}=\mathrm{sup}_{r\in [0,R]} {\mKa}(r,r) <\infty.$$
\end{assumption}

\begin{theorem}[Representer theorem]\label{representerthm}
Let $K^E$ and $K^A$ be two Mercer kernels that satisfy Assumption \eqref{assump1}. Given the training data $\{\bbY_{M},\bbZ_{\sigma^2,M}\}$, if the priors $\intkernele \sim \mathcal{GP} (0, \tilde K^E)$, $\intkernela \sim \mathcal{GP} (0, \tilde K^A)$ with $ \tilde K^E=\frac{\sigma^2 K^E} {MNL\lambda^E}$, $ \tilde K^A=\frac{\sigma^2 K^A} {MNL\lambda^A}$ for some $\lambda^E, \lambda^A>0$, then the posterior mean $\bar\bintkernel = (\bar\phi^E,\bar\phi^A)$ in \eqref{eq:estimated phi} coincides with the minimizer of the regularized empirical risk functional $\mE^{\lambda,M}(\cdot)$ on $\mHe\times\mHa$ where $\mE^{\lambda,M}(\cdot)$ is defined by
\begin{align}
\label{regularizedrisk1}
\mE^{\lambda,M}(\bintkernelvar):&=\frac{1}{LM}\sum_{l,m=1}^{L,M}\| \rhsfo_{\bintkernelvar}(\bY^{(m,l)})-\bZ_{\sigma^2}^{(m,l)}\|^2+\lambda^E \|\intkernelvare\|_{\mHe}^2+\lambda^A \|\intkernelvara\|_{\mHa}^2.
\end{align}
where $\lambda = \{\lambda^E, \lambda^A\}$ and the estimator $\bar\bintkernel \in \mHe\times\mHa$ can also be represented by 
\begin{equation}\label{solutionformat}
  \bar\bintkernel= (\sum_{r^x \in r_{\bbX_M}} \hat c_{r^x} K_{r^x}^E, \sum_{(r^x,r^v) \in (r_{\bbX_M} \times r_{\bbV_M})} \hat c_{r^v} K_{r^x}^A),
\end{equation}
with 
\begin{eqnarray}
\label{eq:solution1}
\hat{\mathbf{c}}_{r^x}= \frac{1}{N}\mbf{r}_{\bbX_M}^T \cdot  (K_{\rhsfo_\bintkernel}(\bbY_M,\bbY_M) + \lambda^E N ML I_{dNML})^{-1}\bbZ_{\sigma^2,M},\notag\\
\hat{\mathbf{c}}_{r^v}= \frac{1}{N}\mbf{r}_{\bbV_M}^T \cdot  (K_{\rhsfo_\bintkernel}(\bbY_M,\bbY_M) + \lambda^A N ML I_{dNML})^{-1}\bbZ_{\sigma^2,M},
\end{eqnarray}
where $\hat{\mathbf{c}}_{r^x}$, $\hat{\mathbf{c}}_{r^v}$ are the vectorizations of $(\hat c_{r^x})_{r^x \in r_{\bbX_M}}$ and   $(\hat c_{r^v})_{r^v \in r_{\bbV_M}}$ respectively, $r_{\bbX_M}\times r_{\bbV_M} \in \mathbb{R}^{MLN^2\times MLN^2}$ is the set containing all the pairwise distances in $\bbX_{M}$ and their associated pairwise distances in $\bbV_{M}$ as defined in \eqref{eq:pairdistances}, $\mbf{r}_{\bbX_M}$ is the block-diagonal matrix defined by $\mathrm{diag}(\mbf{r}_{\bX^{(m,l)}}) \in \mathbb{R}^{MLdN \times MLN^2}$ and $\mbf{r}_{\bX^{(m,l)}} = \mathrm{diag}(\{[\mbf{r}_{i1}^{\bX^{(m,l)}}, \dots, \mbf{r}_{iN}^{\bX^{(m,l)}}]\}_{i=1}^N) \in \mathbb{R}^{dN\times N^2}$, %
similarly for $\mbf{r}_{\bbV_M}$.
\end{theorem}

\JF{}{Detailed proof of \cref{representerthm} is shown in Appendix \cref{secA2}.} From the theorem, it is clear how hyperparameters affect the prediction of interaction kernels: $\theta^E$, $\theta^A$, and $\sigma$ jointly affect the choice of Mercer kernels and regularization constant, which becomes quite crucial in real data applications (see Figure \ref{fig:ex_real_compare2}). In \eqref{eq:solution1}, we also see that the posterior mean estimator $\phi^E$ lies in the span of basis functions with indices determined by the pairwise distances, and their coefficients are correlated with the basis functions. This is an effect imposed by the structure of the governing equation encoded in $\rhsfo_{\bintkernel}$.

\subsection{Well-posedness}\label{subsec:wellposedness}  We are concerned with the nonparametric learning of interaction kernels. That is, we do not assume the parametric form of interaction kernels. In this case, one can not expect to recover the true interaction kernels from finite data as they live in infinite dimensional spaces. Therefore, it is important to ensure one can asymptotically identify the true interaction kernels as the number of observational data snapshots goes to infinity. Otherwise, the empirical estimators from finite data  will have limited value as a scientific and predictive tool. Mathematically, we study the well-posedness under a statistical inverse problem setting. We introduce a linear operator $A: \mHe \times \mHa \rightarrow L^2(\mathbb{R}^{2dN};\rho_{\bY};\mathbb{R}^{dN})$ defined by 
\begin{align}\label{equationinverse}
A\bintkernelvar=\rhsfo_{\bintkernelvar},
\end{align} 
where $\rhsfo_{\bintkernelvar}$ is the right hand side of system \eqref{secondorder} by replacing $\bintkernel$ with $\bintkernelvar$, and $\rho_{\bY}$ is the limiting measure on $\mathbb{R}^{2dN}$ that we assume the observational data are sampled i.i.d from. For example,  if we assume that the initial condition of each trial is sampled i.i.d from a measure, then 
\begin{align}
    \rho_{\bY}(S) = \lim_{M \to \infty} \frac{1}{M} \sum_{m,l=1}^{M,L}\mathbb{1}_{\bY^{(m,l)} \in S}
\end{align}
for any Borel set $S \subset \mathbb{R}^{2dN}$ and the limit does exist in the weak sense by the law of large numbers. We denote the marginal probability measures for $\bX$ and $\bV$ by $\rho_{\bX}$, $\rho_{\bV}$ respectively.

Then  the well-posedness of \eqref{equationinverse} is reduced to studying under which conditions $A$ has a bounded inverse.

\subsubsection{Well-posedness on an $L^2$ space} We first consider the embedding of $\mHe\times \mHa$ to a suitable $L^2$ space and consider the well-posedness in a weaker $L^2$-norm. Motivated by \eqref{solutionformat} in the Representer theorem, we consider the measures $\tilde\rho_{r}^{E}$, $\tilde\rho_{r}^{A}$ for $\intkernele$, $\intkernela$ based on the structure of $\rhsfo_{\bintkernelvar}$,
\begin{align}\label{def1}
    \tilde\rho_{r}^{E}(Q) = \int_{Q} \int_{\mathbb{R}^{dN}} \frac{1}{N(N-1)} \sum_{i \neq j} \delta_{r_{ij}^{\bx^\ast}}(r)\cdot(r_{ij}^{\bx^\ast})^2 d \rho_{\bX}(\bX^\ast) dr
\end{align}
\begin{align}\label{def2}
    \tilde\rho_{r}^{A}(Q) = \int_{Q} \int_{\mathbb{R}^{2dN}} \frac{1}{N(N-1)} \sum_{i \neq j} \delta_{r_{ij}^{\bx^\ast}}(r)\cdot(r_{ij}^{\bv^\ast})^2 d \rho_{\bY}(\bX^\ast,\bV^\ast) dr
\end{align}
for any set $Q \subset [0,R]$, and $\delta(\cdot)$ is the Dirac $\delta$ distribution. By the continuity, $\mHe\times\mHa$ can be naturally embedded as a subspace of  $L^2([0,R]\times [0,R];\tilde\rho_{r};\mathbb{R}\times\mathbb{R})$ with $\tilde\rho_{r}= \tilde\rho_{r}^{E} \times \tilde\rho_{r}^{A}$.  One can follow the proof of Proposition 9 in \cite{learning2022} to show that $A$ is a bounded linear operator from  $L^2([0,R]\times [0,R];\tilde\rho_{r};\mathbb{R}\times\mathbb{R})$ to 
$L^2(\mathbb{R}^{2dN};\rho_{\bY};\mathbb{R}^{dN})$.

 Now we can introduce a sufficient condition to guarantee the existence of a bounded inverse of $A$ on $L^2([0,R]\times [0,R];\tilde\rho_{r};\mathbb{R}\times\mathbb{R})$, called the coercivity condition:
\begin{definition}\label{coercivityrkhs}
 We say that the system \eqref{secondorder} satisfies the coercivity condition if $ \forall \bintkernelvar \in \mHe \times \mHa$,
\begin{align}\label{coercivity}
\|A\bintkernelvar\|^2_{L^2(\rho_{\bY})}=\|\rhsfo_{\bintkernelvar}\|^2_{L^2(\rho_{\bY})}\geq c_{\mHe}\|\intkernelvare\|^2_{L^2( \tilde\rho_{r}^{E})} + c_{\mHa}\|\intkernelvara\|^2_{L^2( \tilde\rho_{r}^{A})}
\end{align}
for some constants $c_{\mHe}, c_{\mHa} > 0$.
\end{definition}

Here we show one example  to support the coercivity condition. 

\begin{theorem}\label{2ndordersingle:coercivity}
Consider  $\rho_{\bY}=\begin{bmatrix}\rho_{\bX}\\ \rho_{\bV} \end{bmatrix}$, where $\rho_{\bX}$  is the product of $N$ independent and identical measures with compact support on $\mathbb{R}^d$, and $\rho_{\bV}$ is defined in the same way and is independent of $\rho_{\bX}$.  Then we have 

\begin{align}\label{coercivityidex}
\|\rhsfo_{\bintkernelvar}\|^2_{L^2(\rho_{\bY})}\geq \frac{N-1}{N^2}\|\intkernelvare\|^2_{L^2( \tilde\rho_{r}^{E})} + \frac{N-1}{N^2}\|\intkernelvara\|^2_{L^2( \tilde\rho_{r}^{A})}
\end{align}
 \end{theorem}

\JF{}{Detailed proof of \cref{2ndordersingle:coercivity} is shown in Appendix \cref{secB2}.} In \cite{miller2020learning}, the identifiability of a structured sum of $\intkernele$ and $\intkernela$ is studied. Here we consider a stronger version of identifiability as we want to individually recover $\intkernele$ and $\intkernela$. Note that it is also possible for   distributions on $\mathbb{R}^{dN}$ with non-i.i.d $\mathbb{R}^d$ components that satisfy the coercivity condition. Finally, we remark that the coercivity condition \eqref{coercivity} holds on measure pairs $(\rho_1,\rho_2)$ equivalent to $(\tilde\rho_{r}^{E},\tilde\rho_{r}^{A})$. This can provide us with many nontrivial examples from the special case in Theorem \ref{2ndordersingle:coercivity}. We conjecture that the coercivity condition is generally satisfied and leave further investigation as future work. 

\subsubsection{Well-posedness on {$\mHe\times\mHa$} and the convergence analysis} Now we turn to study the well-posedness on $\mHe\times\mHa$ with the stronger RKHS norm, and we make the following assumption. 

\begin{assumption}\label{assump2}
We assume that $\tilde\rho_{r}$ is non-degenerate on $[0,R]\times [0,R]$. 
\end{assumption}

We remark that the above assumption is mild. For example, we can pick $\rho_{\bY}$ to be a uniform measure supported on a large enough cube, then $\tilde\rho_{r}$ satisfies the assumption. 

It is straightforward to see that the coercivity condition implies injectivity of $A$ on  $\mHe\times\mHa$: $\bintkernelvar=0$ everywhere on $[0,R]$ when $A\bintkernelvar=0$ for $\bintkernelvar \in \mHe\times\mHa$. This is due to the non-degeneracy of $\tilde\rho_{r}$ on $[0,R]\times [0,R]$ and the continuity of $\bintkernelvar$. Therefore, $A$ is injective. However, showing $A$ has a bounded inverse on $\mHe\times\mHa$ is impossible when it is infinitely dimensional, as $A$ is a compact operator. Suppose the coercivity condition \eqref{coercivity} holds, then following the theoretical framework developed in \cite{learning2022},  one could prove the well-posedness on a suitable subspace determined by the source conditions on $\intkernele,\intkernela$ following inverse problem literature. In this case, it is possible to prove one could recover both kernels with a statistically optimal rate under the corresponding RKHS norm. \JF{}{We obtained the result for the single-kernel case  in our recent work \cite{learning2022}, and we leave the work for the double-kernel case for the future investigation.}

\section{Numerical examples} \label{sec:numerics}
In this section, we investigate the performance of the algorithm proposed in \cref{sec: Methodology} to show the effectiveness of model selection in \eqref{eq:2ndOrder}. Specific instances of \eqref{eq:2ndOrder} have found many applications in modeling the clustering, swarming, and alignment behaviors of collective agents. The examples include (1) Cucker-Smale dynamics (CS) with friction force ($m_i \equiv 1$, $\intkernele \equiv 0$, $\intkernela \neq 0$) in \cref{ex: CS}, (2) fish milling dynamics (FM) with friction force ($m_i \equiv 1$, $\intkernele \neq 0$, $\intkernela \equiv 0$) in \cref{ex: FM}, (3) anticipation dynamics (AD) ($m_i \equiv 1$, $\intkernele, \intkernela \neq 0$) in \cref{ex: AD} and (4) opinion dynamics (OD) with stubborn agents ($m_i\equiv0$, $\intkernele \neq 0$, $\intkernela \equiv 0$) in \cref{ex: ODS}. In (1)-(3), the mass of agents is known in advance, i.e, they are second-order systems. We are interested in learning  $\intkernele$, $\intkernela$, and other hyperparameters $\mbf{\alpha}$ from data, resulting in the selection of types of interactions (energy versus alignment interactions). In (4), we used the prior knowledge that $\intkernela\equiv 0$ and investigate if the true zero mass of the opinions  and $\intkernele$ can be learned from data, resulting in the selection of the order of the system (first versus second order).

The detailed setups of each dynamic are shown in \cref{tab:ex_info}. We applied the strategies proposed in \cref{sec: Methodology} to learn $\balpha$ in $\force_{\balpha}$, and the interaction kernels $\intkernele(r)$, $\intkernela(r)$. We initialize the parameters in $\balpha$ randomly from the uniform distribution $\mathcal{U}([0,1])$, and the same for $\sigma$ in the cases with noisy data. In each experiment, we run 10 independent trials and report the errors of the estimations for $\balpha$, the estimation errors for $\intkernele$, $\intkernela$ in the (relative) $L^\infty([0,R])$-norm,
and compare the discrepancy between the true trajectories (evolved using $\alpha$, $\intkernele$, $\intkernela$) and predicted trajectories (evolved using $\hat{\alpha}$, $\hat{\intkernele}$, $\hat{\intkernela}$) on both the training time interval $[0, T]$ and on the future time interval $[T, T_f]$, over two different sets of initial conditions (IC) – one taken from the training data, and one consisting of new samples from the same initial distribution.

\paragraph{Real data application} We also apply our method to two real datasets of fish in \cref{ex: realdata}, where one shows a flocking behavior and another shows a milling behavior. We fit them into the Cucker-Smale and fish milling dynamics respectively and perform comparisons with two other classical approaches: SINDy \cite{brunton2016discovering} and feed-forward
neural networks.

\paragraph{Numerical Setup.} We simulate the trajectory data $(\bbY, \bbZ)$ on the time interval $[0,T]$ with given i.i.d initial conditions generated from the probability measures specified for each system as shown in \cref{tab:ex_info}. For the training data sets, we generate $M$ trajectories and observe each trajectory at $L$ equidistant times $0 = t_1 < t_2 < \cdots < t_L = T$ and add Gaussian noise to $\bbZ$ with level $\sigma$. We construct an empirical approximation to the probability measure $\tilde \rho_r$, with $2000$ trajectories and let $[0,R]$ be its support. All ODE systems are evolved using \textrm{ode$15$s} in MATLAB\textsuperscript{\textregistered} with a relative tolerance at $10^{-5}$ and absolute tolerance at $10^{-6}$. For noise-free training data, we add a jitter constant $\approx 10^{-6}$ as a way of regularization. We apply the \textit{minimize} function in the GPML package\footnote{Carl Edward Rasmussen \& Hannes Nickisch (http://gaussianprocess.org/gpml/code)} to train the parameters using conjugate gradient optimization with the partial derivatives shown in \cref{sec: Methodology}, and set the maximum number of function evaluations to 400. 

In almost all examples, we use the full GP methods, as we use scarce data and there is no need for acceleration.  However, we show the effectiveness of our acceleration techniques in Fish milling dynamics in  section \ref{accelNumerical} when we have a larger scale of data. 

\begin{table}[tbhp]
\caption{System parameters in the dynamics}
\label{tab:ex_info} 
\begin{center}
\resizebox{\linewidth}{!}{%
\begin{tabular}{ c | c | c | c | c }
\hline
 System & CS & FM & AD & OD \\
 \hline 
 $d$ & 2 & 2 & 2 & 1 \\
 \hline 
 $N$ & 10 & 10 & 10 & 5 \\
 \hline 
 $m_i$ & 1 & 1 & 1 & 0\\
 \hline 
$[0; T; T_f]$ & $[0;10;20]$ & $[0;5;10]$ & $[0;10;20]$ & $[0;2;20]$\\
\hline 
$\mu_0^\bx$ & $\mathrm{Unif}([-2,2]^2)$ & $\mathrm{Unif}([-0.5,0.5]^2)$ & $\mathrm{Unif}([0,5]^2)$ & $\mathrm{Unif}([-1,1])$ \\
\hline 
$\mu_0^\bv$ & $\mathrm{Unif}([-1,1]^2)$ & $\mathrm{Unif}([0,0]^2)$ & $\mathrm{Unif}([0,5]^2)$ & - \\
\hline 
$\intkernele$ & 0 & $\frac{1}{r}\bigg[- e^{-2r}+e^{-\frac{r}{4}}\bigg],$ &  $\frac{0.1}{(1+r)^{2.5}} + \frac{1}{(1+r)^{0.5}}$ & \cref{eq:ODS_phi} \\
\hline 
$\intkernela$ & $\frac{1}{(1 + r^2)^{1/4}}$ & 0 & $ \frac{0.1}{(1+r^2)^{0.5}}$ & 0\\
\hline 
$\force(\bx_i, \dot\bx_i,\mbf{\alpha})$ & $\kappa\dot\bx_i(1-\|\dot\bx_i\|^p)$ & $(\gamma-\beta \|\dot\bx_i\|^2)\dot\bx_i$ & 0 & \cref{eq:ODS_force}\\
\hline
$\balpha$ & $(\kappa, p) = (1,2)$ & $ (\gamma, \beta) = (1.5,0.5)$ & - & $(P_1,\kappa) = (1,10)$\\
\hline
\end{tabular}}
\end{center}
\end{table}

\paragraph{Choice of the covariance function.} We choose the Mat\'{e}rn covariance function defined on $[0,R] \times [0,R]$ for the Gaussian process priors in our numerical experiments, i.e.,
\begin{equation}
\label{eq:Matern kernel}
    K_\theta(r,r')=s_\intkernel^2 \frac{2^{1-\nu}}{\Gamma(\nu)}(\frac{\sqrt{2\nu}\|r-r'\|}{\omega_{\intkernel}})^\nu B_\nu(\frac{\sqrt{2\nu}\|r-r'\|}{\omega_{\intkernel}}),
\end{equation}
where the parameter $\nu > 0$ determines the smoothness; $\Gamma(\nu)$ is the Gamma function; $B_\nu$ is the modified Bessel function of the second kind; and the hyperparameters $\theta = \{s_\intkernel^2, \omega_{\intkernel}\}$ quantify the amplitude and scale. In our numerical examples, we choose $\nu = p + 1/2$ with $p=0 \text{ or } 1$. 

 The Reproducing Kernel Hilbert Space (RKHS), $\mathcal{H}_{Mat\acute{e}rn}$, associated with this Mat\'{e}rn kernel is norm-equivalent to the Sobolev space $W_2^{\nu + 1/2}([0,R])$ defined by 
\begin{equation}
    W_2^{\nu + 1/2}([0,R]) : = \Big\{ f \in L_2([0,R]): \|f\|^2_{W_2^{\nu + 1/2} } := \sum_{\beta \in \mathbb{N}_0^1:|\beta|\leq \nu + 1/2} \|D^\beta f\|_{L_2}^2 < \infty  \Big\}.
\end{equation}
That is to say,  $\mathcal{H}_{Mat\acute{e}rn} = W_2^s([0,R]) $ as a set of functions, and there exists constants $c_1,c_2>0$ such that 
\begin{equation}
    c_1\|f\|_{W_2^{\mu+\frac{1}{2}}} \leq \|f\|_{\mathcal{H}_{Mat\acute{e}rn}} \leq c_2||f||_{W_2^{\mu+\frac{1}{2}}}, \quad \forall f \in \mathcal{H}_{Mat\acute{e}rn}.
\end{equation}
In other words, $\mathcal{H}_{Mat\acute{e}rn}$ consists of functions that are differentiable up to order $\nu$ and weak differentiable up to order $s = \nu + \frac
{1}{2}$.

\begin{figure}[tbhp]
\centering
\subfigure[CS: $\intkernele$]{
\includegraphics[width=0.29\linewidth,height=0.19\linewidth]{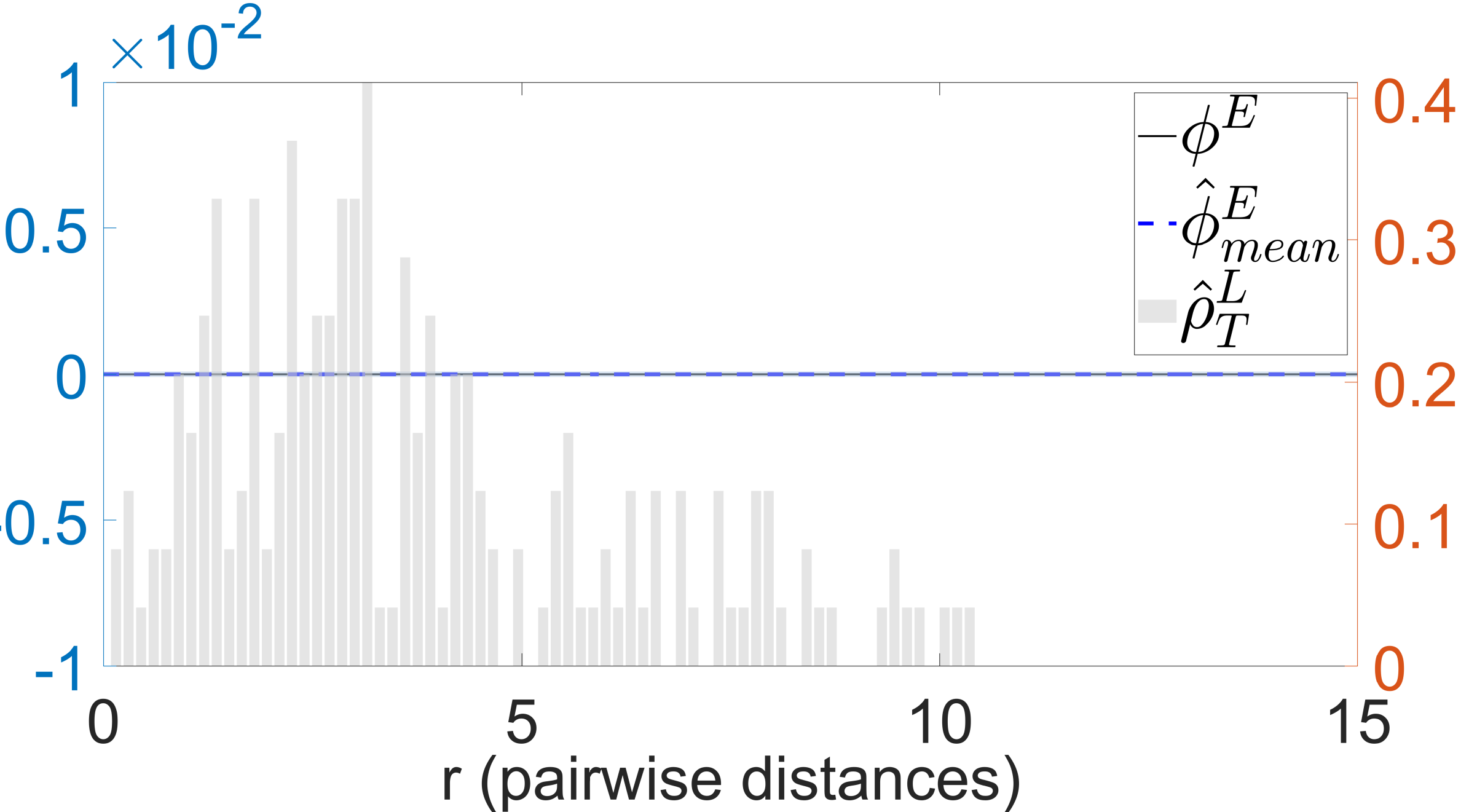}
}
\subfigure[CS: $\intkernela$]{
\includegraphics[width=0.29\linewidth,height=0.19\linewidth]{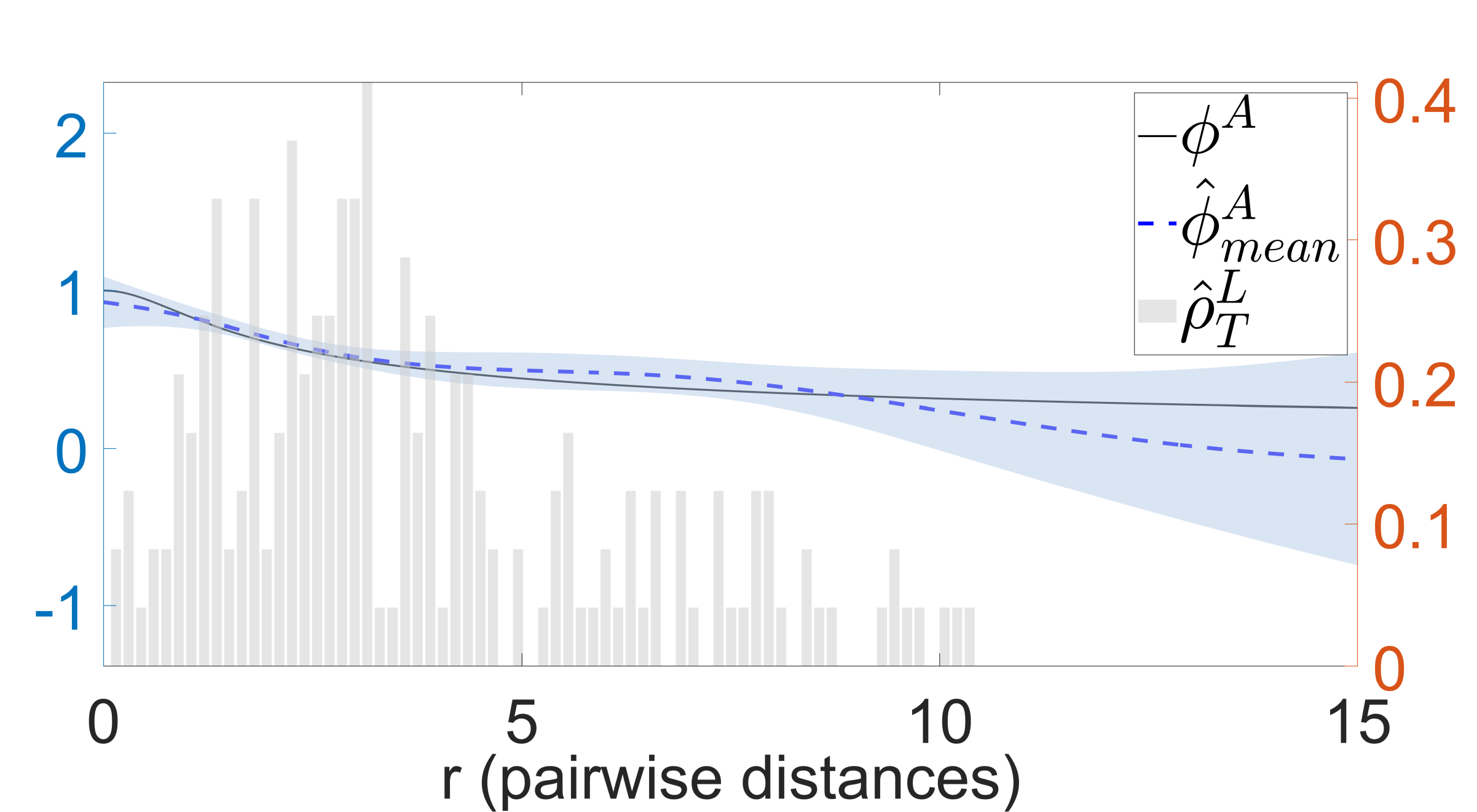}
}
\subfigure[CS: trajectory prediction]{
\includegraphics[width=0.33\linewidth]{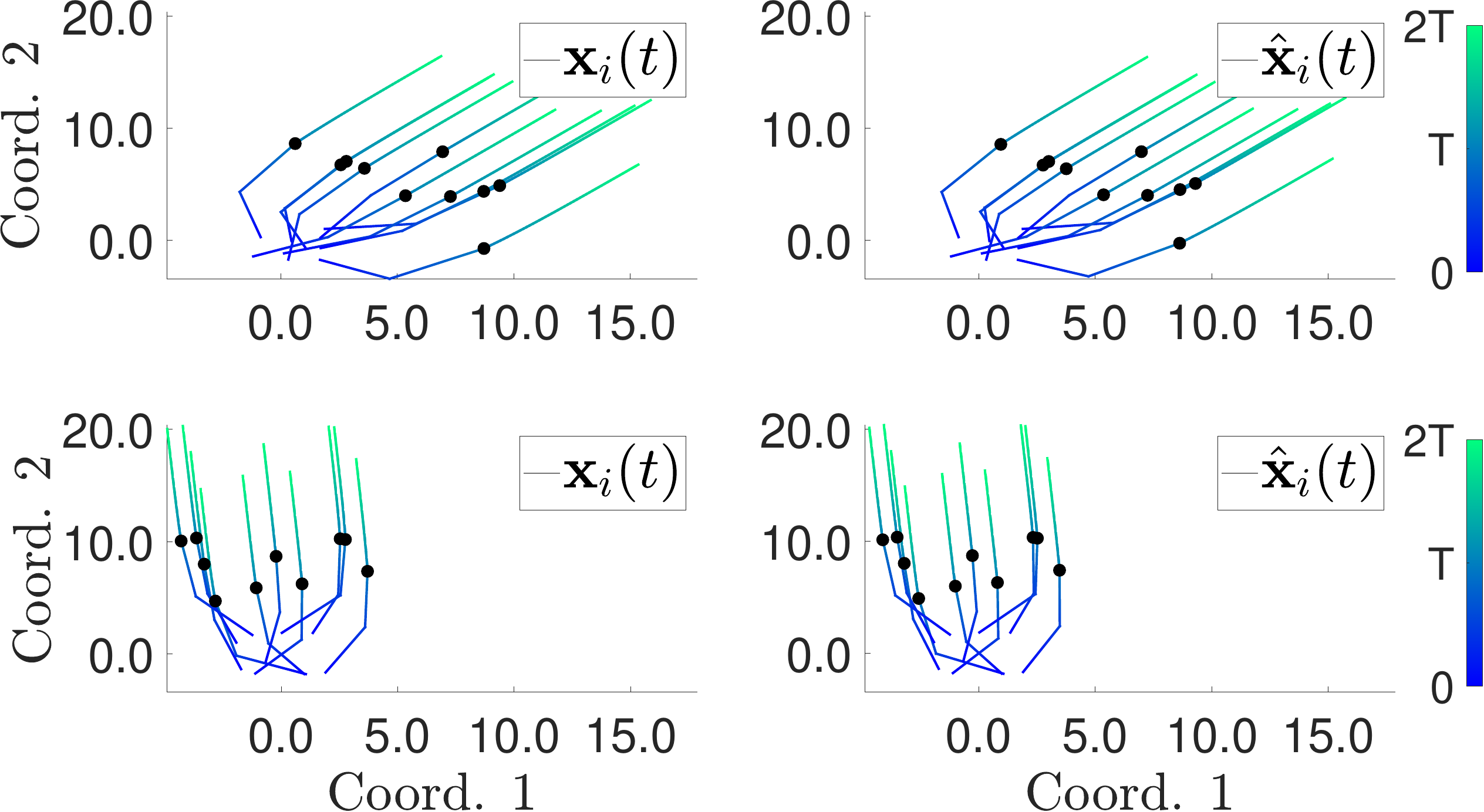}
}
\subfigure[FM: $\intkernele$]{
\includegraphics[width=0.29\linewidth,height=0.19\linewidth]{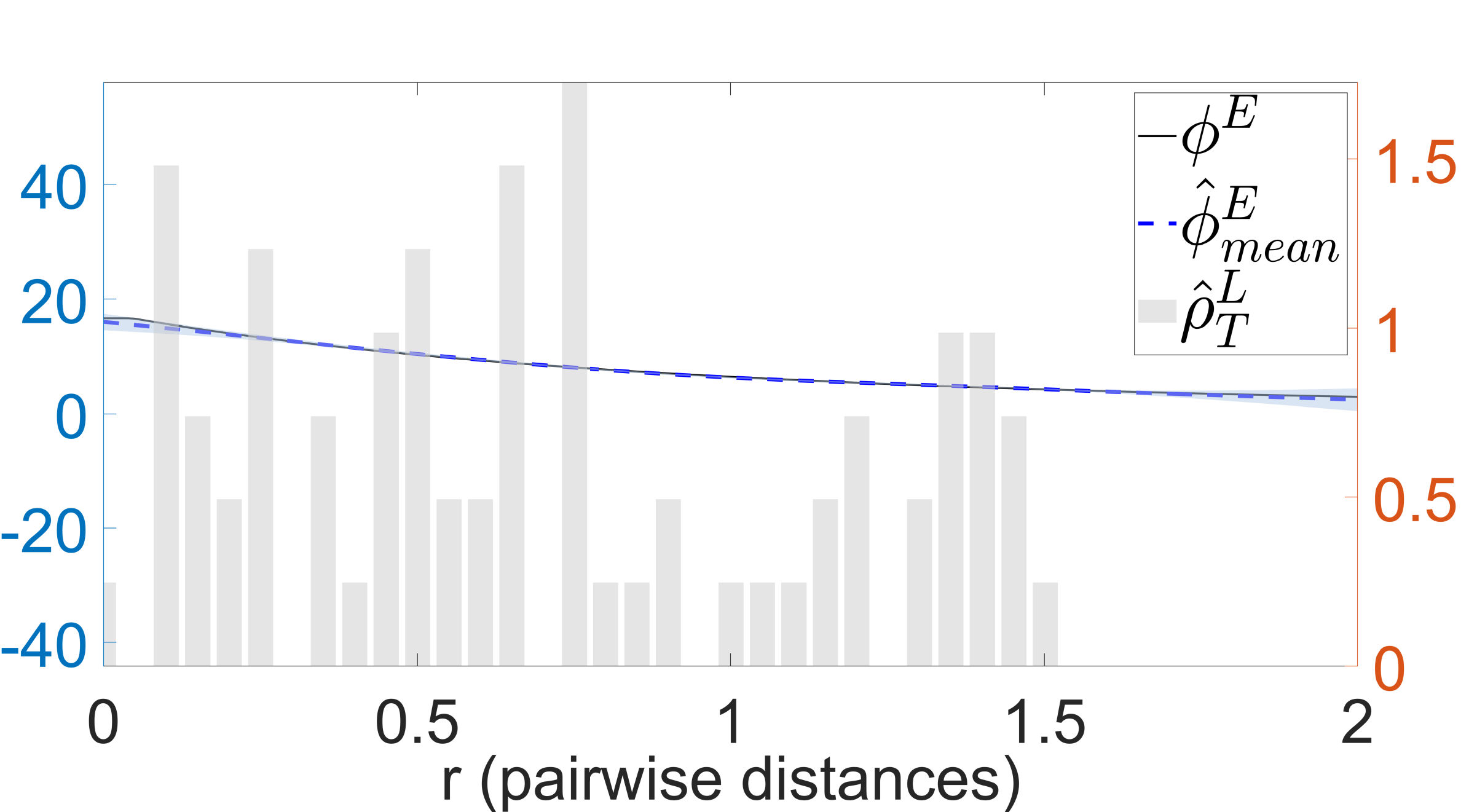}
}
\subfigure[FM: $\intkernela$]{
\includegraphics[width=0.29\linewidth,height=0.19\linewidth]{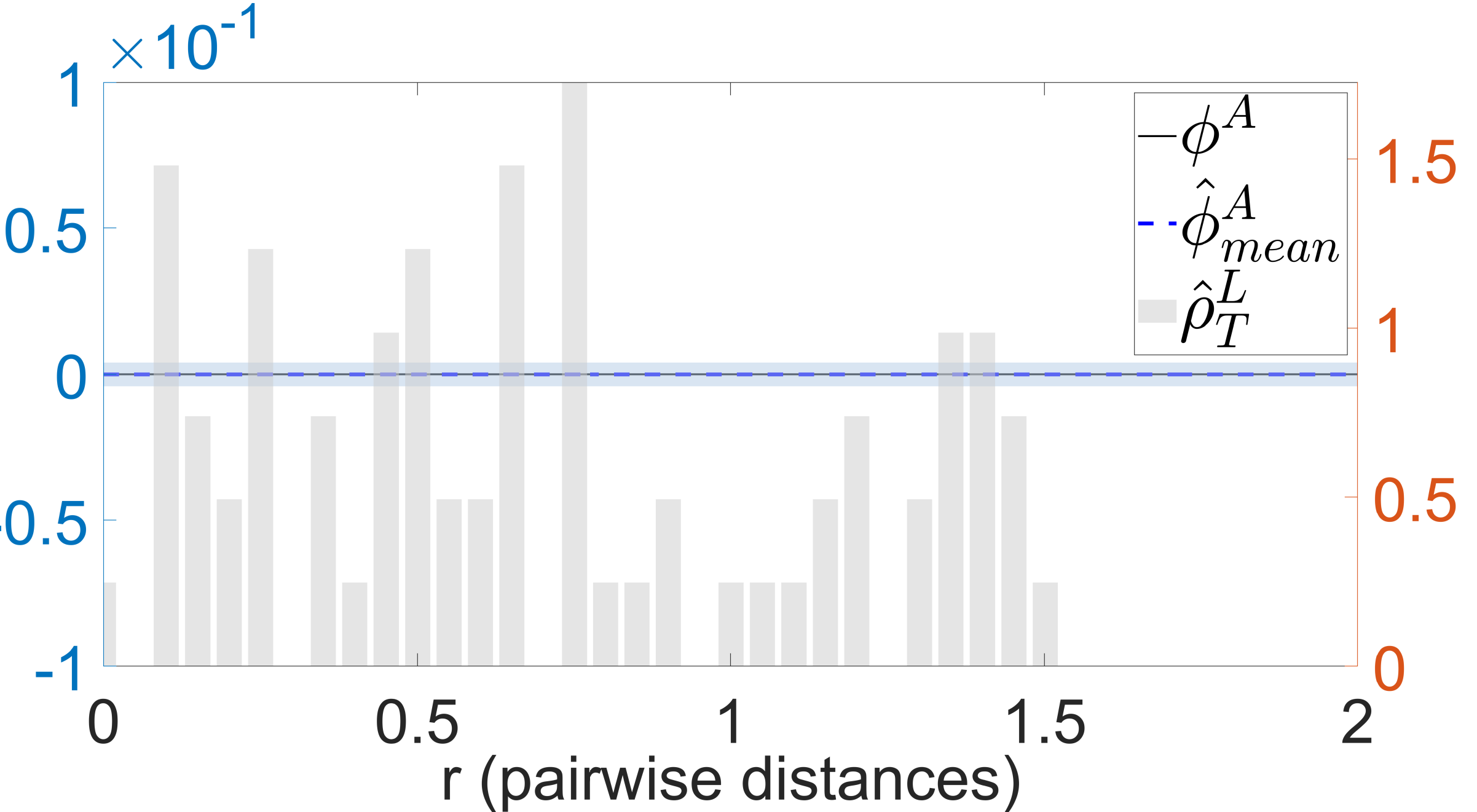}
}
\subfigure[FM: trajectory prediction]{
\includegraphics[width=0.33\linewidth]{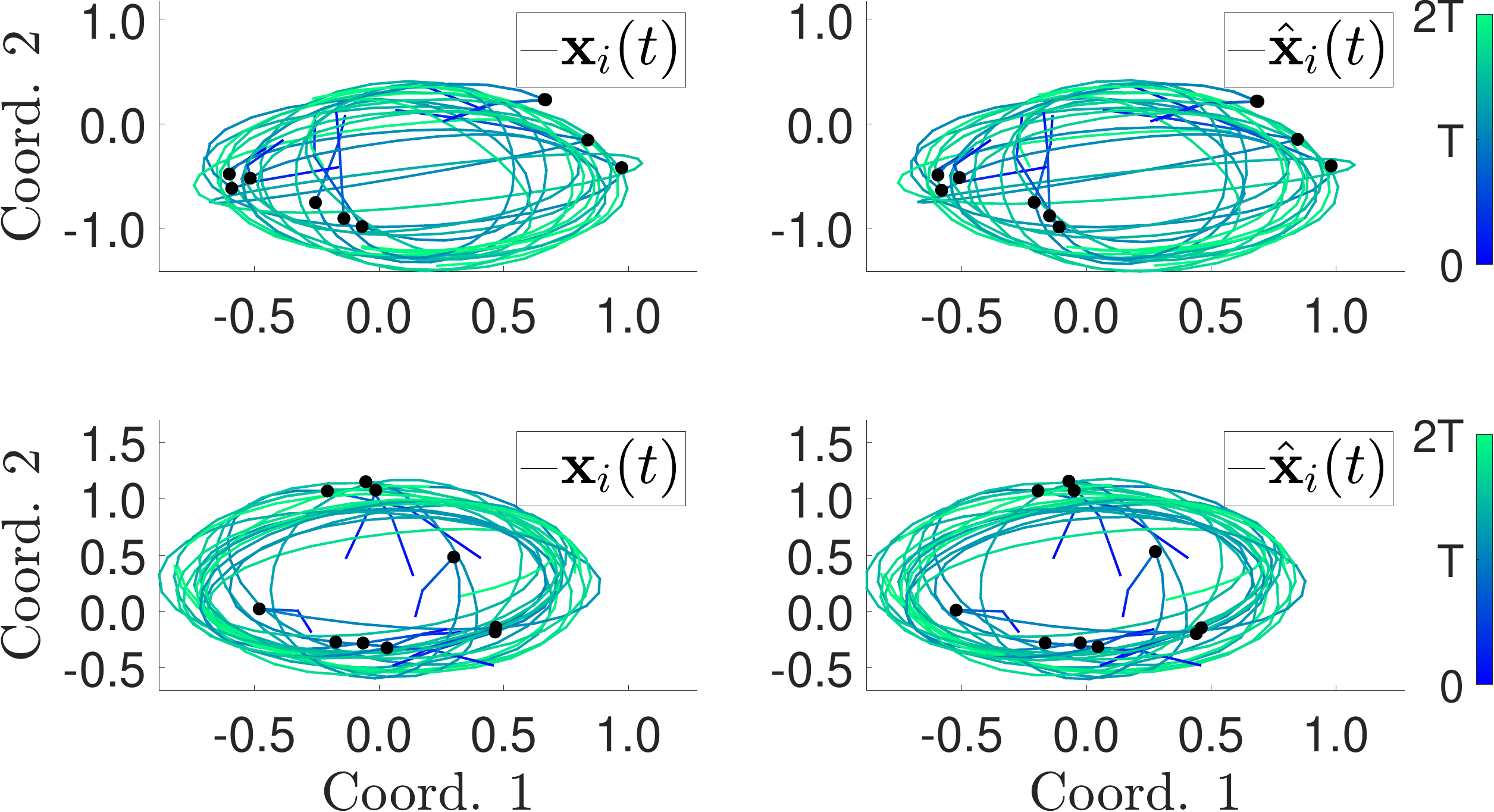}
}
\subfigure[AD: $\intkernele$]{
\includegraphics[width=0.29\linewidth,height=0.19\linewidth]{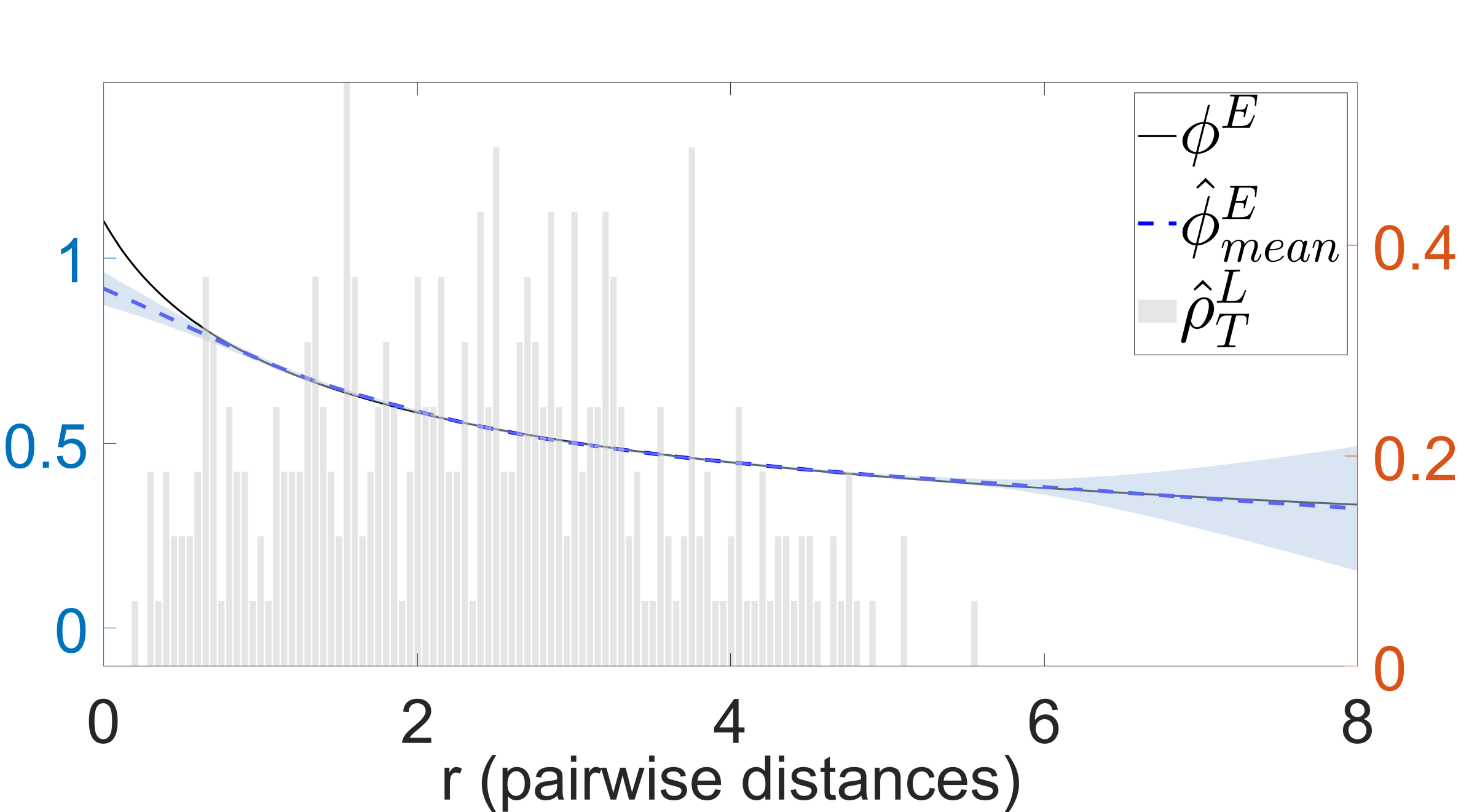}
}
\subfigure[AD: $\intkernela$]{
\includegraphics[width=0.29\linewidth,height=0.19\linewidth]{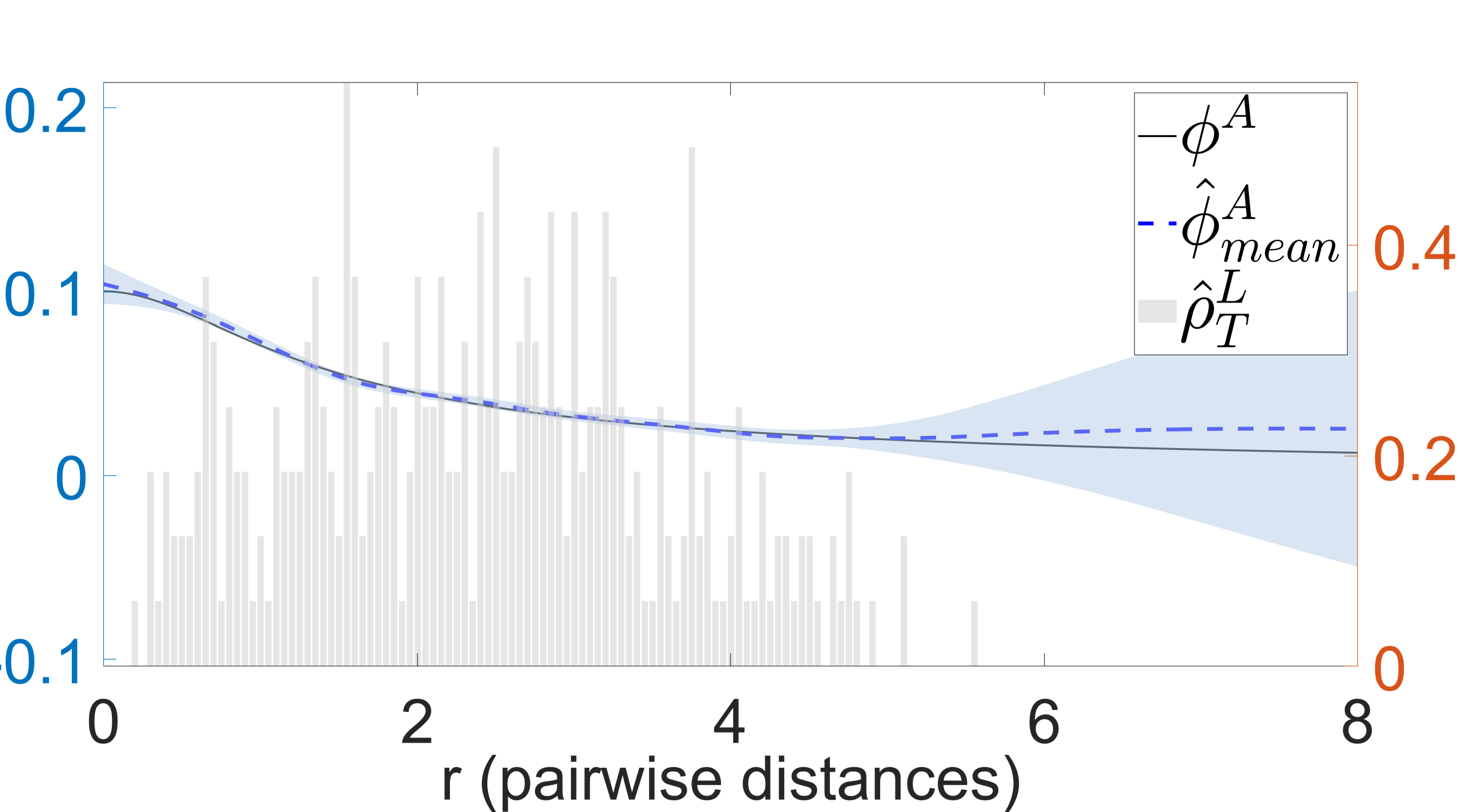}
}
\subfigure[AD: trajectory prediction]{
\includegraphics[width=0.33\linewidth]{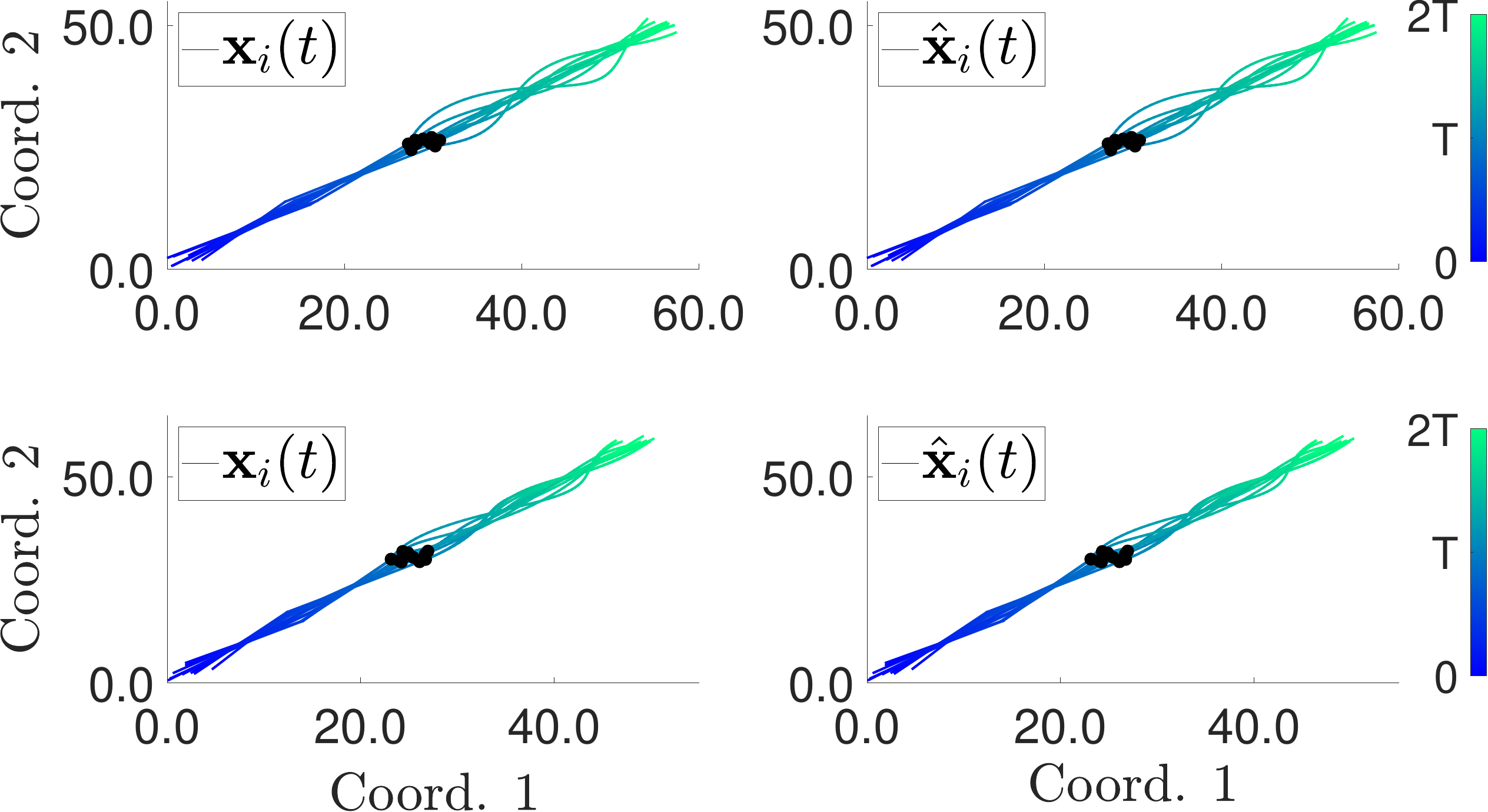}
}
\caption{Results of learning different dynamics using the Matérn kernel. 
Top: Learning CSF ($\{N,M,L,\sigma\}  =\{10,6,3,0.1\}$);
Middle: Learning FM ($\{N,M,L,\sigma\} = \{10,3,3,0.1\}$);
Bottom: Learning AD ($\{N,M,L,\sigma\} = \{10,12,3,0.01\}$).
Left, Center: Predictive mean $\hat{\intkernele}$ and $\hat{\intkernela}$ of the true kernels, and two-standard-deviation band (light blue color) around the means. The grey bars represent the empirical density of the $\tilde \rho_r$;
Right: the true (left) versus predicted (right) trajectories using $\hat{\balpha}$ and $\hat{\intkernel}$  with initial conditions of training data (top) and testing data (bottom) }
\label{fig:ex_results}
\end{figure}

\subsection{Summary of the numerical experiments}

\begin{itemize}

\item The proposed learning approach performs \textit{simultaneous} precise model selections from \textit{small} amounts of \textit{noisy} observation data. The numerical results in all different dynamics show that the algorithm can accurately identify the existence of energy-based/alignment-based interactions and can learn order information of dynamics between agents in the systems.

\item The GP method selects  a  kernel basis to represent the underlying sparse dynamics that generalizes remarkably well in  larger time prediction with new initial conditions.  The occasional larger prediction errors that occur in a larger time interval may be caused by the propagation of estimation errors. We believe the performance is satisfactory since we only have very limited and noisy training data. Even in cases where the prediction errors are relatively large, the estimators can predict remarkably accurate collective behaviors of the agents, e.g. the consensus in the opinion dynamics, the flocking behavior in the Cucker-Smale dynamics, and the milling pattern in the fish milling dynamics.

\item In synthetic experiments, the uncertainty quantification band for the trajectories is rather small ($\mathcal{O}(10^{-3})$), resulting from the narrow uncertainty bands of $\bintkernel$. In real data experiments, we found models using interaction kernels sampled from uncertainty bands  all reproduced the true dynamics very well.

 \item The real data experiments show that the proposed GP approach combined with the particle-based models is practically applicable, and  outperformed  two other competitors in preserving the physics of the true dynamics. 

\end{itemize}

\subsection{Model selection for types of interaction kernels}

\subsubsection{Cucker-Smale dynamics with friction force}
\label{ex: CS}

The Cucker-Smale system \cite{cucker2007emergent,cucker2011general,shvydkoydynamics} is used to model collective behaviors in a system of agents that follow a prescribed protocol of communication, such as wedges of bird flocks, lattices in cell organization, or bee hives \cite{ahn2012collision,choi2017emergent,chuang2007state}. We consider the system of $N$ agents in the form \eqref{eq:2ndOrder} with components defined in \cref{tab:ex_info}, where $\intkernela$ is a communication kernel, or influence function, that makes the agents flock, and $\force_{\balpha}$ a Rayleigh-type friction force that pushes all magnitudes of the velocities $\|v_i\|$ towards the same value 1 and counteracts the directional alignment forces governed by $\intkernela$ to  produce  a rich variety of collective dynamics depending on the relative strengths of the involved forces.

In this example, the unknown parameters $\balpha=(\kappa,p)$ are nonlinear with respect to the system. We show the errors of our estimation for $\balpha$ and $\intkernele$, $\intkernela$ in \cref{tab:ex_CSF}.
Note that for this model, $\intkernela$ is in the RKHS generated by the Mat\'{e}rn kernel we pick. The estimated interaction kernel $\hat{\intkernel}^A$ can recover the true $\intkernela(r)$ almost perfectly in the region within the support of the empirical $\tilde \rho_r$ from both noise-free and noisy training data. Moreover, the true interaction kernel $\intkernela(r)$ is fully covered in the uncertainty region we constructed using the posterior variances. \cref{tab:ex_CSF} also shows that our method can identify the nonexistence of the energy-based interaction well with small errors (at most $O(10^{-4})$) from zero in $L^\infty([0,R])$-norm. See also in \cref{fig:ex_results}(a),(b). The errors for the predicted trajectories are shown in \cref{tab:CSF_traj}. We can see that in both the training time interval $[0,10]$ and future time interval $[10,20]$, the estimators can produce accurate approximations of the true trajectories and the performance becomes better when we increase the size of training data ($M$ or $L$).

\begin{table}[tbhp]
\caption{Means and standard deviations of the  errors of $\hat\balpha$ (including $\hat \sigma$ when noise exists) and $\hat{\intkernel}$ for different settings of the CS dynamics.} 
\label{tab:ex_CSF}
\begin{center}
\resizebox{0.9\linewidth}{!}{%
\begin{tabular}{cccc}
\toprule
$\{N,M,L,\sigma\}$ & $\| \hat{\balpha} - \balpha\|_{\infty}$  & $\| \hat{\intkernele} - 0\|_{\infty}$
& $\| \hat{\intkernela} - \intkernela\|_{\infty}/ \| \intkernela\|_{\infty}$ 
\\
\cmidrule(lr){1-1}\cmidrule(lr){2-4}
$\{10,1,3,0\}$ & $1.9\cdot 10^{-3} \pm 1.0\cdot 10^{-3}$ & $2.1\cdot 10^{-5} \pm 4.0\cdot 10^{-5}$ 
& $5.6\cdot 10^{-2} \pm 1.5\cdot 10^{-2}$ 
\\
\cmidrule(lr){1-1}\cmidrule(lr){2-4}
$\{10,3,3,0\}$ & $\mathbf{1.1\cdot 10^{-3} \pm 7.9\cdot 10^{-4}}$ & $2.6\cdot 10^{-5} \pm 6.5\cdot 10^{-5}$ 
& $4.5\cdot 10^{-2} \pm 2.0\cdot 10^{-2}$ 
\\
\cmidrule(lr){1-1}\cmidrule(lr){2-4}
$\{10,6,3,0\}$ & $1.3\cdot 10^{-3} \pm 2.5\cdot 10^{-3}$ &
$\mathbf{1.1\cdot 10^{-5} \pm 1.3\cdot 10^{-5}}$ 
& $\mathbf{3.2\cdot 10^{-2} \pm 1.0\cdot 10^{-2}}$ 
\\
\cmidrule(lr){1-1}\cmidrule(lr){2-4}
$\{10,6,3,0.05\}$ & $1.1\cdot 10^{-1} \pm 1.1\cdot 10^{-1}$ &
$1.2\cdot 10^{-4} \pm 1.6\cdot 10^{-4}$ 
& $1.6\cdot 10^{-1} \pm 8.6\cdot 10^{-2}$ 
\\
\cmidrule(lr){1-1}\cmidrule(lr){2-4}
$\{10,6,3,0.1\}$ & $2.3\cdot 10^{-1} \pm 2.3\cdot 10^{-1}$ & $1.4\cdot 10^{-4} \pm 2.9\cdot 10^{-4}$ 
& $1.8\cdot 10^{-1} \pm 8.0\cdot 10^{-2}$ 
\\
\bottomrule
\end{tabular} %
}
\end{center}
\end{table}

\begin{table}[tbhp]
\caption{The trajectory prediction errors for different settings.} \label{tab:CSF_traj}
\begin{center}
\resizebox{\linewidth}{!}{%
\begin{tabular}{ccccc}
\toprule
$\{N,M,L,\sigma\}$ & Training IC $[0,10]$ & Training IC $[10,20]$ & new IC $[0,10]$ & new IC $[10,20]$\\
\cmidrule(lr){1-1}\cmidrule(lr){2-5}
$\{10,1,3,0\}$ & $4.9 \cdot 10^{-4} \pm 4.2\cdot 10^{-4}$  &  $6.7 \cdot 10^{-4} \pm 1.3\cdot 10^{-3}$ &  $1.8 \cdot 10^{-3} \pm 4.4\cdot 10^{-3}$ & $1.4 \cdot 10^{-2} \pm 4.2\cdot 10^{-2}$ \\
\cmidrule(lr){1-1}\cmidrule(lr){2-5}
$\{10,3,3,0\}$  & $2.5 \cdot 10^{-4} \pm 2.0\cdot 10^{-4}$  &  $1.5 \cdot 10^{-4} \pm 1.3\cdot 10^{-4}$ &  $4.9 \cdot 10^{-4} \pm 4.9\cdot 10^{-4}$ & $8.7 \cdot 10^{-3} \pm 1.7\cdot 10^{-2}$ \\
\cmidrule(lr){1-1}\cmidrule(lr){2-5}
$\{10,6,3,0\}$  & $\mathbf{1.5 \cdot 10^{-4} \pm 1.2\cdot 10^{-4}}$  &  $\mathbf{9.4 \cdot 10^{-5} \pm 9.2\cdot 10^{-5}}$ &  $\mathbf{2.7 \cdot 10^{-4} \pm 4.1\cdot 10^{-4}}$ & $\mathbf{2.3 \cdot 10^{-4} \pm 4.6\cdot 10^{-4}}$ \\
\cmidrule(lr){1-1}\cmidrule(lr){2-5}
$\{10,6,3,0.05\}$  & $2.3 \cdot 10^{-2} \pm 1.3\cdot 10^{-2}$  &  $1.9 \cdot 10^{-2} \pm 1.3\cdot 10^{-2}$ &  $2.7 \cdot 10^{-2} \pm 1.9\cdot 10^{-2}$ & $2.5 \cdot 10^{-2} \pm 2.0\cdot 10^{-2}$ \\
\cmidrule(lr){1-1}\cmidrule(lr){2-5}
$\{10,6,3,0.1\}$  & $4.2 \cdot 10^{-2} \pm 2.6\cdot 10^{-2}$  &  $3.8 \cdot 10^{-2} \pm 2.8\cdot 10^{-2}$ &  $4.9 \cdot 10^{-2} \pm 3.4\cdot 10^{-2}$ & $4.5 \cdot 10^{-2} \pm 3.9\cdot 10^{-2}$ \\
\bottomrule
\end{tabular} 
}
\end{center}
\end{table}

\subsubsection{Fish-Milling dynamics with friction force}
\label{ex: FM}
 In this subsection, we consider another type of cohesive collective system that produces milling patterns \cite{abaid2010fish,lukeman2009conceptual}. A special instance of such systems is the D'Orsogna model \cite{d2006self,chuang2007state,bhaskar2019analyzing}, which describes the motion of $N$  self-propelled particles powered by biological or mechanical motors, that experience a frictional force,  and can produce a rich variety of collective patterns.
We consider the system of $N$ agents of the form \eqref{eq:2ndOrder} with components defined in \cref{tab:ex_info}, where the interaction kernel $\intkernele$ is derived from the Morse-type potential. Since it is singular at $r=0$, we truncate it at $r_0=0.05$ with a function of the form $ae^{-br}$ to ensure the new function has a continuous derivative. The force function $\forcev$ includes self-propulsion with strength $\gamma$ and nonlinear drag with strength $\beta$.

The errors of the estimations for $\mbf{\alpha}$ after our training procedure and the learned $\intkernele$, $\intkernela$ are shown in \cref{tab:ex_FM}. In this model, $\intkernele$ is in the RKHS generated by the chosen Mat\'{e}rn kernel. We can see that our estimators produced faithful approximations to the true kernel based on the results we report in \cref{tab:ex_FM} and \cref{fig:ex_results}(d),(e). They also show that we can identify the nonexistence of the alignment-based interaction with very small errors and select the correct model. The discrepancy between the true trajectories and the predicted trajectories on both the training time interval $[0, 5]$ and future time interval $[5, 10]$ are shown in \cref{tab:FM_traj}. Even if the trajectory prediction errors can go up to $O(10^{-1})$ with the presence of a relatively large noise for the systems with $N=10$, our estimators provided faithful predictions to most of the agents in the system and the milling pattern as shown in \cref{fig:ex_results}(f).

\begin{table}[tbhp]
\caption{Means and standard deviations of the  errors of $\hat\balpha$ (including $\hat \sigma$ when noise exists) and $\hat{\intkernel}$ for different settings of FM dynamics.
} \label{tab:ex_FM}
\begin{center}
\resizebox{0.9\linewidth}{!}{%
\begin{tabular}{cccc}
\toprule
$\{N,M,L,\sigma\}$ & $\| \hat{\balpha} - \balpha\|_{\infty}$  & $\| \hat{\intkernele} - \intkernele\|_{\infty}/\|\intkernele\|_{\infty}$  
& $\| \hat{\intkernela} - 0\|_{\infty}$ 
\\
\cmidrule(lr){1-1}\cmidrule(lr){2-4}
$\{10,1,3,0\}$ & $7.9 \cdot 10^{-4} \pm 1.0\cdot 10^{-3}$  & $3.6 \cdot 10^{-2} \pm 4.3\cdot 10^{-3}$ %
& $6.6 \cdot 10^{-4} \pm 6.9\cdot 10^{-4}$ 
\\
\cmidrule(lr){1-1}\cmidrule(lr){2-4}
$\{10,1,9,0\}$ & $6.4 \cdot 10^{-5} \pm 6.2\cdot 10^{-5}$  & $3.9 \cdot 10^{-2} \pm 2.7\cdot 10^{-3}$ %
& $1.6 \cdot 10^{-4} \pm 1.3\cdot 10^{-4}$ 
\\
\cmidrule(lr){1-1}\cmidrule(lr){2-4}
$\{10,3,3,0\}$ & $\mathbf{4.7 \cdot 10^{-5} \pm 5.0\cdot 10^{-5}}$  & $3.8 \cdot 10^{-2} \pm 5.4\cdot 10^{-3}$ 
& $1.2 \cdot 10^{-4} \pm 1.7\cdot 10^{-4}$ 
\\
\cmidrule(lr){1-1}\cmidrule(lr){2-4}
$\{10,3,3,0.01\}$ & $3.4 \cdot 10^{-3} \pm 1.9\cdot 10^{-3}$  & $\mathbf{2.9 \cdot 10^{-2} \pm 5.7\cdot 10^{-3}}$ 
& $2.9 \cdot 10^{-3} \pm 4.3\cdot 10^{-3}$ 
\\
\cmidrule(lr){1-1}\cmidrule(lr){2-4}
$\{10,3,3,0.05\}$ & $1.4 \cdot 10^{-2} \pm 8.5\cdot 10^{-3}$ & $4.9 \cdot 10^{-2} \pm 1.5\cdot 10^{-2}$ %
& $\mathbf{4.6 \cdot 10^{-5} \pm 7.0\cdot 10^{-5}}$ %
\\
\cmidrule(lr){1-1}\cmidrule(lr){2-4}
$\{10,3,3,0.1\}$ & $3.5 \cdot 10^{-2} \pm 7.2\cdot 10^{-2}$ & $7.1 \cdot 10^{-2} \pm 2.0\cdot 10^{-2}$ %
& $2.9 \cdot 10^{-2} \pm 9.0\cdot 10^{-2}$ 
\\
\bottomrule
\end{tabular} %
}
\end{center}
\end{table}

\begin{table}[tbhp]
\caption{The trajectory prediction errors for different settings of FM dynamics.} \label{tab:FM_traj}
\begin{center}
\resizebox{\linewidth}{!}{%
\begin{tabular}{ccccc}
\toprule
$\{N,M,L,\sigma\}$ & Training IC $[0,5]$ & Training IC $[5,10]$ & new IC $[0,5]$ & new IC $[5,10]$\\
\cmidrule(lr){1-1}\cmidrule(lr){2-5}
$\{10,1,3,0\}$ & $2.1 \cdot 10^{-3} \pm 2.0 \cdot 10^{-3}$  & $1.0 \cdot 10^{-2} \pm 8.7 \cdot 10^{-3}$  &  $1.9 \cdot 10^{-3} \pm 1.9 \cdot 10^{-3}$ & $5.4 \cdot 10^{-3} \pm 4.4 \cdot 10^{-3}$ \\
\cmidrule(lr){1-1}\cmidrule(lr){2-5}
$\{10,1,9,0\}$ & $\mathbf{3.4 \cdot 10^{-4} \pm 2.9 \cdot 10^{-4}}$  & $\mathbf{1.4 \cdot 10^{-3} \pm 1.2 \cdot 10^{-3}}$  &  $\mathbf{4.7 \cdot 10^{-4} \pm 4.2 \cdot 10^{-4}}$ & $\mathbf{1.3 \cdot 10^{-3} \pm 1.2 \cdot 10^{-3}}$ \\
\cmidrule(lr){1-1}\cmidrule(lr){2-5}
$\{10,3,3,0\}$ & $8.1 \cdot 10^{-4} \pm 8.0 \cdot 10^{-4}$  & $2.2 \cdot 10^{-3} \pm 2.0 \cdot 10^{-3}$  & $8.8 \cdot 10^{-4} \pm 8.8 \cdot 10^{-4}$  & $3.5 \cdot 10^{-3} \pm 2.8 \cdot 10^{-3}$ \\
\cmidrule(lr){1-1}\cmidrule(lr){2-5}
$\{10,3,3,0.01\}$ &  $8.3 \cdot 10^{-3} \pm 3.8 \cdot 10^{-3}$ & $1.8 \cdot 10^{-2} \pm 1.2 \cdot 10^{-2}$  & $6.6 \cdot 10^{-3} \pm 3.2 \cdot 10^{-3}$  & $1.4 \cdot 10^{-2} \pm 9.3 \cdot 10^{-3}$ \\
\cmidrule(lr){1-1}\cmidrule(lr){2-5}
$\{10,3,3,0.05\}$ & $3.4 \cdot 10^{-2} \pm 2.1 \cdot 10^{-2}$  & $7.1 \cdot 10^{-2} \pm 4.7 \cdot 10^{-2}$  &  $3.7 \cdot 10^{-2} \pm 1.9 \cdot 10^{-2}$ & $7.0 \cdot 10^{-2} \pm 4.7 \cdot 10^{-2}$ \\
\cmidrule(lr){1-1}\cmidrule(lr){2-5}
$\{10,3,3,0.1\}$ & $8.0 \cdot 10^{-2} \pm 9.8 \cdot 10^{-2}$  & $1.5 \cdot 10^{-1} \pm 1.9 \cdot 10^{-1}$  & $9.5 \cdot 10^{-2} \pm 1.3 \cdot 10^{-1}$  & $1.5 \cdot 10^{-1} \pm 2.3 \cdot 10^{-1}$ \\
\bottomrule
\end{tabular} 
}
\end{center}
\end{table}

\subsubsection{Anticipation Dynamics}
\label{ex: AD}
 In this subsection, we consider a more complicated model where the interactions depend on both the pairwise distance and the differences in velocities, i.e. both $\intkernele$ and $\intkernela$ are nonzero. The anticipation dynamics (AD) models in \cite{shu2021anticipation} are suitable candidates, and we consider the system of $N$ agents in the form \eqref{eq:2ndOrder} with components defined in \cref{tab:ex_info}.

The errors of the estimations for $\mbf{\alpha}$ and the learned $\intkernele$, $\intkernela$ are shown in \cref{tab:ex_AD}. In this model, both $\intkernele$ and $\intkernela$ are in the RKHS generated by the chosen Mat\'{e}rn kernel. We can see that our estimators produced faithful approximations to both true kernels based on the results we report in \cref{tab:ex_AD} and \cref{fig:ex_results}(g),(h). The comparisons between the true trajectories and the predicted trajectories on both the training time interval $[0, 10]$ and future time interval $[10, 20]$ are shown in \cref{tab:AD_traj}. The estimators can produce accurate approximations of the true
trajectories with errors at most $O(10^{-1})$, see also \cref{fig:ex_results}(i). 

\begin{table}[tbhp]
\caption{Means and standard deviations of the  errors of $\hat \sigma$ (when noise exists) and $\hat{\intkernel}$ for different settings of AD dynamics.
} \label{tab:ex_AD}
\begin{center}
\resizebox{0.9\linewidth}{!}{%
\begin{tabular}{cccc}
\toprule
$\{N,M,L,\sigma\}$ & $\| \hat{\sigma} - \sigma\|_{\infty}$  & $\| \hat{\intkernele} - \intkernele\|_{\infty}/\|\intkernele\|_{\infty}$ 
& $\| \hat{\intkernela} - \intkernela\|_{\infty}/\| \intkernela\|_{\infty}$ 
\\
\cmidrule(lr){1-1}\cmidrule(lr){2-4}
$\{10,3,3,0\}$ & - & $9.2 \cdot 10^{-2} \pm 7.4 \cdot 10^{-3}$ 
& $4.5 \cdot 10^{-2} \pm 1.0 \cdot 10^{-2}$ 
\\
\cmidrule(lr){1-1}\cmidrule(lr){2-4}
$\{10,6,3,0\}$ & - & $7.9 \cdot 10^{-2} \pm 6.7 \cdot 10^{-3}$ 
& $4.3 \cdot 10^{-2} \pm 5.1 \cdot 10^{-3}$ 
\\
\cmidrule(lr){1-1}\cmidrule(lr){2-4}
$\{10,12,3,0\}$ & - & $\mathbf{7.4 \cdot 10^{-2} \pm 6.1 \cdot 10^{-3}}$ 
& $\mathbf{3.6 \cdot 10^{-2} \pm 7.0 \cdot 10^{-3}}$ %
\\
\cmidrule(lr){1-1}\cmidrule(lr){2-4}
$\{10,12,3,0.005\}$ & $8.8 \cdot 10^{-5} \pm 5.1 \cdot 10^{-5}$ & $1.3 \cdot 10^{-1} \pm 1.7 \cdot 10^{-2}$ 
& $7.3 \cdot 10^{-2} \pm 3.2 \cdot 10^{-2}$ 
\\
\cmidrule(lr){1-1}\cmidrule(lr){2-4}
$\{10,12,3,0.01\}$ & $1.8 \cdot 10^{-4} \pm 9.9 \cdot 10^{-5}$ & $1.6 \cdot 10^{-1} \pm 1.9 \cdot 10^{-2}$ 
& $9.3 \cdot 10^{-2} \pm 4.1 \cdot 10^{-2}$ 
\\
\bottomrule
\end{tabular} %
}
\end{center}
\end{table}

\begin{table}[tbhp]
\caption{The trajectory prediction errors for different settings of AD dynamics.} \label{tab:AD_traj}
\begin{center}
\resizebox{\linewidth}{!}{%
\begin{tabular}{ccccc}
\toprule
$\{N,M,L,\sigma\}$ & Training IC $[0,10]$ & Training IC $[10,20]$ & new IC $[0,10]$ & new IC $[10,20]$\\
\cmidrule(lr){1-1}\cmidrule(lr){2-5}
$\{10,3,3,0\}$ & $\mathbf{4.2 \cdot 10^{-4} \pm 3.8 \cdot 10^{-4}}$ & $\mathbf{2.3 \cdot 10^{-4} \pm 2.1 \cdot 10^{-4}}$ & $6.1 \cdot 10^{-4} \pm 8.4 \cdot 10^{-4}$ & $3.5 \cdot 10^{-4} \pm 5.0 \cdot 10^{-4}$ \\
\cmidrule(lr){1-1}\cmidrule(lr){2-5}
$\{10,6,3,0\}$ & $6.6 \cdot 10^{-4} \pm 7.4 \cdot 10^{-4}$ & $3.8 \cdot 10^{-4} \pm 4.1 \cdot 10^{-4}$ & $7.1 \cdot 10^{-4} \pm 9.2 \cdot 10^{-4}$ & $3.9 \cdot 10^{-4} \pm 5.2 \cdot 10^{-4}$ \\
\cmidrule(lr){1-1}\cmidrule(lr){2-5}
$\{10,12,3,0\}$ & $6.2 \cdot 10^{-4} \pm 6.8 \cdot 10^{-4}$ & $3.3 \cdot 10^{-4} \pm 3.7 \cdot 10^{-4}$ & $\mathbf{3.7 \cdot 10^{-4} \pm 5.2 \cdot 10^{-4}}$ & $\mathbf{2.1 \cdot 10^{-4} \pm 3.1 \cdot 10^{-4}}$ \\
\cmidrule(lr){1-1}\cmidrule(lr){2-5}
$\{10,12,3,0.005\}$ & $1.9 \cdot 10^{-3} \pm 2.1 \cdot 10^{-3}$ & $1.1 \cdot 10^{-3} \pm 1.2 \cdot 10^{-3}$ & $1.1 \cdot 10^{-3} \pm 1.2 \cdot 10^{-3}$ & $6.8 \cdot 10^{-4} \pm 7.1 \cdot 10^{-4}$ \\
\cmidrule(lr){1-1}\cmidrule(lr){2-5}
$\{10,12,3,0.01\}$ & $3.4 \cdot 10^{-3} \pm 4.3 \cdot 10^{-3}$ & $1.9 \cdot 10^{-3} \pm 2.4 \cdot 10^{-3}$ & $1.9 \cdot 10^{-3} \pm 2.1 \cdot 10^{-3}$ & $1.2 \cdot 10^{-3} \pm 1.3 \cdot 10^{-3}$ \\
\bottomrule
\end{tabular} 
}
\end{center}
\end{table}

\subsection{Model selection for the order of systems}
\label{ex: ODS}
An example of opinion dynamics is shown below to test the validity of our method for identifying the order of dynamic systems. This is a first-order system of $N$ interacting agents, and each agent $i$ is characterized by a continuous opinion variable $x_i \in \mathbb{R}$. The dynamics of opinion exchange are governed by the first-order equation mentioned in \cref{subsec: MS_order}
with
\begin{equation}
  \intkernele(r) = 
  \begin{cases}
  25 r & \textrm{if } 0 \leq r < 0.4,\\
  10 & \textrm{if } 0.4 \leq r < 0.6,\\
  25 - 25r & \textrm{if } 0.6 \leq r < 1,\\
  0 & \textrm{if } r \geq 1.
  \end{cases}
\label{eq:ODS_phi}
\end{equation}
The interaction kernel $\intkernele$ encodes the non-repulsive interactions between agents: all agents aim to align their opinions to their connected neighbors according to distance-based attractive influences. We consider the case where there is no non-collective force, i.e. $\force_i(\bx_i,\mbf{\alpha}) \equiv 0$.
We also consider a more complicated case where there exist stubborn agents, i.e. \begin{equation}
  \force_i(\bx_i, \mbf{\alpha}) = 
  \begin{cases}
  - \kappa(\bx_i - P_i) & \textrm{if agent $i$ is stubborn with bias $P_i$}\\
  0 & \textrm{otherwise}
  \end{cases}
  \label{eq:ODS_force}
\end{equation}
where $\force_i(\bx_i, \mbf{\alpha})$ describes the additional influence induced by the stubbornness: the stubborn agents have strong desires to follow their bias $P_i$, and $\kappa$ controls the rate of convergence towards their bias. The stubborn agents may cause a major effect on the collective opinion formation process. If $\kappa=0$, then stubborn agents do not follow their biases and behave as regular agents.

\cref{tab:OD_MS} shows the errors of the estimations for $m$, $\balpha$, and $\intkernele(r)$ in 10 independent trails of experiments. It shows our method can identify the order of dynamics with the estimation of $m \approx 0$ and learn the interaction kernel $\intkernel$ simultaneously, see also \cref{fig:ex_OD_MS}.

\begin{table}[tbhp]
\caption{Means and standard deviations of the  errors of $\hat m$ (including $\hat \sigma$ when noise exists) and $\hat{\intkernel}$ for different settings of OD dynamics.} \label{tab:OD_MS}
\begin{center}
\resizebox{\linewidth}{!}{%
\begin{tabular}{ccccc}
\toprule
Model & $\{N,M,L,\sigma\}$ & $\| \hat{m} - 0\|_{\infty}$ & $\| \hat{\balpha} - \balpha\|_{\infty}$ & $\| \hat{\intkernel} - \intkernel\|_{\infty}/ \| \intkernel\|_{\infty}$ 
\\
\cmidrule(lr){1-2}\cmidrule(lr){3-5}
OD & $\{5,6,3,0\}$ & $8.5\cdot 10^{-4}\pm 9.0\cdot 10^{-4}$ & - & $3.8\cdot 10^{-3}\pm 1.1\cdot 10^{-3}$ 
\\
\cmidrule(lr){1-2}\cmidrule(lr){3-5}
OD & $\{5,6,3,0.1\}$ & $4.8\cdot 10^{-3}\pm 5.2\cdot 10^{-4}$ & $3.2\cdot 10^{-2}\pm 1.6\cdot 10^{-2}$ & $1.1\cdot 10^{-2}\pm 5.6\cdot 10^{-3}$ 
\\
\cmidrule(lr){1-2}\cmidrule(lr){3-5}
ODS & $\{10,3,3,0\}$ & $5.5\cdot 10^{-4}\pm 2.8\cdot 10^{-4}$ & $7.2\cdot 10^{-2}\pm 4.1\cdot 10^{-2}$ & $5.2\cdot 10^{-2}\pm 4.4\cdot 10^{-2}$ 
\\
\cmidrule(lr){1-2}\cmidrule(lr){3-5}
ODS & $\{10,3,3,0.1\}$ & $3.8\cdot 10^{-3}\pm 1.8\cdot 10^{-3}$ & $9.0\cdot 10^{-1}\pm 1.1\cdot 10^{0}$ & $3.3\cdot 10^{-2}\pm 1.9\cdot 10^{-2}$ 
\\
\bottomrule
\end{tabular} 
}
\end{center}
\end{table}

\begin{figure}[tbhp]
\centering
\includegraphics[width=0.31\linewidth,height=0.22\linewidth]{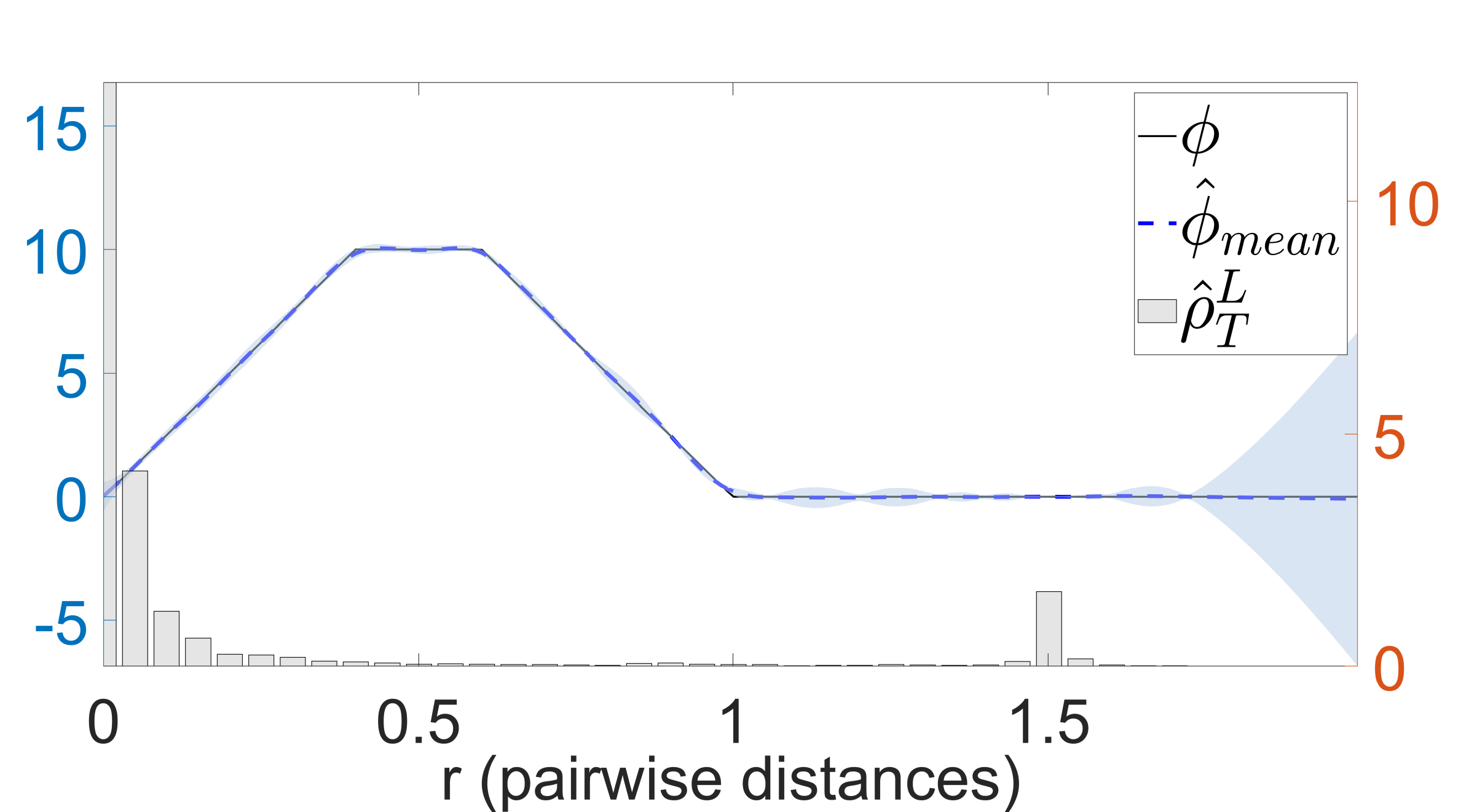}
\includegraphics[width=0.31\linewidth,height=0.22\linewidth]{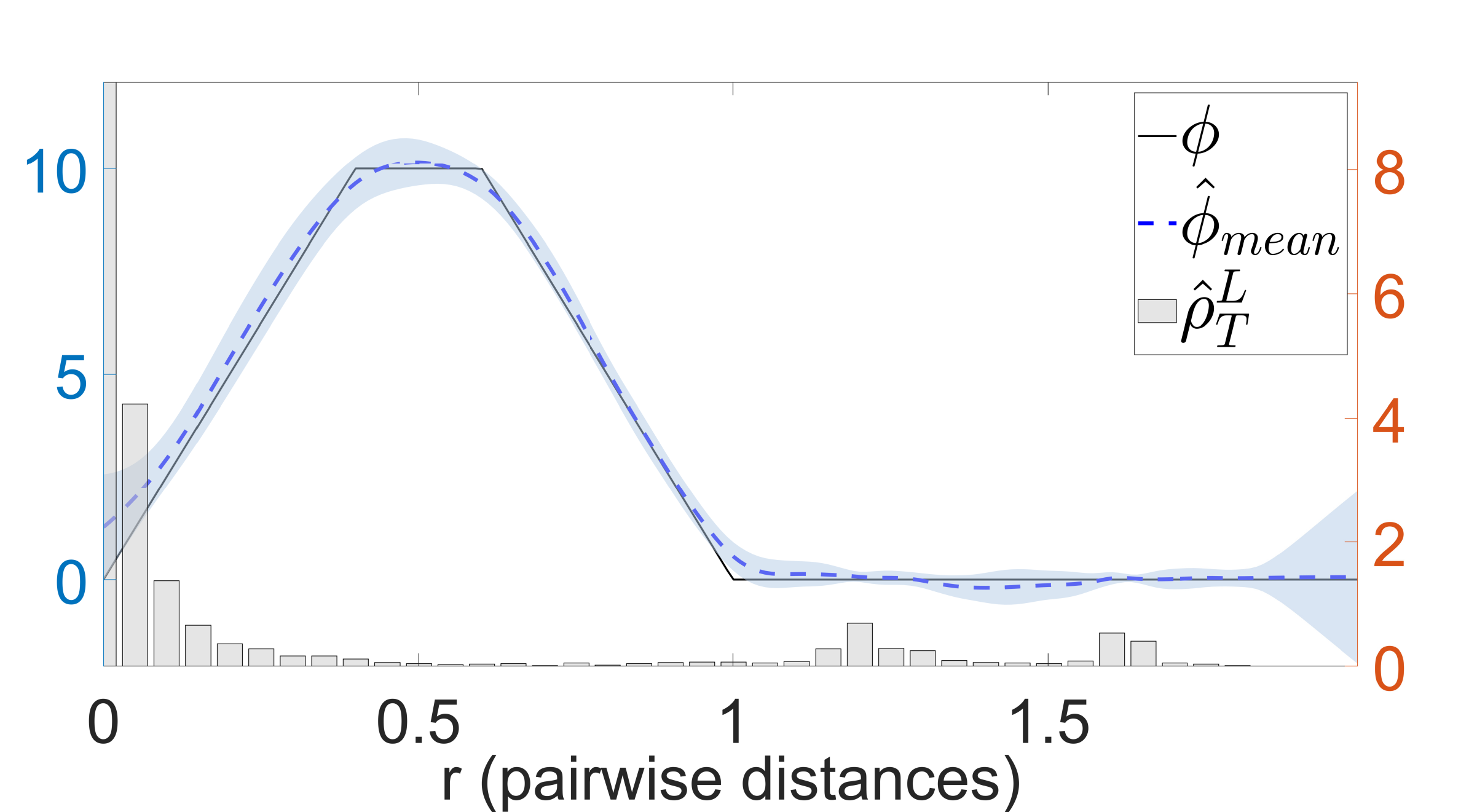}
\includegraphics[width=0.36\linewidth]{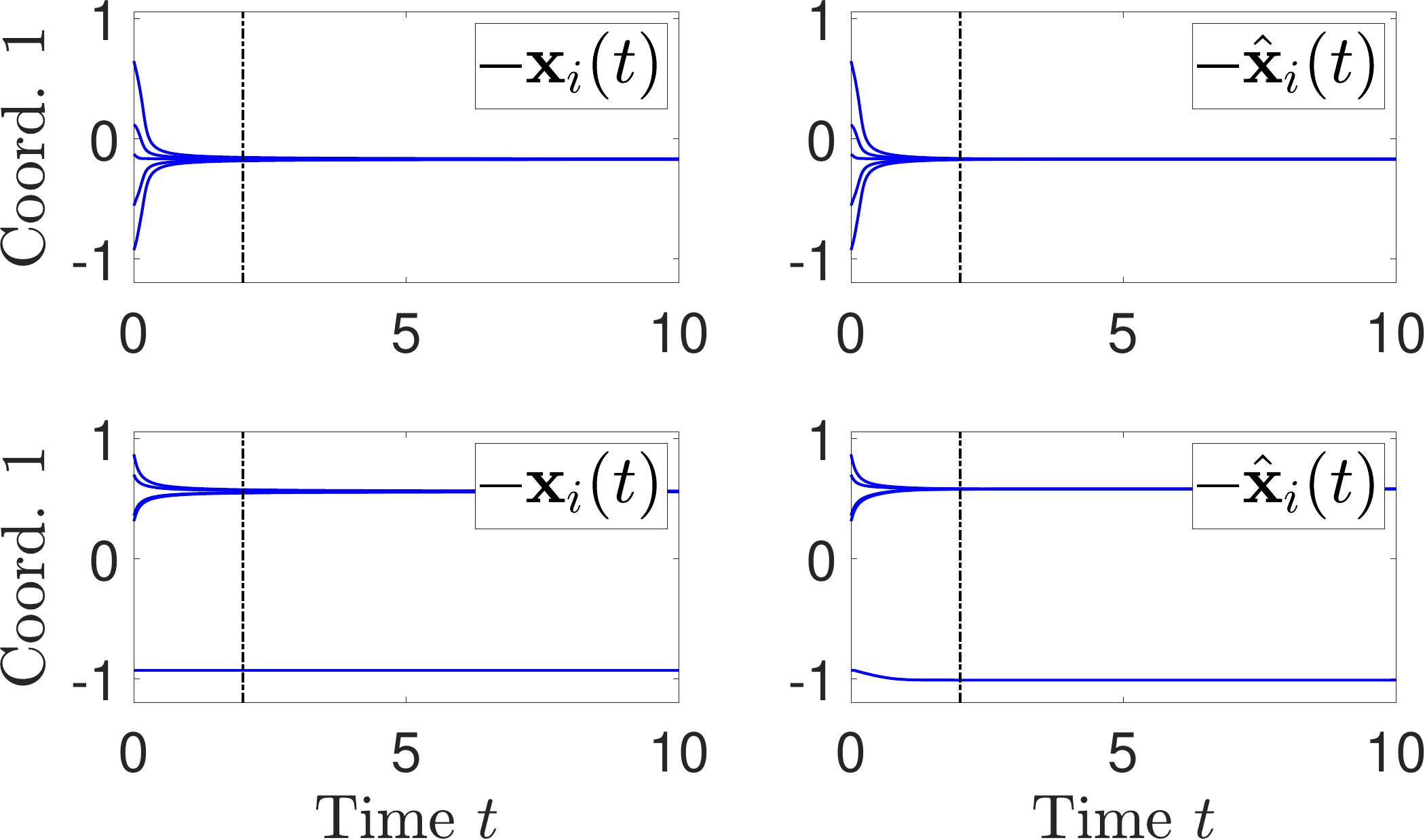}
\caption{Model selection of OD ($\{N,M,L\} = \{5,6,3\}$) and $\sigma = 0, 0.1$ using the Matérn kernel. Left: Predictive mean $\hat{\intkernel}$ of the true kernel, and two-standard-deviation band (light blue color) around the mean. The grey bars represent the empirical density of the $\tilde \rho_r$. Right: the true (left) versus predicted (right) trajectories using $\hat{\balpha}$ and $\hat{\intkernel}$  with initial conditions of training data (top) and testing data (bottom) when $\sigma = 0.1$.}
\label{fig:ex_OD_MS}
\end{figure}

\subsection{Real fish data}
\label{ex: realdata}
Finally, we test the performance of our method using two real datasets of swimming fish by Couzin et al., which are available at ScholarsArchive of Oregon State University\footnote{Katz, Yael, Kolbjorn Tunstrom, Christos C Ioannou, Cristian Huepe, and Iain D Couzin, 2021. The URL address is https://ir.library.oregonstate.edu/concern/datasets/zk51vq07c}. The experimental arena consisted of a white shallow tank of size $2.1 \times 1.2$ m ($7 \times 4$ ft) surrounded by a floor-to-ceiling white curtain. Water depth was chosen to be $4.5$-$5$ cm so the schools would be approximately 2D. We consider two data sets, one is from frame 2201 to frame 2296 which consists of 50 fish and forms a flocking behavior, and another one is from frame 4601 to frame 4798 which consists of 124 fish and forms a milling behavior. We relabel them as frame 0 to frame 95 and frame 0 to frame 198 respectively, refer to more details of the dataset in the supplementary information of \cite{katz2011inferring}.   We first normalize the position data into the region [0,1], and then we smooth the data by using a moving window average with a window size of 10 frames and apply the finite difference method to calculate the velocities and accelerations. 

\paragraph{Flocking behavior example}
For the first data set, as shown in \cref{fig:ex_real_results}(e), the fish will eventually follow approximately the same direction as time evolves. In this case, the magnitude of the velocity data of fish is relatively small. The velocities can be considered the same, as long as their normalized direction vectors are very close. So the fish exhibit approximate  flocking behavior (i.e. $\| \mbf{v}_i - \mbf{v}_c \|\approx 0$ for all $i$ and some common velocity $\mbf{v}_c$). Therefore, we use the Cucker-Smale system shown in \cref{ex: CS} to model the flocking behavior, i.e. considering the governing equation \eqref{eq:2ndOrder} with corresponding interaction kernel $\intkernela$ and force $\force(\bx_i, \dot\bx_i, \mbf{\alpha})$,  $\balpha = (\kappa, p)$ shown in \cref{tab:ex_info} for CS dynamics.

The training data consists of frame 0 and frame 28. In the training procedure, we first use a subset of data with two selected agents, the initialization of hyperparameters for $\theta^E,\theta^A,\sigma$, and $\balpha$ are $(1,1), (1,1), 0.001, (1, 1)$, and we set the length of runs in the minimizer solver to be 100. The results shown in \cref{fig:ex_real_results}(a),(b) suggest there only exist alignment-based interactions since the estimated energy-based interaction $\hat{\intkernel}^E \equiv 0$. Therefore, we use all data to learn the system with only $\intkernela$. After we obtain the estimators, we run the learned dynamical system on the time interval [0,20] with frame 0 as the initial condition. We find that the simulated position data at $t=19$ matches the position data at frame 95 very well. We then compare the original position data set with the simulated ones at $t=0:0.2:19$.

\paragraph{Milling behavior example}

For the second data set, as shown in \cref{fig:ex_real_results}(g), the fish will eventually follow approximately a milling pattern. Therefore, we use the Fish-Milling system shown in \cref{ex: FM} to model the milling behavior, i.e. considering the governing equation \eqref{eq:2ndOrder} with interaction kernel $\intkernele$ and force $\force(\bx_i, \dot\bx_i, \mbf{\alpha})$,  $\balpha = (\gamma, \beta)$ shown in \cref{tab:ex_info} for FM dynamics.

The training data consists of frame 0 and frame 28. In the training procedure, we first use a subset of data with two selected agents, the initialization of hyperparameters for $\theta^E,\theta^A,\sigma$, and $\balpha$ are $(1,1), (1,1), 0.001, (1, 1)$, and we set the length of runs in the minimizer solver to be 100. With the estimated parameters, we obtained the estimators for $\intkernele$ and $\intkernela$ using all data, and run the learned dynamical system at the time interval [0,38] with the frame 0 as the initial condition.

\begin{figure}[tbhp]
\centering
\subfigure[Real data CS: $\intkernele$]{
\includegraphics[width=0.22\textwidth,height=0.14\textwidth]{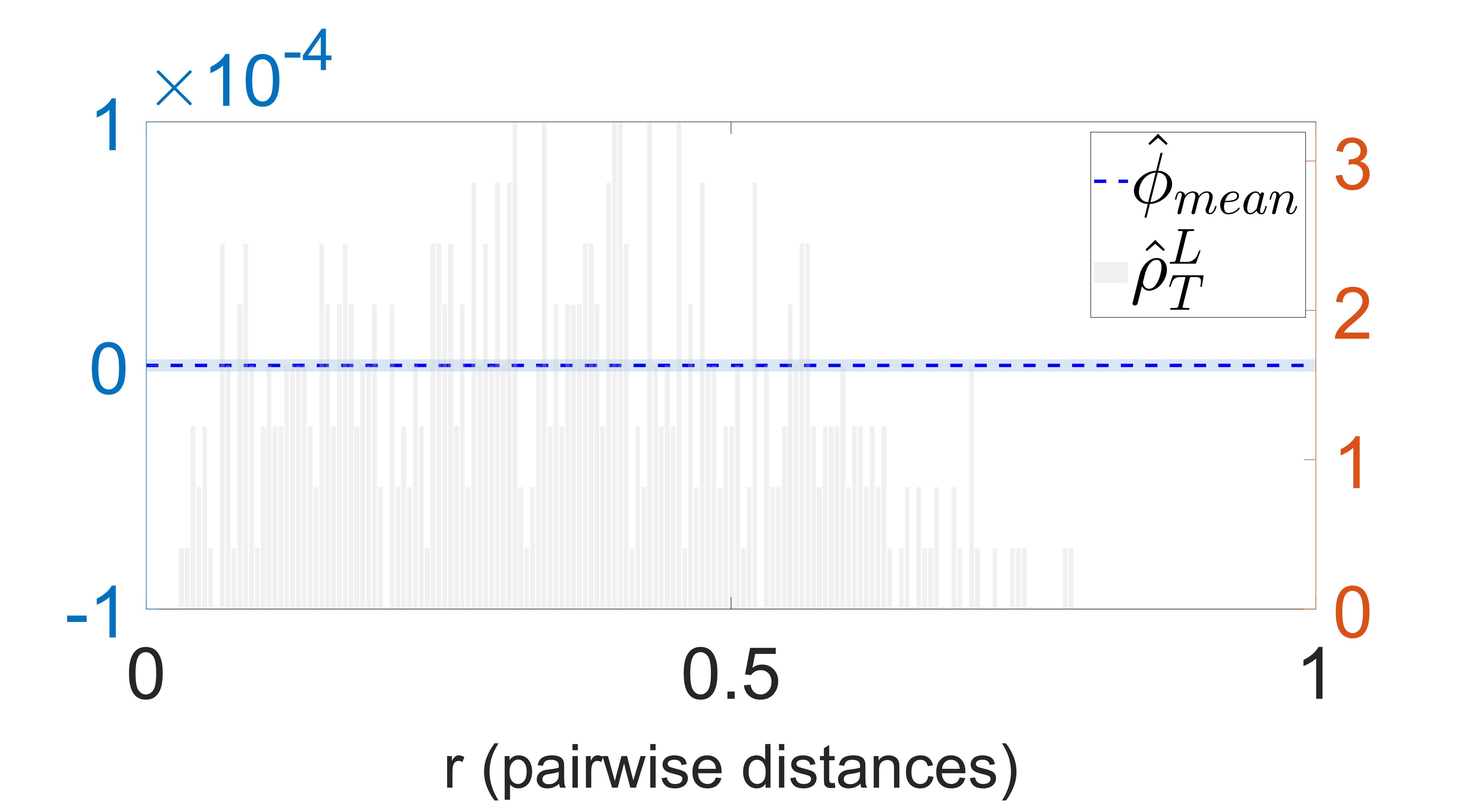}
}
\subfigure[Real data CS: $\intkernela$]{
\includegraphics[width=0.22\textwidth,height=0.14\textwidth]{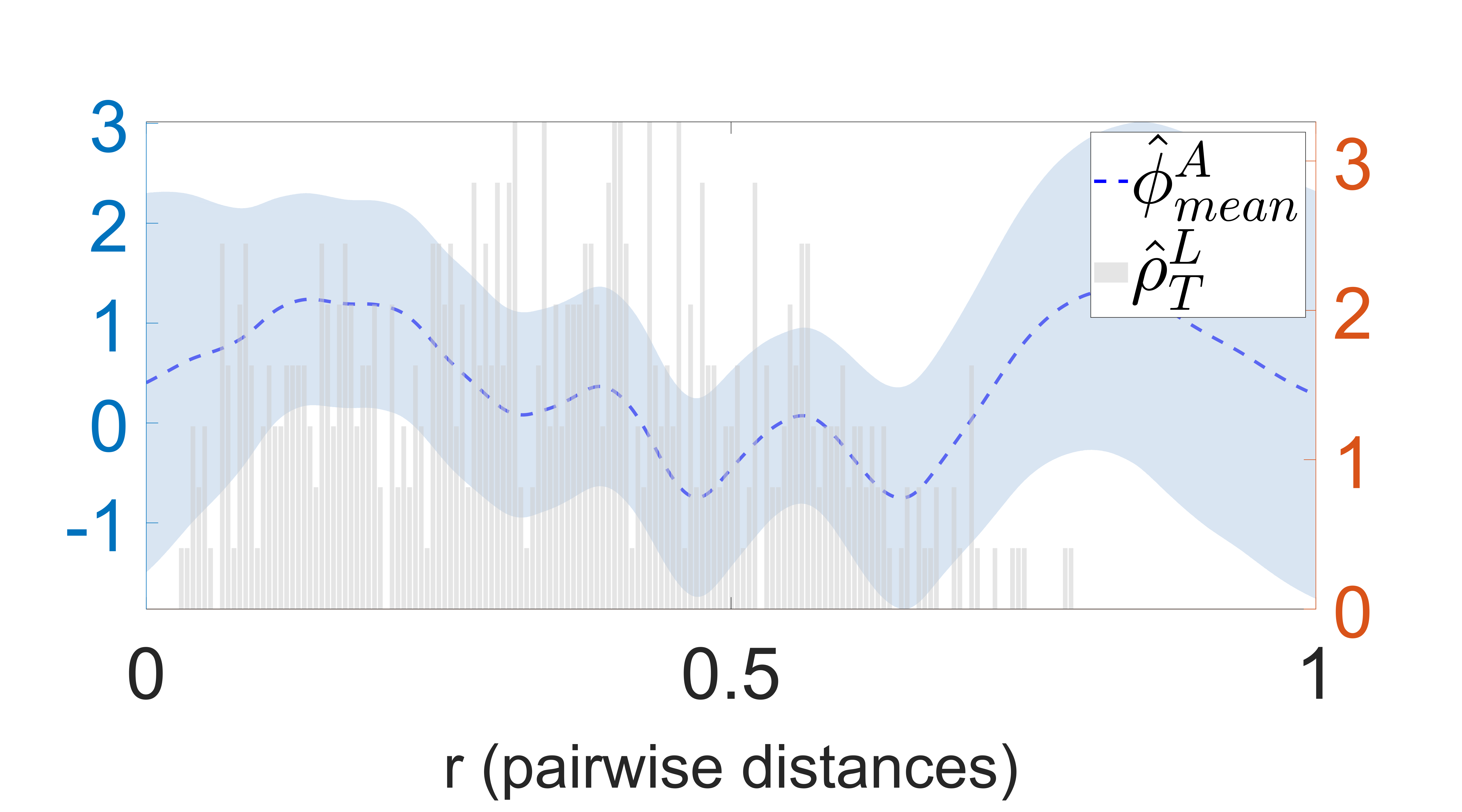}
}
\subfigure[Real data FM: $\intkernele$]{
\includegraphics[width=0.22\textwidth,height=0.14\textwidth]{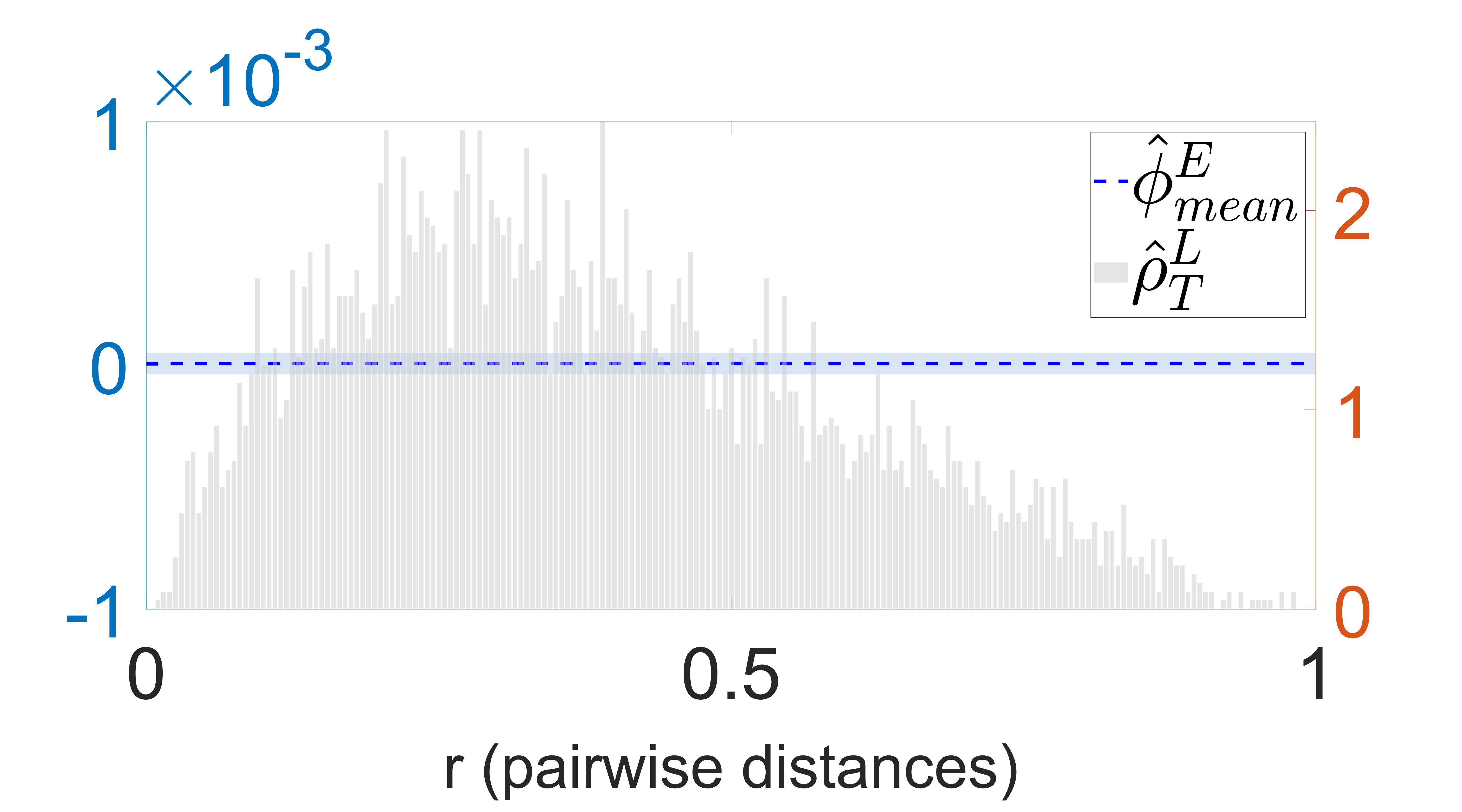}
}
\subfigure[Real data FM: $\intkernele$]{
\includegraphics[width=0.22\textwidth,height=0.14\textwidth]{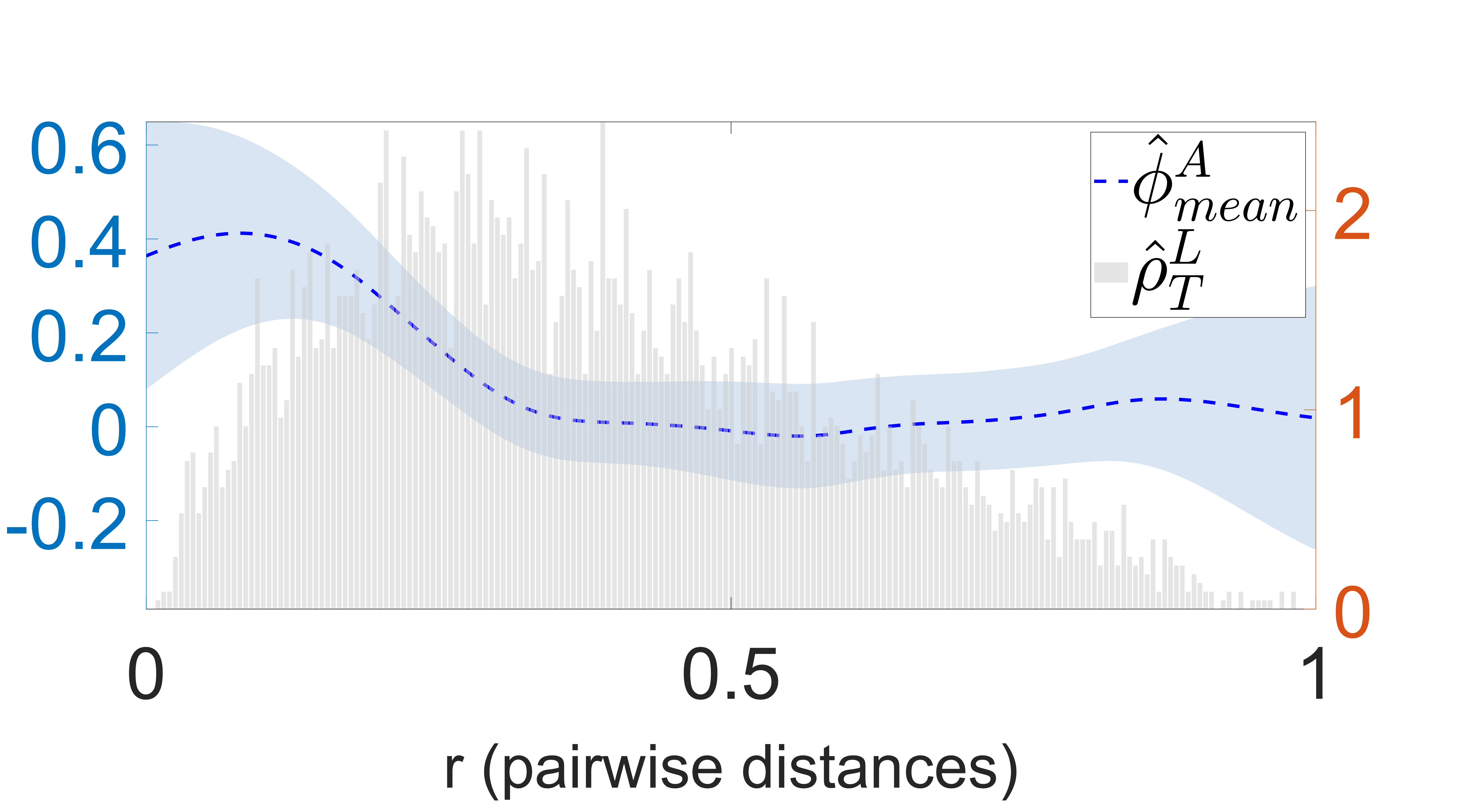}
}
\subfigure[Real data CS trajectory prediction]{
\includegraphics[width=0.3\textwidth]{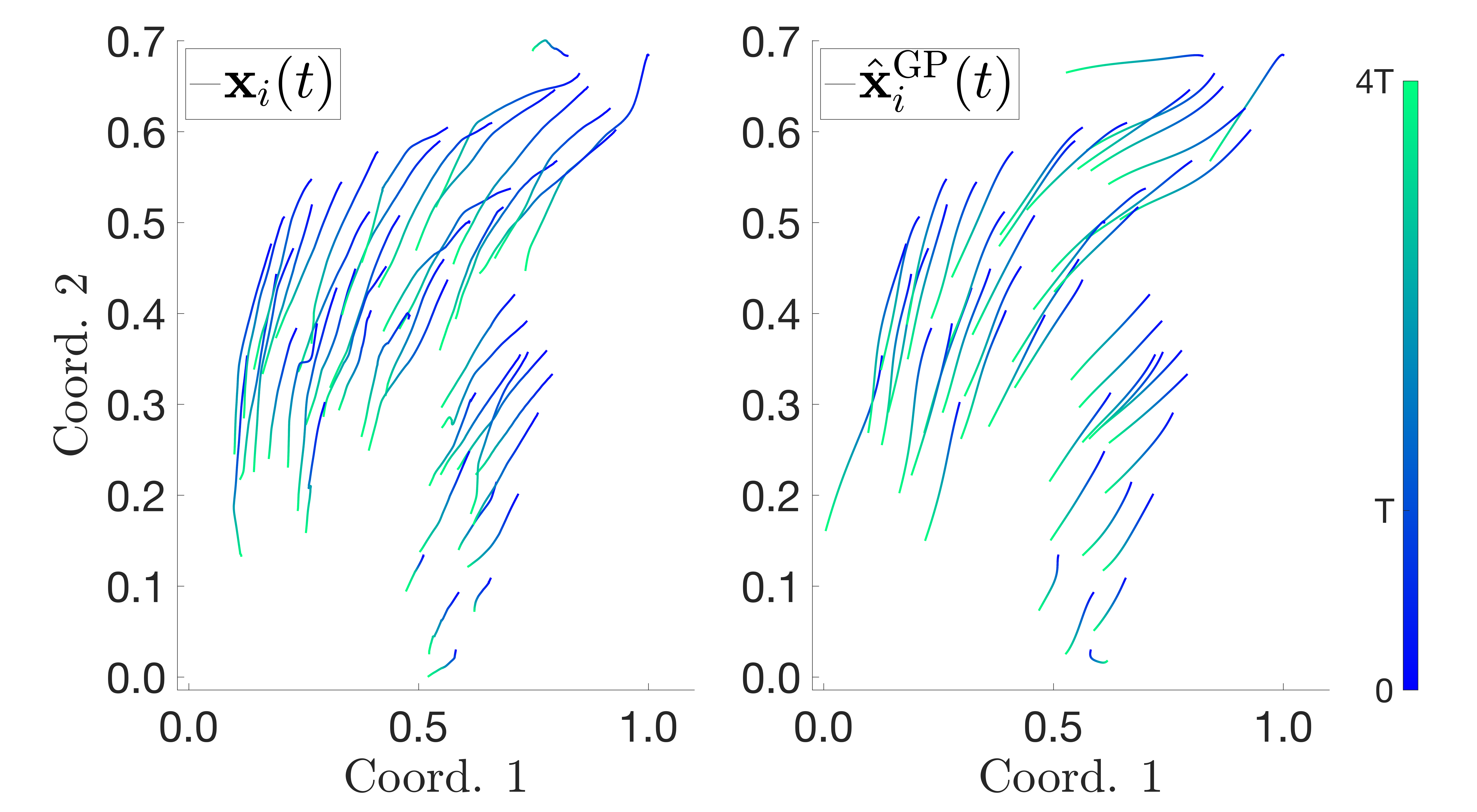}
\includegraphics[width=0.3\textwidth]{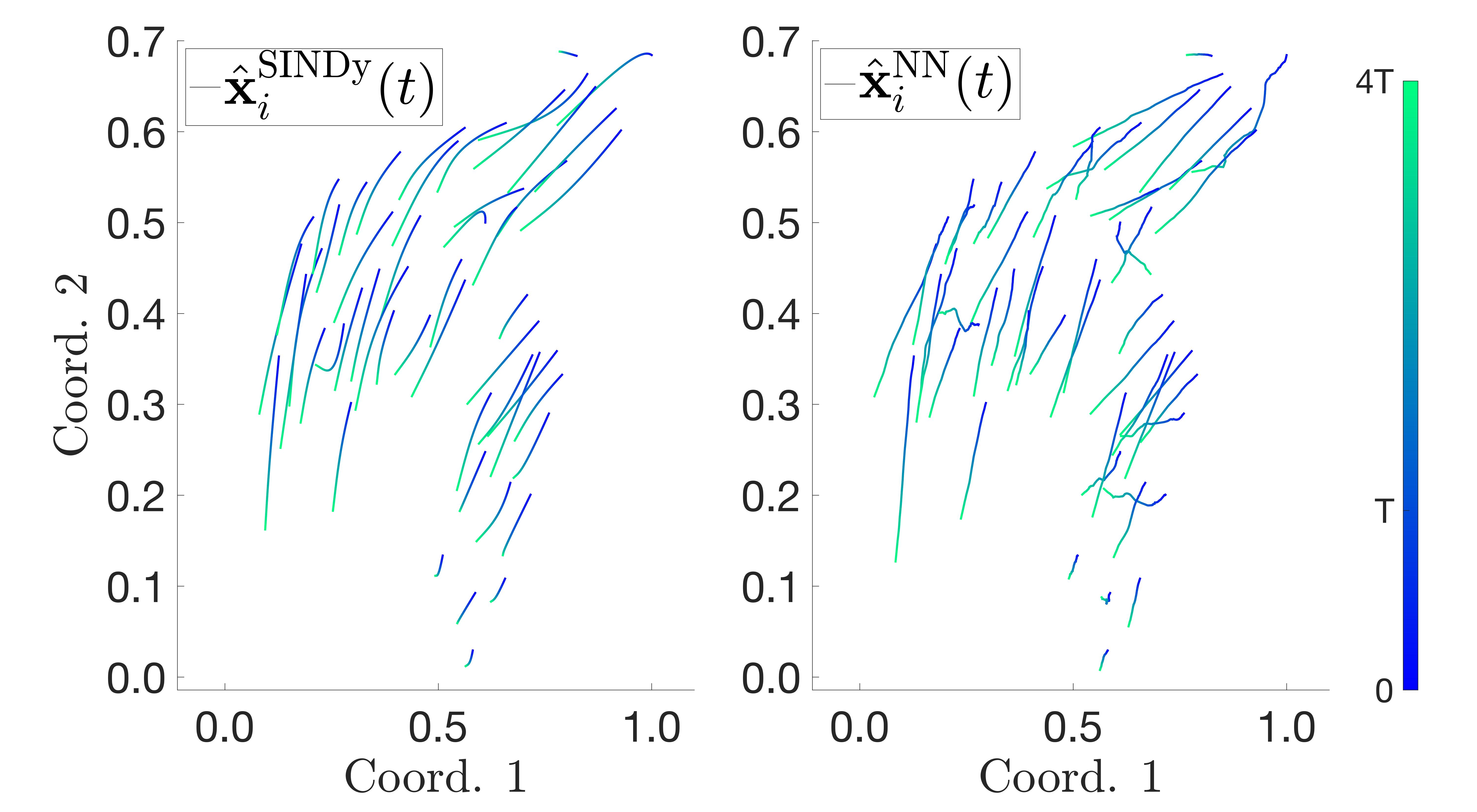}
}
\subfigure[Real data CS comparisons]{
\includegraphics[width=0.3\textwidth]{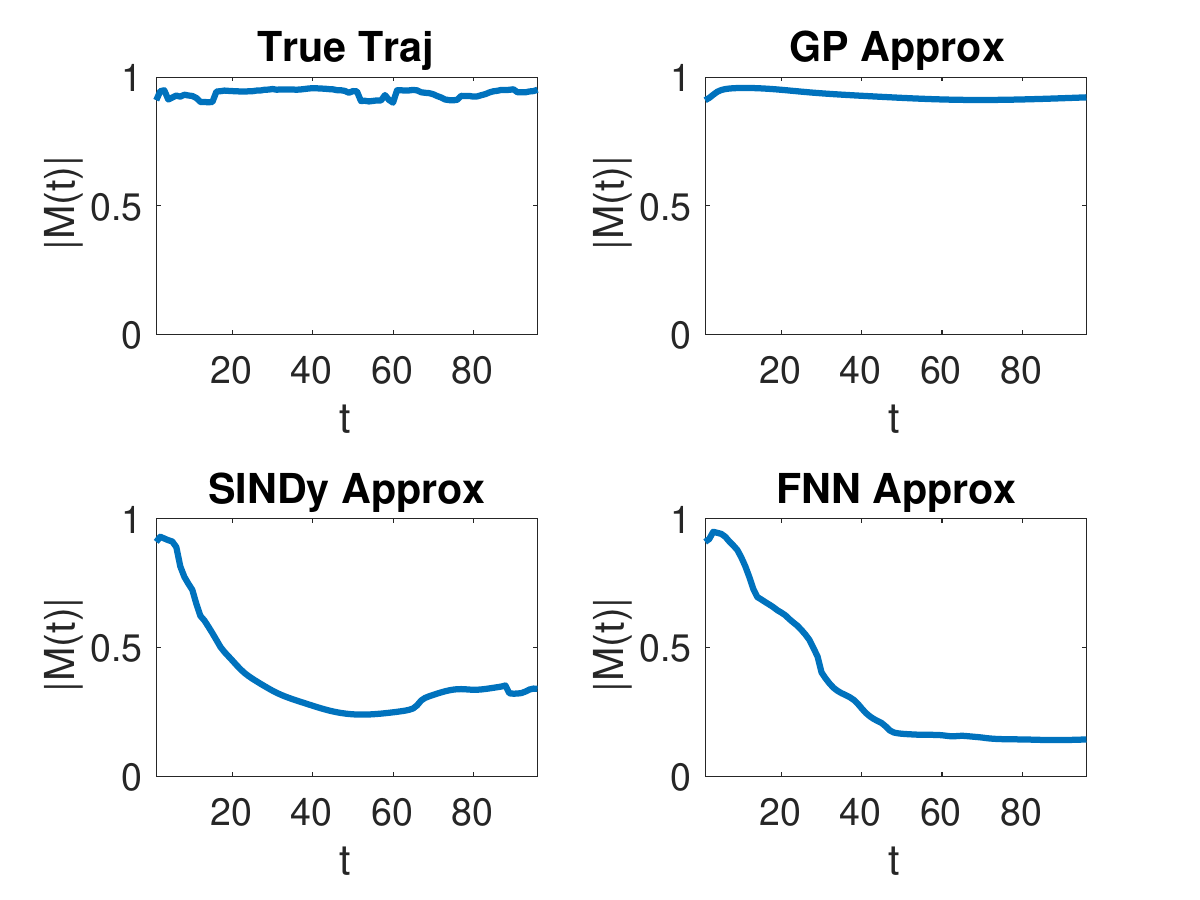}
}
\subfigure[Real data FM trajectory prediction]{
\includegraphics[width=0.3\textwidth]{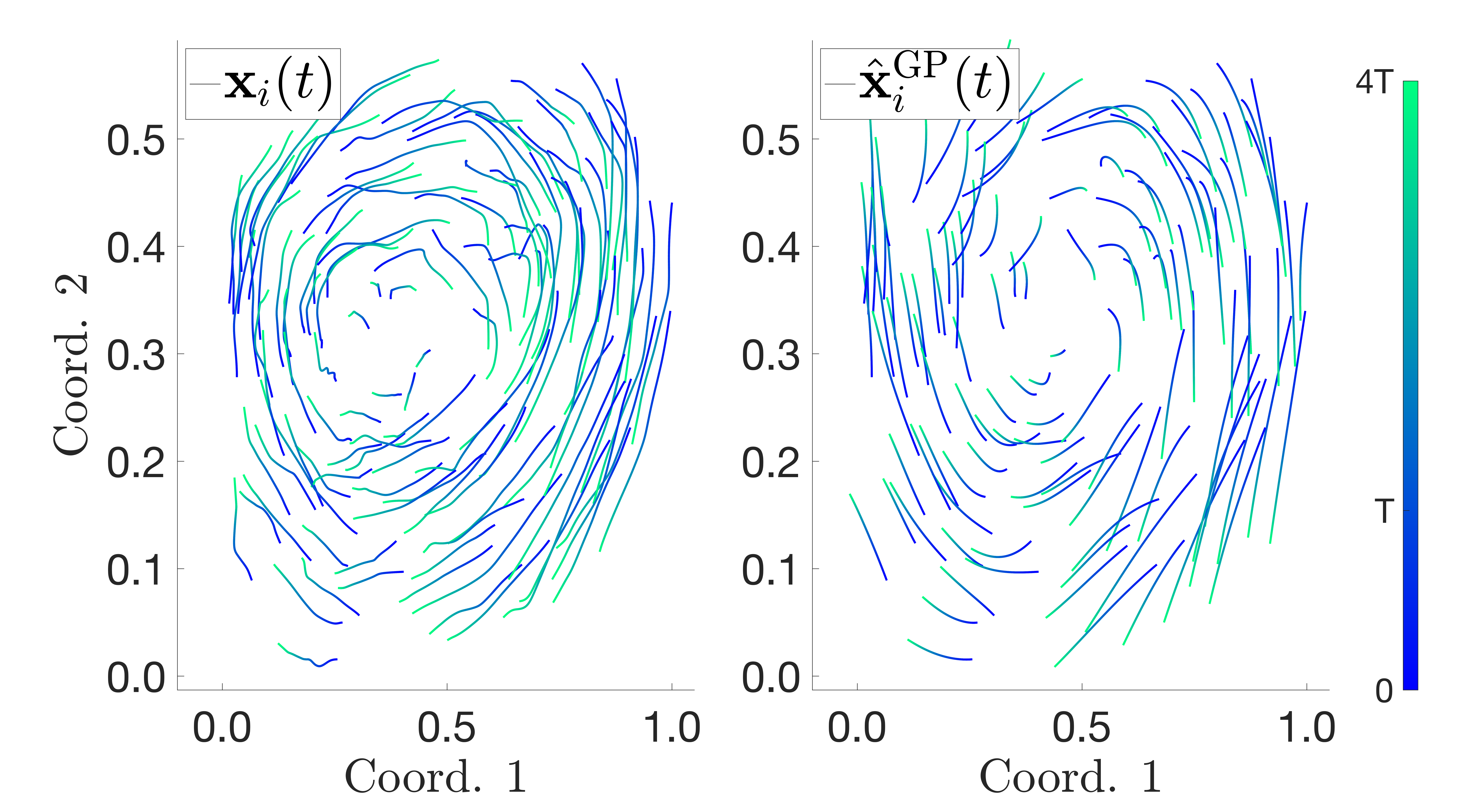}
\includegraphics[width=0.3\textwidth]{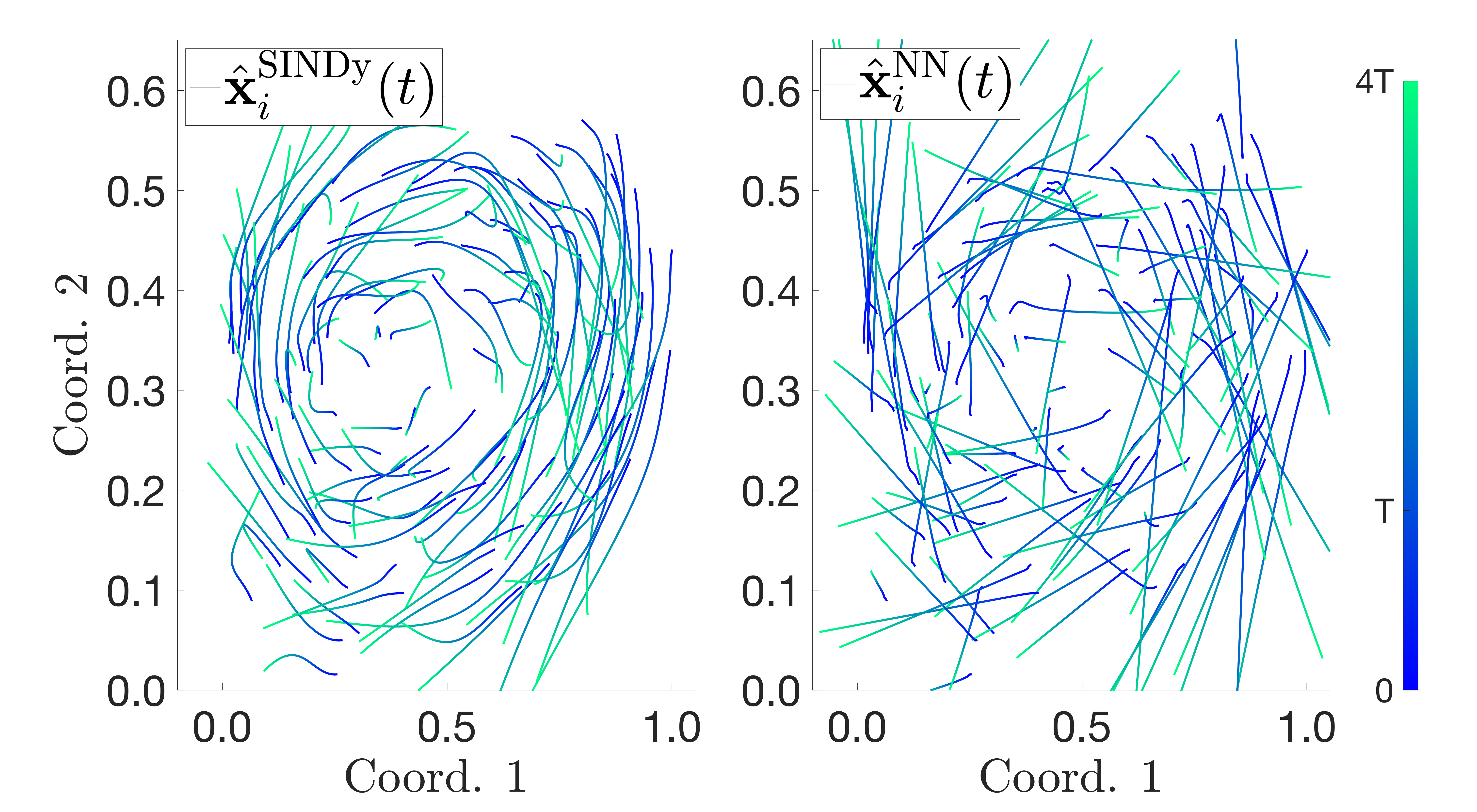}
}
\subfigure[Real data FM comparisons]{
\includegraphics[width=0.3\textwidth]{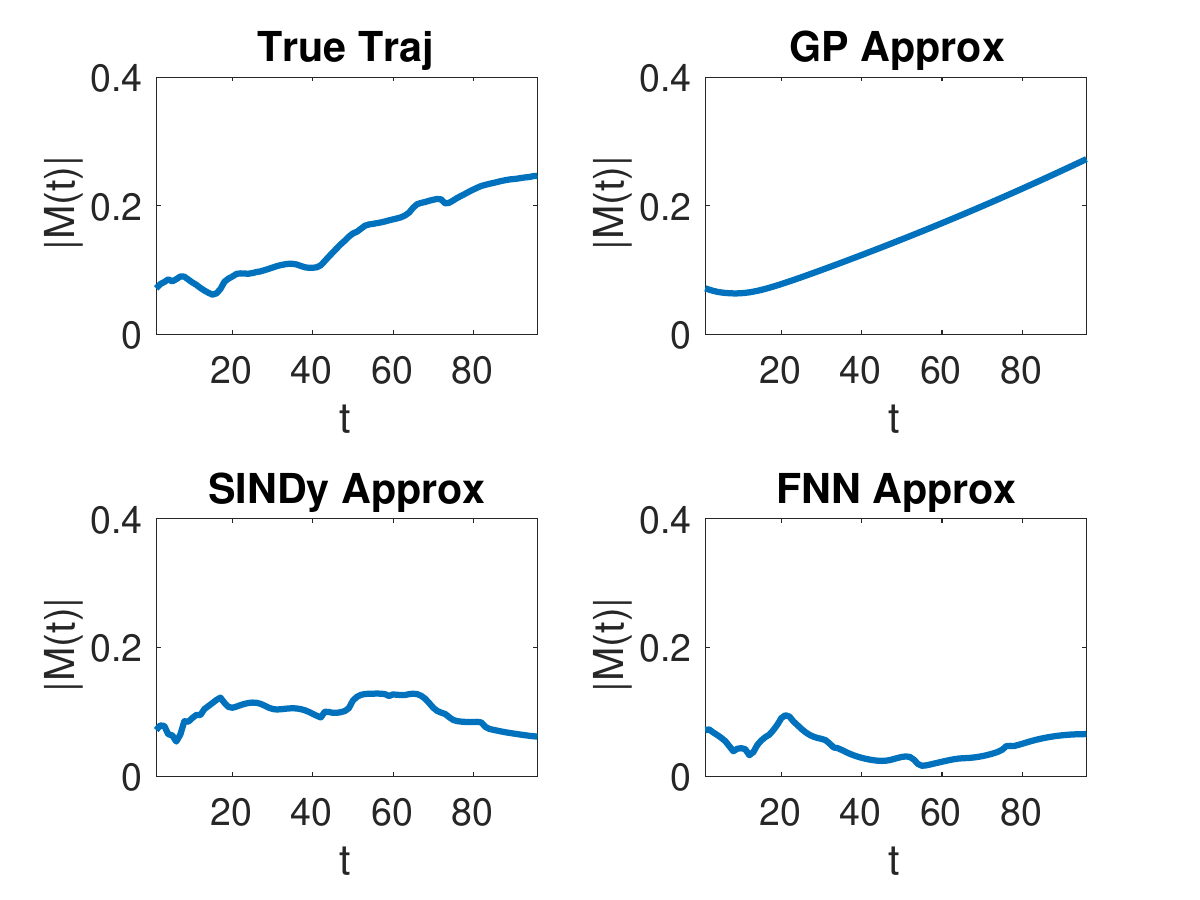}
}

\caption{(a)(b): Fitting into a Cucker-Smale system (dim=100), estimated $\intkernele$ and $\intkernela$ with all data; (c)(d): Fitting into a Fish-Milling system (dim=248), estimated $\intkernele$ and $\intkernela$ with data of 2 agents; 
(e)(g): true dynamics (left) v.s. the predicted dynamics using our proposed approach, SINDy model, and FNN model, with frame 0 as the initial condition; (f)(h): baseline comparisons using the group polarisation parameter $M(t)$.
}
\label{fig:ex_real_results}
\end{figure}

\paragraph{Measure of Performance}
To evaluate the performance at the group level, we consider the group polarisation $M(t)$ \cite{jhawar2020noise}, which is a vector order parameter that encapsulates both the direction and degree of the fish
alignment, which is defined by 
\begin{equation}
    M(t) = \frac{1}{N} \sum_{i=1}^N \hat{\bv}_i(t) 
\end{equation}
where $\hat{\bv}_i(t) = \bv_i(t)/\|\bv_i(t)\|$ is the direction of motion of the i-th fish (at time t). When $|M|$ is close to 1, the fish are moving in a coherent direction, whereas when $|M|$ is close to zero, there is no prevailing direction and individual motion is effectively isotropic.

\paragraph{Baseline Comparisons}
We perform comparisons with approaches that learn the right-hand side function of \eqref{eq:2ndOrder} directly from trajectory data: the first one is SINDy  \cite{brunton2016discovering}, which aims at finding a sparse representation for each row of governing equations in a (typically large) dictionary; the second one is regression using feed-forward neural networks, for which we use the MATLAB\textsuperscript{\textregistered} 2021a Deep Learning Toolbox\textsuperscript{\texttrademark}. 

For the SINDy model, we apply a reasonably large dictionary consisting of monomials up to order 2, sines, and cosines of frequencies $\{k\}_{k=1}^{10}$.
For the neural network model, we consider a three-layer FNN (Feed-Forward Neural Network) with $[50, 50, 25]$ hidden units for the flocking behavior example, and a two-layer FNN with $[40, 20]$ hidden units for the milling behavior example.

The predictive trajectories for the flocking behavior example using different models are shown in \cref{fig:ex_real_results}(e). We compare the performances in terms of group polarisation $M(t)$ in \cref{fig:ex_real_results}(f), and the order-1 Wasserstein distances between the empirical distributions of $M(t)$ in true data and the predicted dynamics are $W(p_M^{true},p_M^{gp}) = 0.0112$, $W(p_M^{true},p_M^{SINDy}) = 0.5497$, and $W(p_M^{true},p_M^{NN}) = 0.5869$, see \cref{fig:ex_real_compare1} (Left). The results for the milling behavior example are shown in \cref{fig:ex_real_results}(g). We compare the performances in \cref{fig:ex_real_results}(h), and the order-1 Wasserstein distances between the empirical distributions of $M(t)$ in the true data and the predicted dynamics are $W(p_M^{true},p_M^{gp}) = 0.0087$, $W(p_M^{true},p_M^{SINDy}) = 0.0530$, and $W(p_M^{true},p_M^{NN}) = 0.1055$, see \cref{fig:ex_real_compare1} (Right). In both examples, we can see that although both predictions using the SINDy and FNN models look similar to the true trajectories,  based on the group polarisation parameter $M(t)$ and comparing the changes of $M(t)$ in $t$ or the empirical distributions of $M(t)$, only our model using GP captures the group behaviors.

\begin{figure}[tbhp]
\centering
\includegraphics[width=0.47\textwidth]{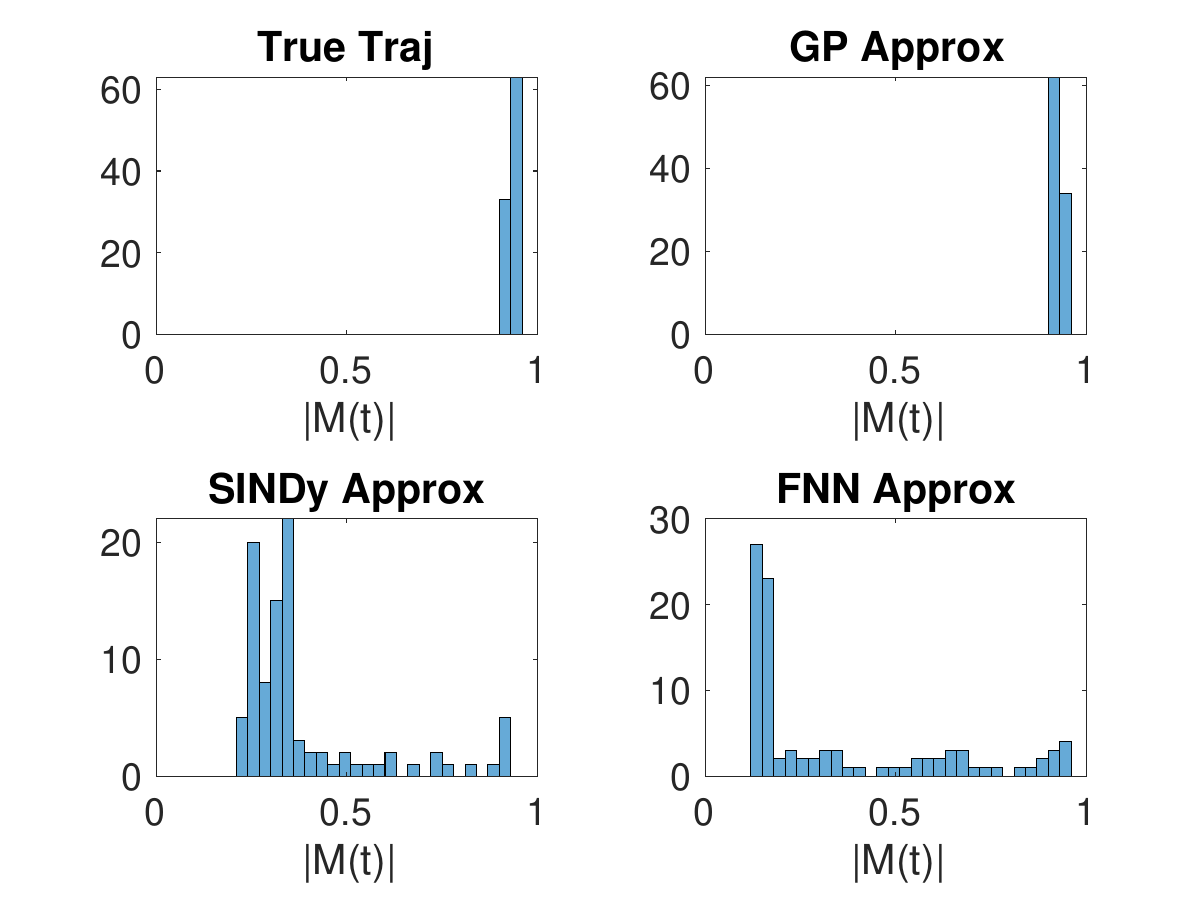}
\includegraphics[width=0.47\textwidth]{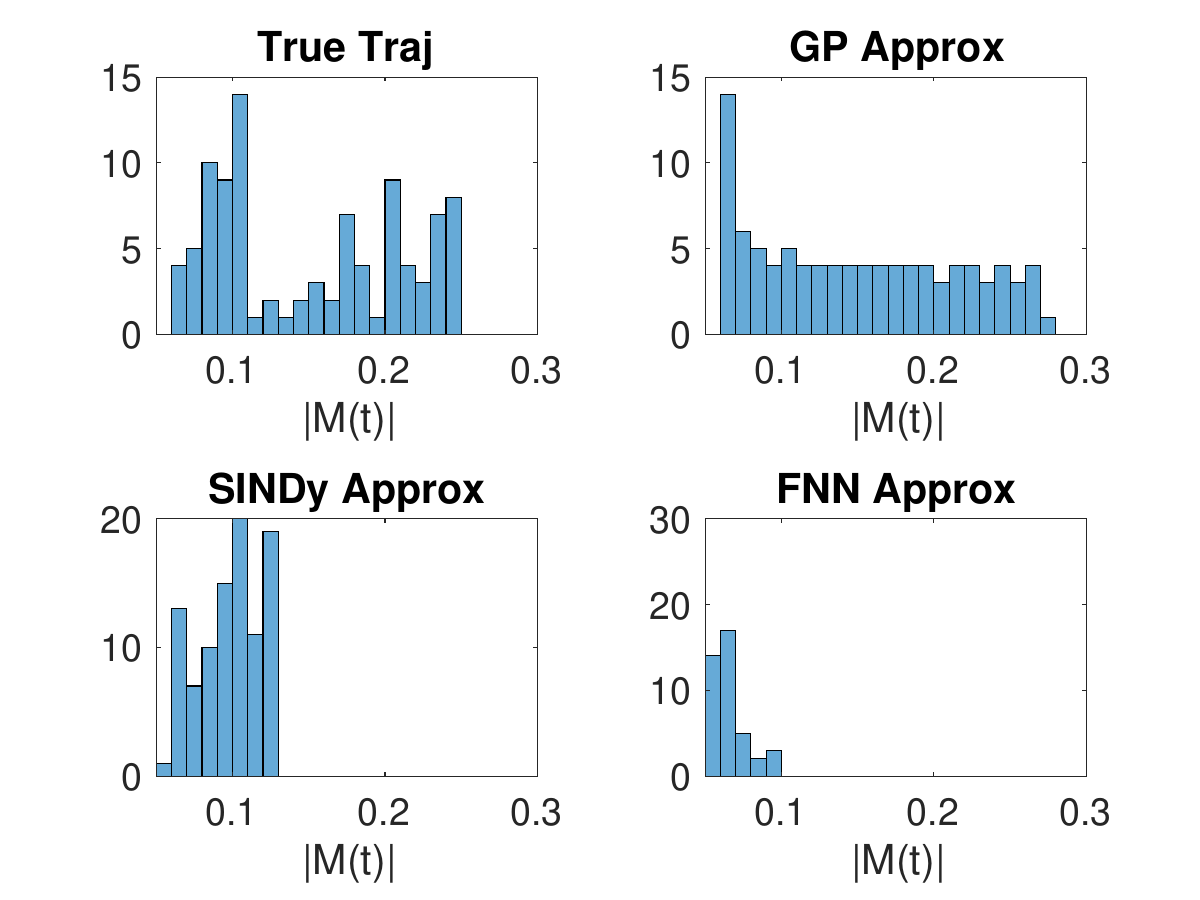}
\caption{Baseline comparisons using
the group polarisation parameter $M(t)$. Left: the empirical distribution of $M(t)$ in the flocking behavior example, the order-1 Wasserstein distances between true data and the predicted dynamics are $W(p_M^{true},p_M^{gp}) = 0.0112$, $W(p_M^{true},p_M^{SINDy}) = 0.5497$, and $W(p_M^{true},p_M^{NN}) = 0.5869$. Right: the empirical distribution of $M(t)$ in the milling behavior example, the order-1 Wasserstein distances between true data and the predicted dynamics are $W(p_M^{true},p_M^{gp}) = 0.0087$, $W(p_M^{true},p_M^{SINDy}) = 0.0530$, and $W(p_M^{true},p_M^{NN}) = 0.1055$.}
\label{fig:ex_real_compare1}
\end{figure}

We also compare our result with two other GP models in the milling behavior example, where the parameters $\balpha$ are not estimated properly: (1) we use the initial values of hyperparameters, i.e. let $\theta^E,\theta^A,\sigma$ and $\balpha$ equal $(1,1), (1,1), 0.001, (1,1)$, and do not train those hyperparameters; (2) we apply the noise-free model, i.e. do not consider noise and let $\sigma \equiv 0$. The results of these two GP models are shown in \cref{fig:ex_real_compare2}.

\begin{figure}[tbhp]
\centering
\includegraphics[width=0.3\linewidth,height=0.2\linewidth]{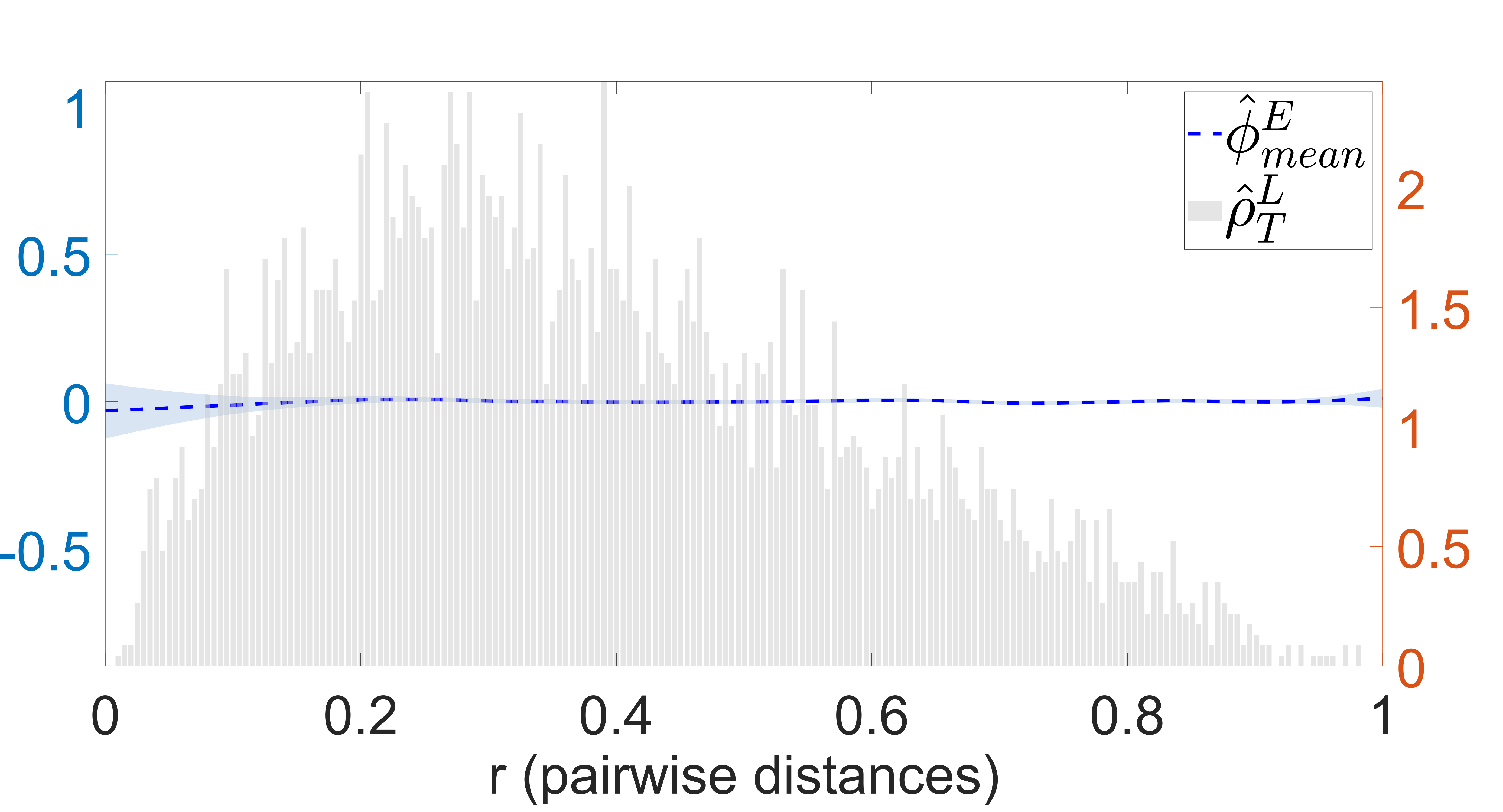}
\includegraphics[width=0.3\linewidth,height=0.2\linewidth]{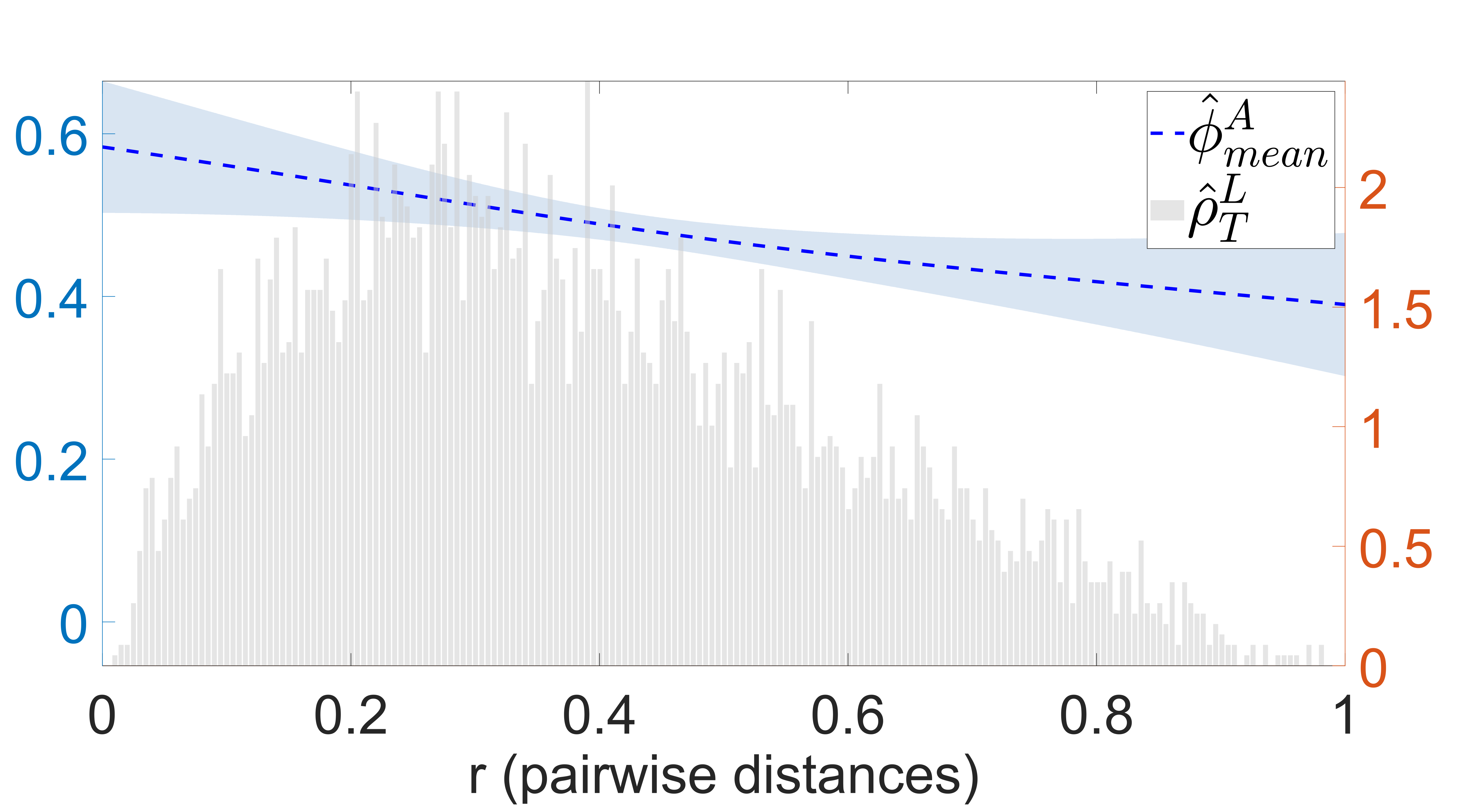}
\includegraphics[width=0.38\textwidth]{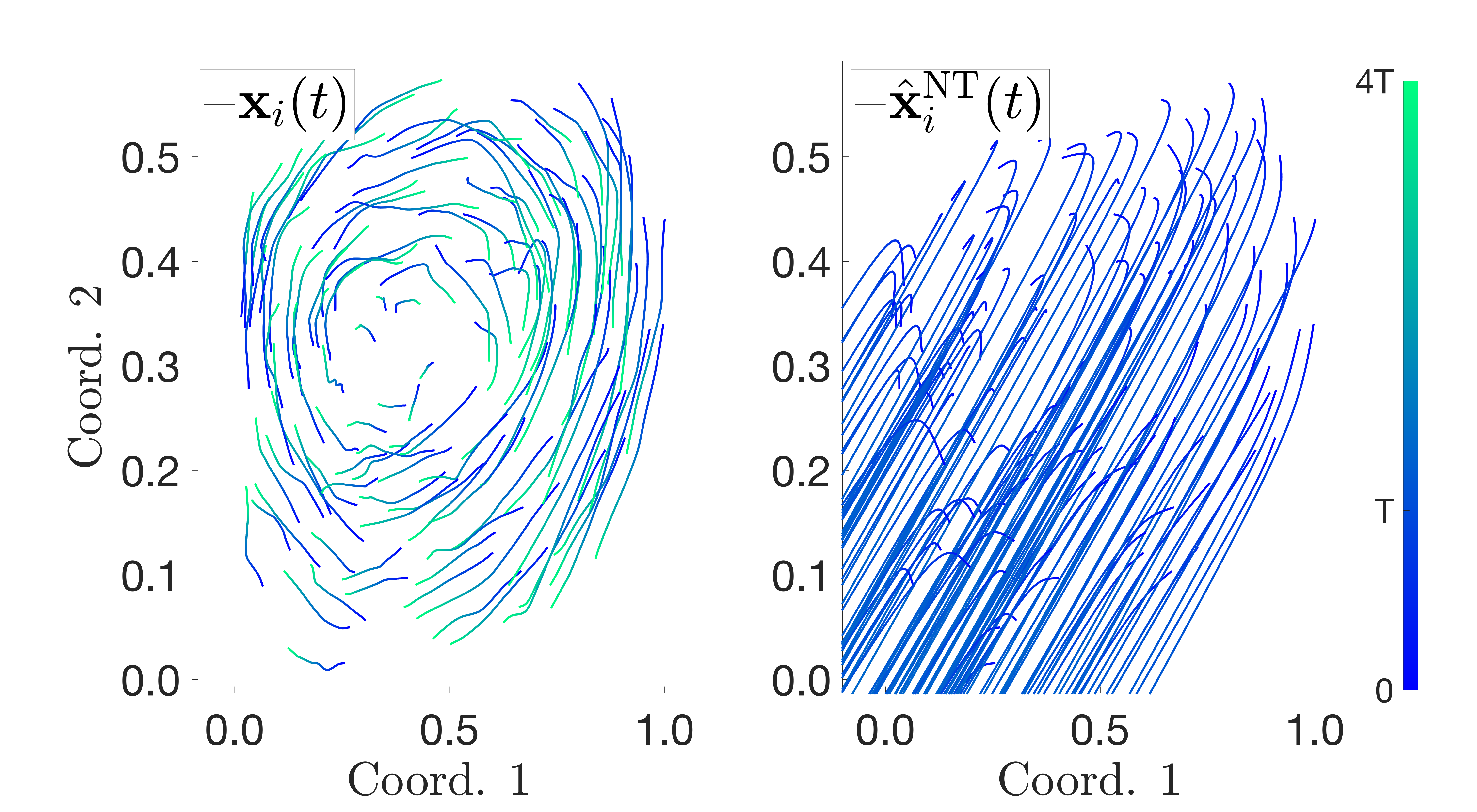}
\includegraphics[width=0.3\linewidth,height=0.2\linewidth]{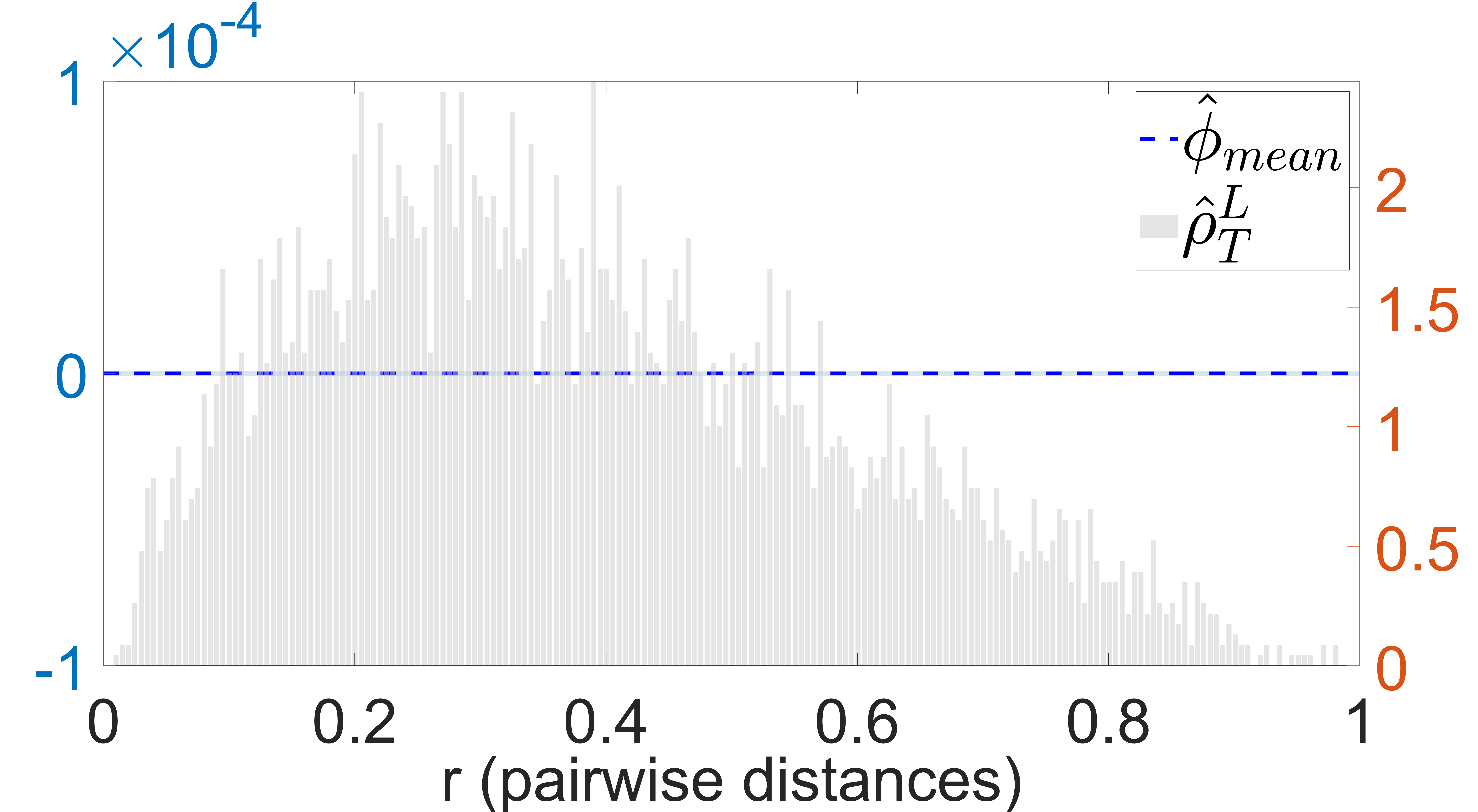}
\includegraphics[width=0.3\linewidth,height=0.2\linewidth]{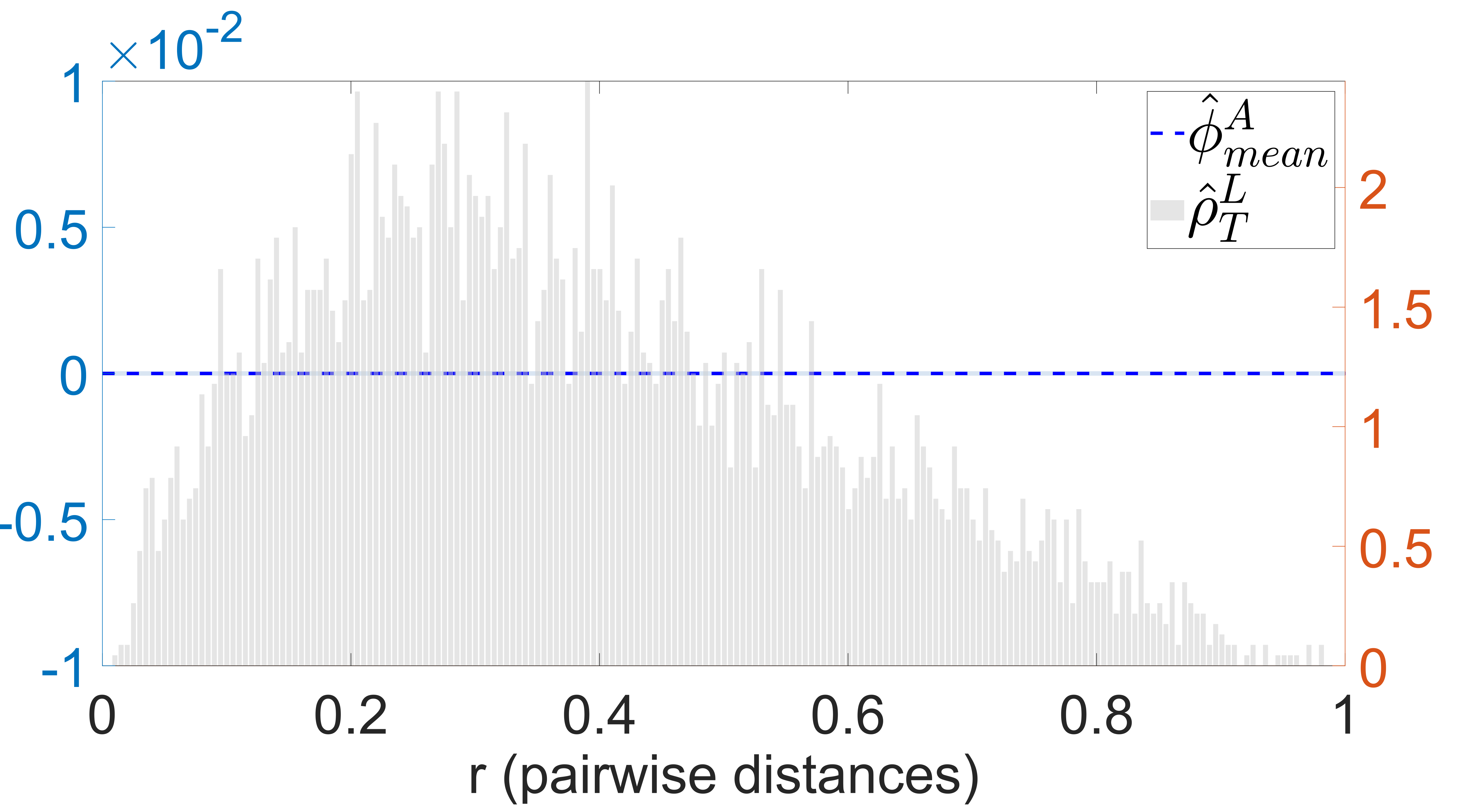}
\includegraphics[width=0.38\textwidth]{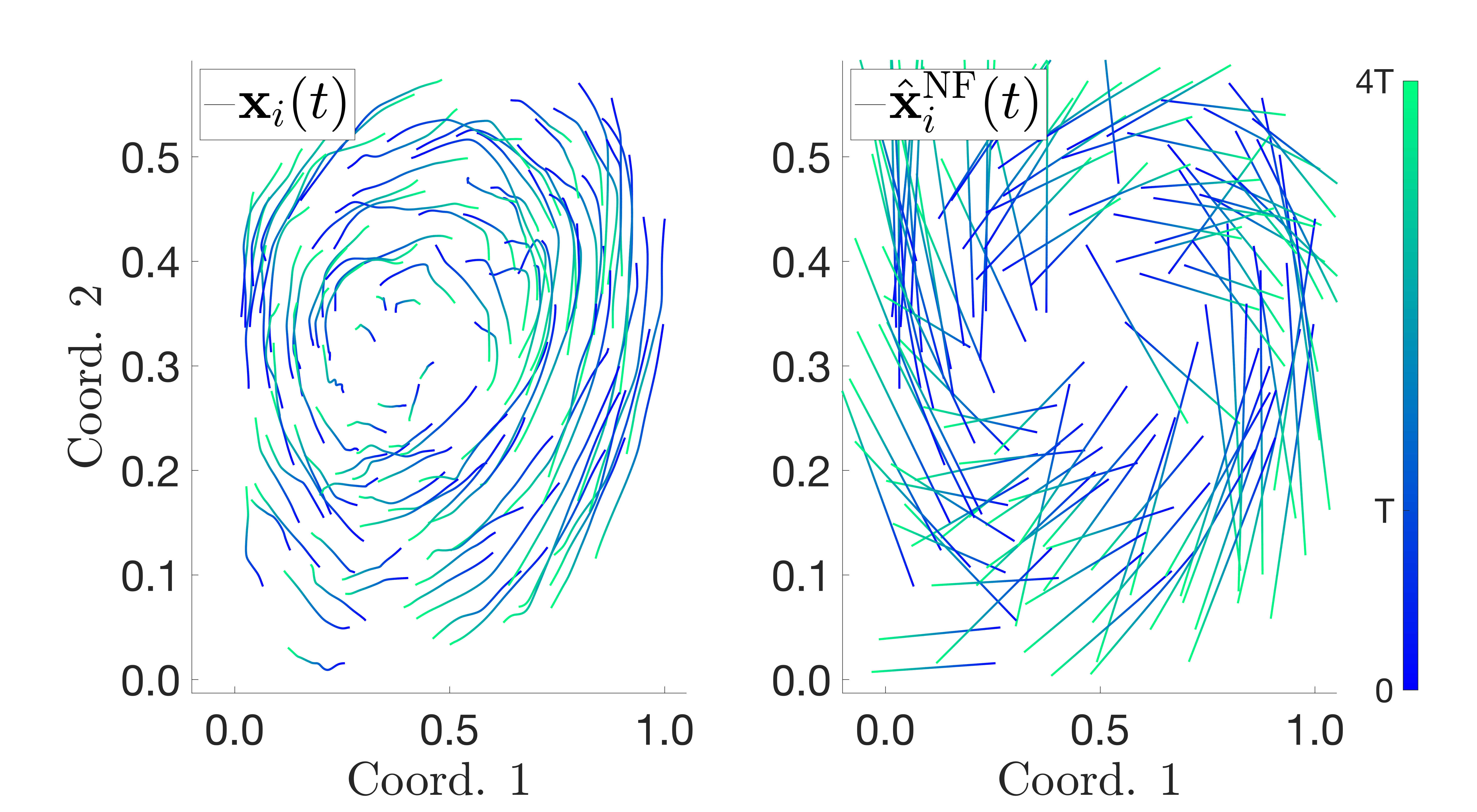}
\caption{Fitting into a Fish-Milling system (dim=248). Top: estimated $\intkernele$, $\intkernela$, and predictive trajectory with initial parameters, i.e. no training (NT) for hyperparameters $\theta^E,\theta^A,\sigma$, and $\balpha$; Bottom: estimated $\intkernele$, $\intkernela$, and predictive trajectory using noise-free (NF) model, i.e. $\hat\sigma = 0$. }
\label{fig:ex_real_compare2}
\end{figure}

\section{Acceleration Result Comparison}\label{accelNumerical}

We now present our acceleration (see \cref{accelex}) results for a 20-dimensional Fish Milling (FM) system with increasing observational data. When we have a larger amount of observational data, we will focus on learning $\sigma, \gamma,$ and $\beta$, and use $\theta^E = \theta^A = 1$ for a default prior with the Mat\'ern kernel. We will show the impact on kernel predictions is minimal. 

 We use $\nu = \frac 3 2$ for all examples below. All results shown are averaged over $10$ complete runs with standard deviation included where we used the same training data but with initialized hyperparameters uniformly at random from an interval centered at the ground truth with radius 0.5 in each trial. We use the randomized Gaussian Nystrom preconditioner \cite{randomnyst} for all tests with rank the floor of $\frac{30}{\log(12)} \cdot \log(\frac{NdML}{10})$. While we would ideally use the effective rank of our kernel matrix, this is expensive to compute in practice and we resort to empirical approximation for our trials.

\begin{figure}[!htb]
\centering

\includegraphics[width=0.45\linewidth]{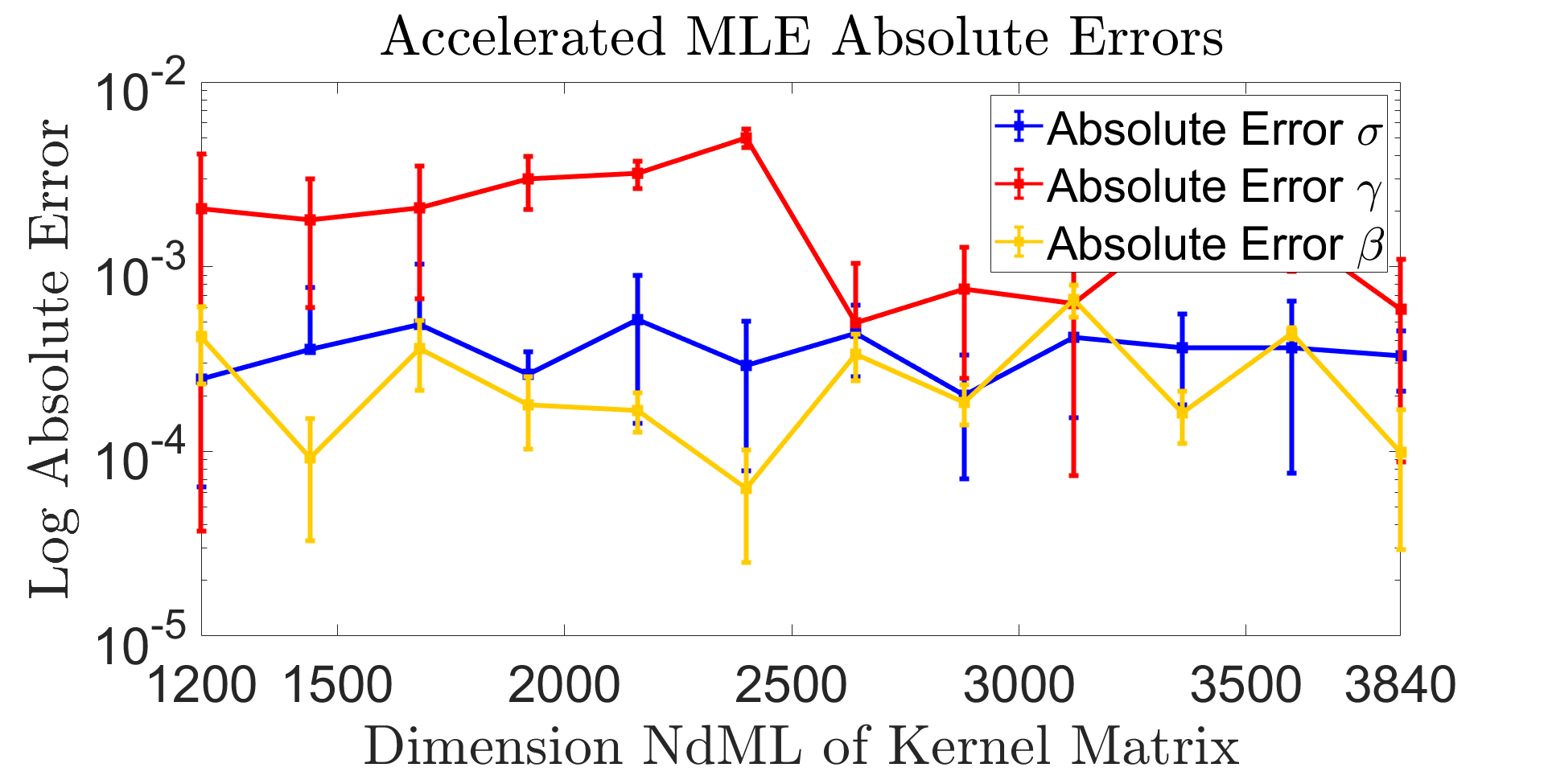}
\includegraphics[width=0.45\linewidth]{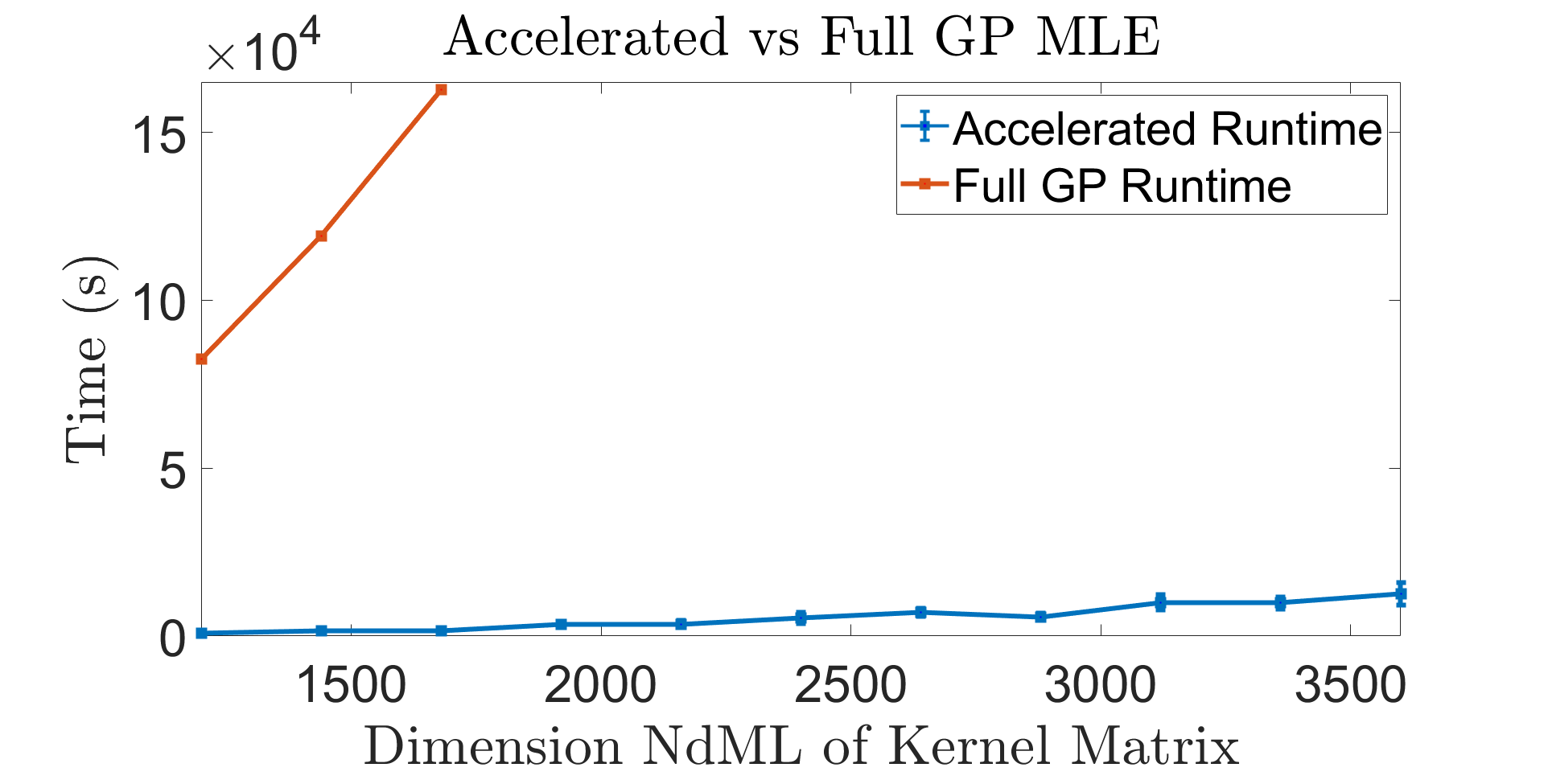}

\caption{Log plot of absolute error of learned hyperparameters $\gamma$ (red), $\beta$ (yellow), and $ \sigma$ (blue) for the FM system ($\{N,M,L\} = \{20,M,6\}$) with varying $M$. True values are $\sigma = 0.01, \gamma = 1.5, \beta = 0.5$. Shown also is a runtime comparison with Full GP.}
\label{fig:acceleration_mle}

\includegraphics[width=0.45\linewidth]{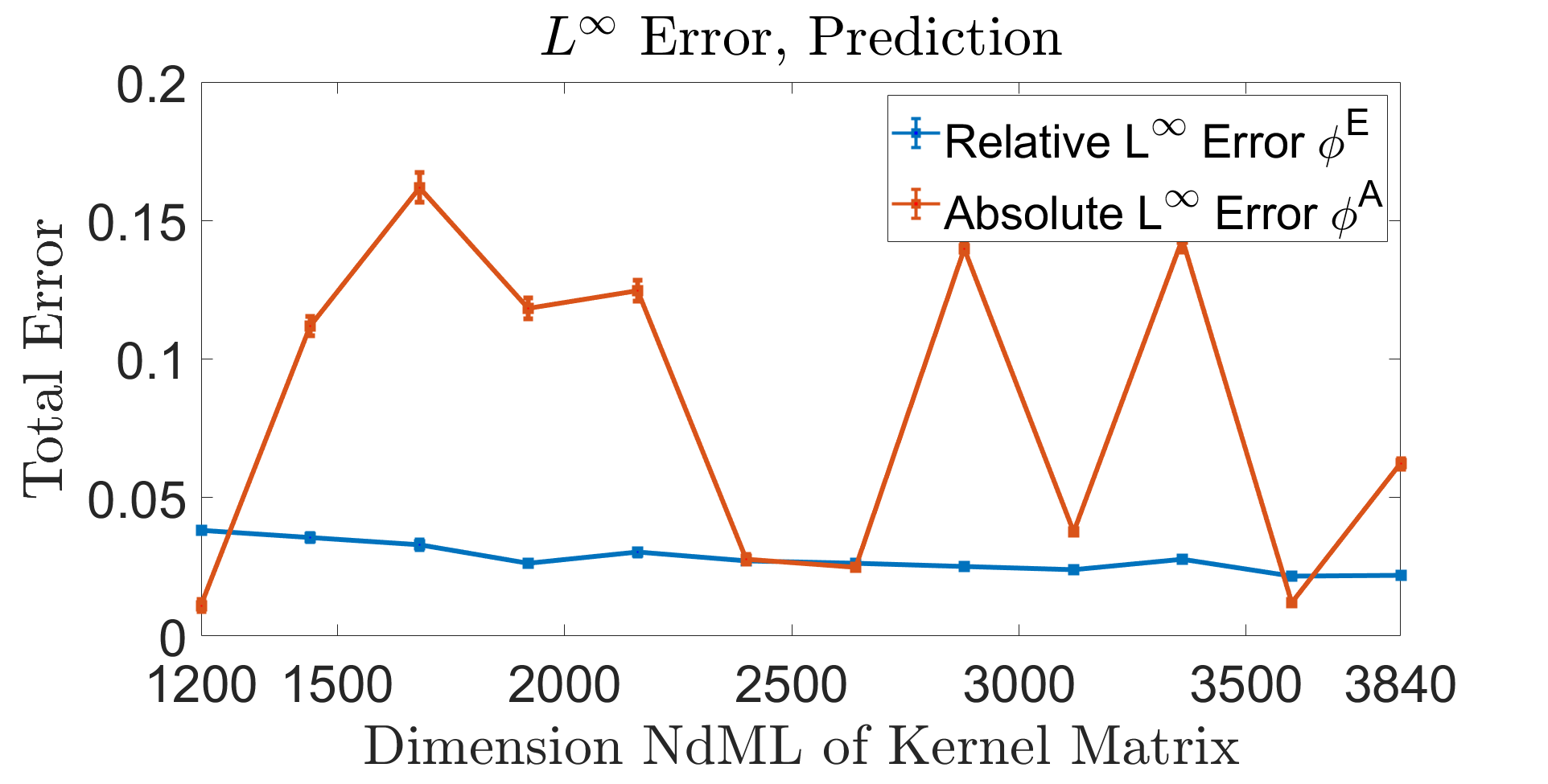}
\includegraphics[width=0.45\linewidth]{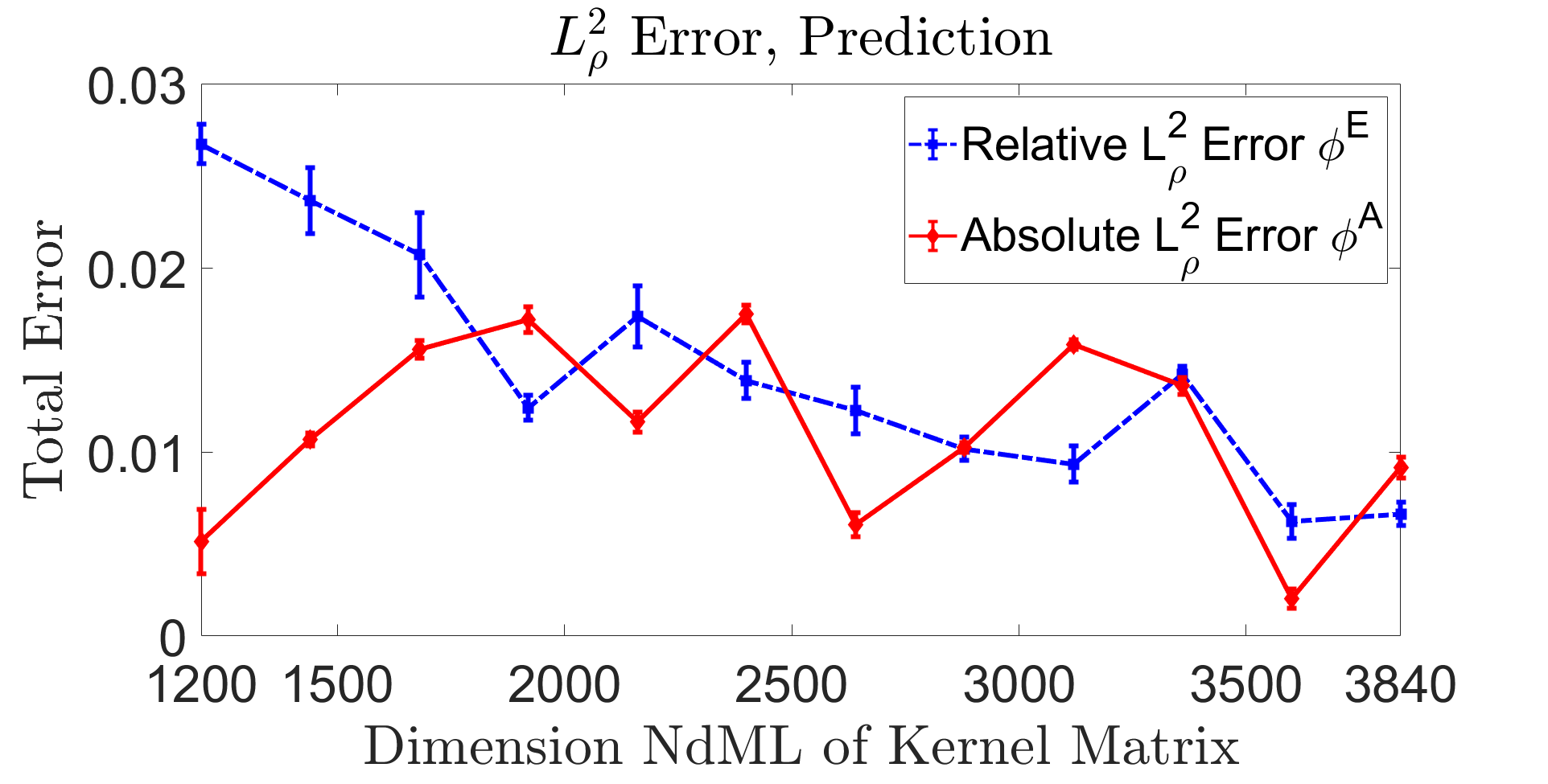}

\caption{Accelerated kernel prediction error for the experiments above. Left is the $L^{\infty}$ error, relative for $\phi^E$ and absolute for $\phi^A$ as $\phi^A = 0$ is the ground truth. Right is the $L_{\rho}^2$ error, relative for $\phi^E$ and absolute for $\phi^A$.}
\label{fig:acceleration_phi}

\includegraphics[width=0.45\linewidth]{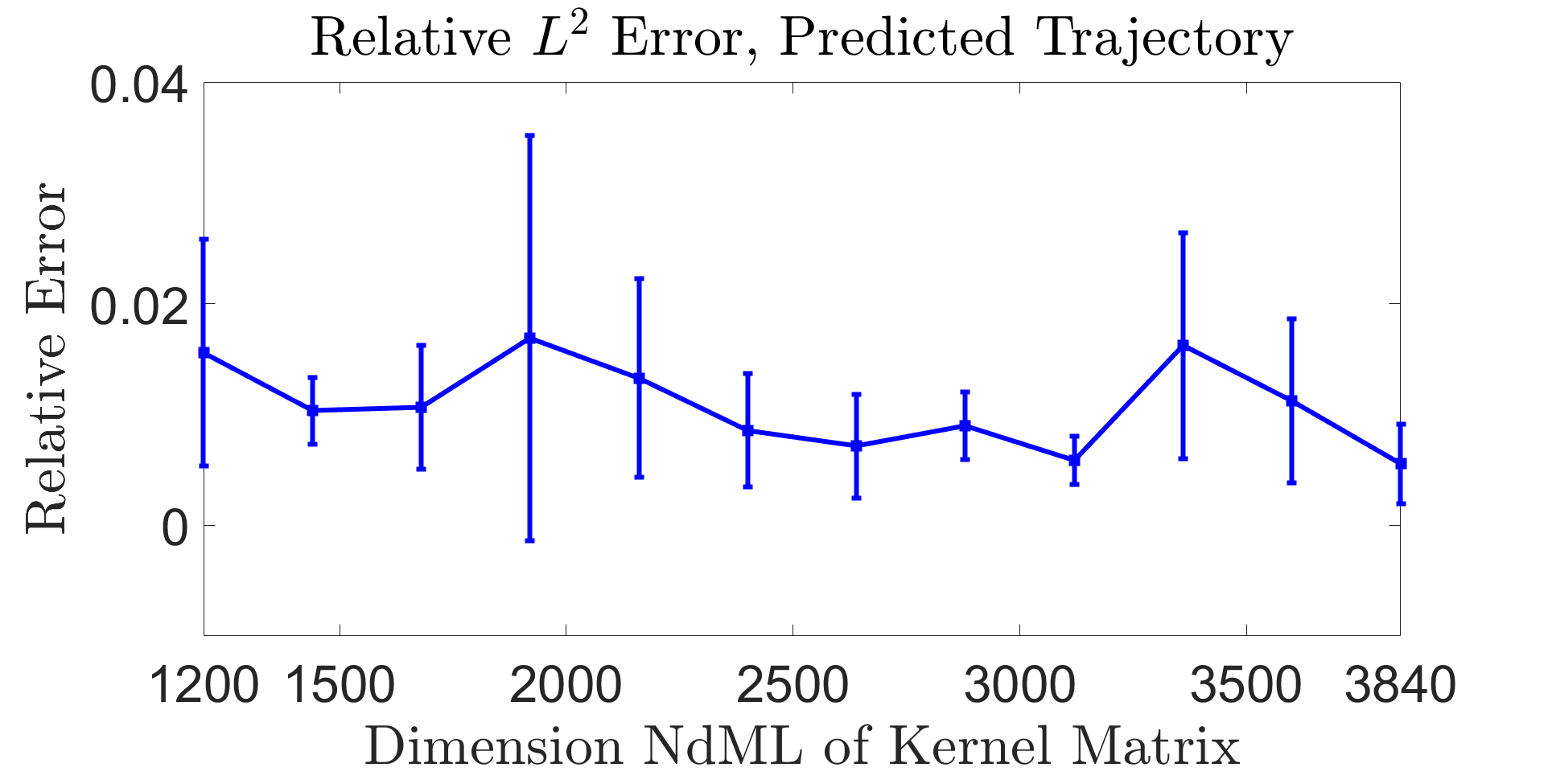}
\includegraphics[width=0.45\linewidth]{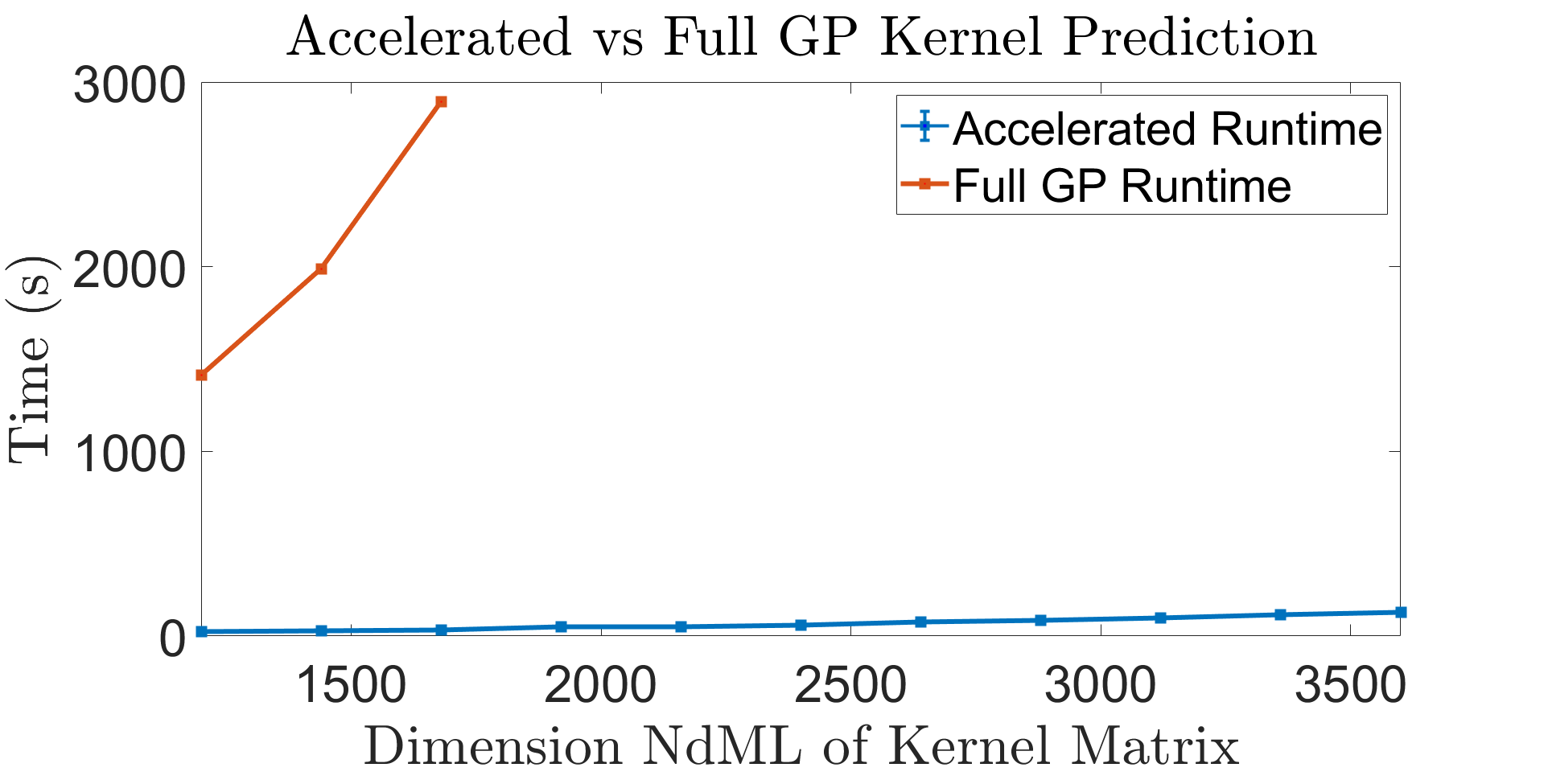}

\caption{Relative trajectory prediction error on testing data, FM system ($\{N,M,L\} = \{20,M,6\}$) with varying $M$. This plot uses test error on the full interval of $[0,10]$. Prediction error quickly goes to zero in testing. Shown also is a runtime comparison with full GP.}
\label{fig:acceleration_traj}

\end{figure}

\begin{table}[!htb]
\caption{System parameters in Fish Milling above}
\label{tab:ex_ODS_info} 
\centering
\small{
\small{\begin{tabular}{ cccccc}
\toprule
 $d$ & $N$ & $[0; T; T_f]$ & $\alpha = (\gamma, \beta)$ & $\mu_0^\bx$ & $\mu_0^\bv$\\
\midrule
2 & 20 & $[0,5,10]$ & $(1.5,0.5)$ & $\mathrm{Unif}([-1,1]^2)$ & $\mathrm{Unif}([0,0]^2)$\\
\bottomrule
\end{tabular}}  
}
\end{table}

Figure \ref{fig:acceleration_mle} shows that our hyperparameter learning method is able to accurately recover the hyperparameters $\sigma, \gamma, \beta$ with greatly improved runtime compared to the full GP method. Once these hyperparameters are learned, our acceleration can also be utilized for the prediction of the kernel, which also has a low observed error, see Figure \ref{fig:acceleration_phi}. Most errors of the kernel prediction occur away from the support of observed data in the FM system and do not affect the trajectory prediction of our system. This is quantified in the very low relative $L^2$ error of the predicted trajectories of the FM system using our predicted kernel as shown in Figure \ref{fig:acceleration_traj}.

These results provide clear evidence of successful acceleration options while maintaining highly acceptable accuracy. While the running time of accelerated MLE can still be expensive for prohibitively large data, the accelerated method scales much better than the fully explicit method and opens up exciting possibilities in modeling large datasets. We note that prediction also scales quite well and relies only upon PCG and preconditioner choice, allowing the usage of efficient cross-validation techniques for hyperparameter choice in certain classes of problems. 

Our central findings are the following:

\begin{itemize}
\item We have discovered that accurate hyperparameter recovery can be achieved using a small set of observational data, and using more training data does not necessarily improve the accuracy. This is due to the lack of consistency in the training of MLE, which is a well-known result in Gaussian process regression. We recommend that one should split a small subset for hyperparameter tuning and then use the full dataset for kernel learning. We have seen empirical success with this method.

\item In 10 trials with small $M$, we often observed one or two trials with relatively large recovery errors in hyperparameters. We removed these outliers from our data before plotting above. We attribute this to the instability of the Lanczos algorithm or the non-convexity of the optimization problem, as in these cases, we observed that the minimization of MLE stopped very early. Nonetheless, we would like to point out that even in these cases, we obtained very satisfying performance in kernel learning and trajectory prediction.

\item  There are additional opportunities for acceleration in kernel learning that depend on the specific problem and infrastructure available. For instance, in the case of $\nu = \frac{1}{2}$, we may exploit sparsity in a decomposition of the kernel matrix $K_{\rhsfo_{\bintkernel}}(\bbY,\bbY;\theta)$, as developed in \cite{sparsecgpaper}, while maintaining desired exactness. We leave the extension of this method to all half-integer $\nu$ values for future work. Furthermore, there are avenues for accelerating GP learning using modern hardware. With access to GPUs, one can parallelize the explicit construction of kernel matrices and the Lanczos algorithm calculations. These steps are embarrassingly parallel and allow for demanding much greater accuracy. 
\end{itemize}

\section{Final remarks and future work}

In this work, we present an approach based on Gaussian processes to perform the model selection of particle/agent-based models from scarce and noisy data. We propose efficient acceleration techniques to improve the scalability.  
The methodology is extendable to cover heterogeneous systems with multiple types of agents and external potentials. It is also possible to extend the learning approach to the mean-field limits of the particle models. Another line of future work is to apply the quantitative framework developed in this paper to design a data acquisition plan (active learning). The goal is to optimize the kernel learning  using the least amount of trajectory data by looking at their marginal pairwise distance distributions. We leave it as future work.

\section*{Acknowledgments}
Charles Kulick  was partially supported by   NSF DMS-2111303.  S.T. was partially supported by    Hellman Family Faculty Fellowship, and the NSF DMS-2111303.
S.T. would like to thank Hengrui Luo and Didong Li for their helpful discussions.

\section{Appendix}

\bigskip

\begin{appendices}

\section{Learning approach for model selection}\label{learningapproach}
Our learning approach is a generalization of the methodology proposed in \cite{learning2022}; to be self-contained, we state the detailed formulation here.

\begin{lemma}
\label{lemma:prior}

Let $\bintkernel = (\intkernele,\intkernela)$ be two Gaussian processes with mean zero and covariance function $K_\thetae, K_\thetaa : [0, R] \times [0, R] \to \mathbb{R}$ respectively, i.e., $\intkernel^{\mathrm{type}} \sim\mathcal{GP}(0,K_{\theta^{\mathrm{type}}}(r,r'))$, type = $E$ or $A$, and $\mbf{m}\bZ(t) =\force_{\mbf{\alpha}}(\bY(t)) + \rhsfo_\bintkernel(\bY(t))$ as defined in  \eqref{eq:2ndOrder_compact}. Then for any $t,t' \in [0,T]$, we have that,
 \begin{equation}
 \begin{bmatrix}
 \mbf{m}\bZ(t)\\ \mbf{m}\bZ(t')
 \end{bmatrix}
 \sim \mathcal{N} \left(
 \begin{bmatrix}
 \force_{\mbf{\alpha}}(\bY(t))\\ \force_{\mbf{\alpha}}(\bY(t'))
 \end{bmatrix}
 , K_{\rhsfo_{\bintkernel}}(\bY(t),\bY(t'))\right),
\label{eqZdist2}
\end{equation}
where $K_{\rhsfo_{\intkernel}}(\bX(t),\bX(t')))$ is the covariance matrix $
 \cov(\rhsfo_{\intkernel}(\bX(t)),\rhsfo_{\intkernel}(\bX(t'))) $
with $(i,j)$th block
\begin{eqnarray}
 \cov([\rhsfo_{\bintkernel}(\bY)]_i,[\rhsfo_{\bintkernel}(\bY')]_j)=
  \frac{1}{N^2}\sum_{k\neq i,k'\neq j}\big( K_\thetae(r^\bx_{ik},r^{\bx'}_{jk'})\mbf{r}^\bx_{ik}{\mbf{r}^{\bx'}_{jk'}}^T + K_\thetaa(r^\bx_{ik},r^{\bx'}_{jk'})\mbf{r}^\bv_{ik}{\mbf{r}^{\bv'}_{jk'}}^T \big)\notag,\\
\end{eqnarray} {see Table \ref{tab:2ndOrder_vecdef} for the definitions}.
\end{lemma}

\begin{proof}
For  $\intkernel \sim \mathcal{GP}(0,K_\theta(r,r'))$, and any $r,r' \in [0,R]$, we have that,
 \begin{eqnarray}
 \mathbb{E}[\intkernel(r)]&=&0,\\
 \cov[\intkernel(r),\intkernel(r')]&=&K_\theta(r,r').
\end{eqnarray}
Therefore, for any collection of states $\{r_i\}_{i=1}^n \subset [0,R]$, and $\{a_i\}_{i=1}^n, \{b_i\}_{i=1}^n \subset \mathbb{R}$, the linear operator on function values $\mathcal{L}(\{\intkernel(r_i)\}_{i=1}^n) : = (a_i \intkernel(r_i)+b_i)_{i=1}^n$ satisfies
\begin{equation}
    \mathcal{L}(\{\intkernel(r_i)\}_{i=1}^n) \sim \mathcal{N}(\mathrm{vec}(\{b_i\}_{i=1}^n), \Sigma_{\mathcal{L}(\intkernel)}),
\label{eq:phi_cov}
\end{equation}
where $\mathcal{N}$ denotes the Gaussian distribution, $\mathrm{vec}(\{b_i\}_{i=1}^n) \in \mathbb{R}^n$ is the vectorization of $\{b_i\}_{i=1}^n$, and the covariance matrix   $\Sigma_{\mathcal{L}(\intkernel)} = \{a_ia_jK_\theta(r_i,r_j)\}_{i,j=1}^{n} \in \mathbb{R}^{n\times n}$. 

Therefore, since $\intkernele$, $\intkernela$ are independent, and $\rhsfo_{\bintkernel}(\bY(t))$ is linear in $\bintkernel$, for any $t$, $t'$, we have that
 \begin{equation}
 \begin{bmatrix}
 \rhsfo_{\bintkernel}(\bY(t))\\
 \rhsfo_{\bintkernel}(\bY(t'))
 \end{bmatrix}
 \sim \mathcal{N} (\bm{0}, K_{\rhsfo_{\bintkernel}}(\bY(t),\bY(t'))),
\end{equation}
where $K_{\rhsfo_{\bintkernel}}(\bY(t),\bY(t')))$ is the covariance matrix  
\begin{eqnarray}
 \cov(\rhsfo_{\bintkernel}(\bY(t)),\rhsfo_{\bintkernel}(\bY(t')))=\big(\cov([\rhsfo_{\bintkernel}(\bY(t))]_i,[\rhsfo_{\bintkernel}(\bY(t')]_j)) \big)_{i,j=1}^{N,N},
\label{eqSigma2}
\end{eqnarray}
with $(i,j)$th block
\begin{eqnarray}
 \cov([\rhsfo_{\bintkernel}(\bY)]_i,[\rhsfo_{\bintkernel}(\bY')]_j)= \frac{1}{N^2}\sum_{k\neq i,k'\neq j}\big( K_\thetae(r^\bx_{ik},r^{\bx'}_{jk'})\mbf{r}^\bx_{ik}{\mbf{r}^{\bx'}_{jk'}}^T + K_\thetaa(r^\bx_{ik},r^{\bx'}_{jk'})\mbf{r}^\bv_{ik}{\mbf{r}^{\bv'}_{jk'}}^T \big)\notag,
\end{eqnarray}
Thus, by \eqref{eq:2ndOrder_compact}, the observation $\bZ$ in the model follows the Gaussian distribution
 \begin{equation}
 \begin{bmatrix}
 \mbf{m}\bZ(t)\\ \mbf{m}\bZ(t')
 \end{bmatrix}
 \sim \mathcal{N} (
 \begin{bmatrix}
 \force_{\mbf{\alpha}}(\bY(t))\\ \force_{\mbf{\alpha}}(\bY(t'))
 \end{bmatrix}
 , K_{\rhsfo_{\bintkernel}}(\bY(t),\bY(t'))).
\end{equation}

\end{proof}

Then suppose that the training data consists of \begin{align}\label{empiricaldata}
\{\bbY_M, \bbZ_{\sigma^2, M}\}=\{\bbX_{M},  \bbV_{M}, \bbZ_{\sigma^2,M}  \}
\end{align} with 
 
 \begin{align*}
\bbX_{M}&= \mathrm{Vec}\big(\{\bX^{(m,l)}\}_{m,l=1}^{M,L}\big) \in \mathbb{R}^{dNML},\\
\bbV_{M}&= \mathrm{Vec}\big(\{\bV^{(m,l)}\}_{m,l=1}^{M,L}\big) = \mathrm{Vec}\big(\{\dot\bX^{(m,l)}\}_{m,l=1}^{M,L}\big) \in \mathbb{R}^{dNML},\\
\bbZ_{\sigma^2,M}&=  \mathrm{Vec}\big(\{\bZ^{(m,l)}_{\sigma^2}\}_{m,l=1}^{M,L}\big) = \mathrm{Vec}\big(\{\ddot\bX^{(m,l)}+\sigma^2\mbf{\epsilon}^{(m,l)}\}_{m,l=1}^{M,L}\big) \in \mathbb{R}^{dNML}
\end{align*} 
where we observe the dynamics at $0=t_1< t_2<\cdots<t_L=T$; $m$ indexes trajectories corresponding to different initial conditions at $t_1=0$; $\bX^{(m,1)} \stackrel{i.i.d}{\sim} \mu_0^\bx$, $\bV^{(m,1)} \stackrel{i.i.d}{\sim} \mu_0^\bv$, $(\mu_0^\bx, \mu_0^\bv)$ are two independent probability measure on $\mathbb{R}^{dN}$; the noise term $\mbf{\epsilon}^{(m,l)}  \stackrel{i.i.d}{\sim} \mathcal{N}(\bm{0}, I_{dN})$; we assume that  $\mu_0 = (\mu_0^\bx, \mu_0^\bv)$ is independent of the  distribution of noise.

Applying Lemma \ref{lemma:prior}, we  now derive the negative log marginal likelihood for training parameters $\balpha$, $\mbf{\theta}$, and $\sigma$, with given observational data as specified above.

\begin{proposition}
\label{prop: liklihood}
Denote $\bY^{(m,l)}=\bY^{(m)}(t_l)$  and $\bZ^{(m,l)}_{\sigma^2}=\bZ^{(m)}(t_l)+\epsilon^{(m,l)}$ with i.i.d noise $\epsilon^{(m,l)} \sim \mathcal{N}(0, \sigma^2 I_{dN\times dN})$. Suppose we are given the training data set $(\bbY_M,\bbZ_{\sigma^2,M}): = $\\ 
$\{(\bY^{(m,l)},\bZ^{(m,l)}_{\sigma^2})\}_{m,l=1}^{M,L}$ for $M,L \in \mathbb{N}$, such that
\begin{equation}
    \bZ^{(m,l)}_{\sigma^2}  = \force_{\mbf{\alpha}}(\bY^{(m,l)}) + \rhsfo_{\bintkernel}(\bY^{(m,l)}) + \epsilon^{(m,l)},
\end{equation}
with $\force_{\mbf{\alpha}}$, $\rhsfo_{\bintkernel}$ defined in Table \ref{tab:2ndOrder_vecdef}. Then the {negative log} marginal likelihood of $\bbZ_{\sigma^2,M}$ given $\bbY_M$ and parameters $\balpha$, $\theta$, $\sigma$ satisfies
\begin{align}
    & -\log p(\mbf{m}\bbZ_{\sigma^2,M}|\bbY_M,\mbf{\alpha},\mbf{\theta},\sigma^2) \\
    &= \frac{1}{2} (\mbf{m}\bbZ_{\sigma^2,M} - \force_{\mbf{\alpha}}(\bbY_M))^T(K_{\rhsfo_{\bintkernel}}(\bbY_M,\bbY_M;\mbf{\theta}) + \sigma^2 I_{dNML})^{-1}(\mbf{m}\bbZ_{\sigma^2,M} - \force_{\mbf{\alpha}}(\bbY_M))\notag\\ 
    & \qquad +\frac{1}{2}\log|K_{\rhsfo_{\bintkernel}}(\bbY_M,\bbY_M;\mbf{\theta})+\sigma^2 I_{dNML}| + \frac{dNML}{2} \log 2\pi.
\label{apd eq:likelihood}
\end{align} 
where $I_{dNML}$ is the identity matrix of consistent size. 
\end{proposition}
 
\begin{proof} Using Lemma \ref{lemma:prior}, since $\epsilon^{(m,l)}$ is i.i.d Gaussian noise and is independent of the initial distributions, we have that
\begin{equation}
    \mbf{m}\bbZ_{\sigma^2,M} \sim \mathcal{N}(\force_{\mbf{\alpha}}(\bbY_M), K_{\rhsfo_{\bintkernel}}(\bbY_M,\bbY_M;\mbf{\theta}) + \sigma^2 I_{dNML}),
\end{equation}
where the mean vector $\force_{\mbf{\alpha}}(\bbY_M) = \mathrm{Vec}((\force_{\mbf{\alpha}}(\bY^{(m,l)}))_{m,l=1}^{M,L})\in \mathbb{R}^{dNML}$, and the covariance matrix $K_\intkernel(\bbY_M,\bbY_M;\theta) = \big(\cov(\rhsfo_{\bintkernel}(\bY^{(i,j)}),\rhsfo_{\bintkernel}(\bY^{(i',j')})) \big)_{i,i',j,j'=1}^{M,M,L,L} \in \mathbb{R}^{dNML \times dNML}$ can be computed by using \eqref{eqSigma2}. According to the properties of the Gaussian distribution, given $\bbY_M$ and parameters $\balpha$, $\theta$, $\sigma$, we have the negative log marginal likelihood function as shown in \eqref{apd eq:likelihood}.
\end{proof}
 
As mentioned in the main text, we can apply the gradient-based method \cite{liu1989limited}, to minimize the negative log marginal likelihood and solve for the hyperparameters $(\mbf{\alpha}, \mbf{\theta}, \sigma)$.  
\begin{proposition} 
\label{prop: derivs}
Let  $\bgamma = (K_{\rhsfo_{\bintkernel}}(\bbY_M,\bbY_M;\mbf{\theta})+ \sigma^2I)^{-1} (\mbf{m}\bbZ_{\sigma^2,M} - \force_{\mbf{\alpha}}(\bbY_M))$. The  partial derivatives of the marginal likelihood w.r.t. the parameters $\mbf{\alpha}$, $\mbf{\theta}$, and $\sigma$ can be computed as follows:
\begin{align}
\frac{\partial}{\partial \mbf{\alpha}_i} \log p(\mbf{m}\bbZ_{\sigma^2,M}|\bbY_M,\mbf{\alpha},\mbf{\theta},\sigma^2) &= \bgamma^T \frac{\partial \force_{\mbf{\alpha}}(\bbY_M)}{\partial \mbf{\alpha}_i}.
\label{eqalpha}\\
\frac{\partial}{\partial \mbf{\theta}_j} \log p(\mbf{m}\bbZ_{\sigma^2,M}|\bbY_M,\mbf{\alpha},\mbf{\theta},\sigma^2) &= \frac{1}{2} \mathrm{Tr}\left( (\bgamma \bgamma^T - (K_{\rhsfo_{\bintkernel}}(\bbY_M,\bbY_M;\mbf{\theta}) + \sigma^2I)^{-1}) \frac{\partial K_{\rhsfo_{\bintkernel}}(\bbY_M,\bbY_M;\mbf{\theta})}{\partial \mbf{\theta}_j}\right).
\label{eqtheta}\\
\frac{\partial}{\partial \sigma} \log p(\mbf{m}\bbZ_{\sigma^2,M}|\bbY_M,\mbf{\alpha},\mbf{\theta},\sigma^2) &=  \mathrm{Tr}\left( (\bgamma \bgamma^T - (K_{\rhsfo_{\bintkernel}}(\bbY_M,\bbY_M;\mbf{\theta}) + \sigma^2I)^{-1}) \right)\sigma.
\label{eqsigma}
\end{align}

\end{proposition}

With the updated prior $\intkernel$ from $\mbf{\theta}$, and the parameters $\balpha$, $\sigma$, we show the detailed derivation of our estimators for the prediction $\intkernel(r^*)$ at $r^\ast \in [0,R]$.

\begin{theorem}
Suppose we are given the training data set $(\bbY_M,\bbZ_{\sigma^2,M}): = \{(\bY^{(m,l)},\bZ^{(m,l)}_{\sigma^2})\}_{m,l=1}^{M,L}$ defined in Proposition \ref{prop: liklihood}, and the hyperparameters $(\balpha, \mbf{\theta}, \sigma)$ are known. Then for any $r^\ast \in [0,R]$, type = $E$ or $A$, $\intkernel^{\mathrm{type}}(r^\ast)$ satisfies
\begin{equation}
    p(\intkernel^{\mathrm{type}}(r^\ast)|\bbY_M,\bbZ_{\sigma^2,M}) \sim \mathcal{N}(\bar{\intkernel}^{\mathrm{type}},var(\bar{\intkernel}^{\mathrm{type}})),
\end{equation}
where
\begin{align}
    \bar{\intkernel}^{\mathrm{type}} &= K_{\intkernel^{\mathrm{type}},\rhsfo_\bintkernel}(r^\ast,\bbY_M)(K_{\rhsfo_{\bintkernel}}(\bbY_M,\bbY_M) + \sigma^2I_{dNML})^{-1}(\mbf{m}\bbZ_{\sigma^2,M} - \force_{\mbf{\alpha}}(\bbY_M)),
\label{apd eq:estimated phi}    \\
    var(\bar\intkernel^{\mathrm{type}}) &= K_{\theta^{\mathrm{type}}}(r^\ast,r^\ast) - K_{\intkernel^{\mathrm{type}},\rhsfo_\bintkernel}(r^\ast,\bbY_M)(K_{\rhsfo_{\bintkernel}}(\bbY_M,\bbY_M) + \sigma^2I_{dNML})^{-1}K_{\rhsfo_\bintkernel,\intkernel^{\mathrm{type}}}(\bbY_M,r^\ast).
    \label{apd eq:estimated var phi}
\end{align}
and $K_{\rhsfo_\bintkernel,\intkernel^{\mathrm{type}}}(\bbY_M, r^*) = K_{\intkernel^{\mathrm{type}},\rhsfo_\bintkernel}(r^*,\bbY_M)^T$ denotes the covariance matrix between $\rhsfo_{\bintkernel}(\bbY_M)$ and $\intkernel^{\mathrm{type}}(r^*)$. 
\end{theorem}

\begin{proof}
Since $\rhsfo_{\bintkernel}(\bbY_M)$ is defined componentwisely as in \eqref{eq:2ndOrder_compact}, for any $r^\ast \in [0,R]$, we have that
  \begin{equation}
    \begin{bmatrix}
    \rhsfo_{\bintkernel}(\bbY_M)\\
    \intkernel^{\mathrm{type}}(r^\ast)
    \end{bmatrix}
    \sim \mathcal{N} \left( 0,
    \begin{bmatrix}
    K_{\rhsfo_{\bintkernel}}(\bbY_M, \bbY_M) & K_{\rhsfo_\bintkernel,\intkernel^{\mathrm{type}}}(\bbY_M, r^\ast)\\
    K_{\intkernel^{\mathrm{type}},\rhsfo_\bintkernel}(r^\ast, \bbY_M) & K_{\theta^{\mathrm{type}}}(r^\ast,r^\ast)
    \end{bmatrix}
    \right),
\end{equation} 
where $K_{\rhsfo_\bintkernel}(\bbY_M, \bbY_M)$ is the covariance matrix between $\rhsfo_{\bintkernel}(\bbY_M)$ and $\rhsfo_{\bintkernel}(\bbY_M)$ as we defined in Proposition \ref{prop: liklihood}, and $K_{\rhsfo_\bintkernel,\intkernel^{\mathrm{type}}}(\bbY_M, r^*) = K_{\intkernel^{\mathrm{type}},\rhsfo_\bintkernel}(r^*,\bbY_M)^T$ is the covariance matrix between $\rhsfo_{\bintkernel}(\bbY_M)$ and $\intkernel^{\mathrm{type}}(r^*)$, i.e., $K_{\rhsfo_\bintkernel,\intkernel^{\mathrm{type}}}(\bbY_M, r^*) = (\cov(\rhsfo_{\bintkernel}(\bY^{(m,l)}),\intkernel^{\mathrm{type}}(r^\ast)))_{m,l=1}^{M,L}$ and the i-th component of $\cov(\rhsfo_{\bintkernel}(\bY^{(m,l)}),\intkernel^{\mathrm{type}}(r^\ast))$ is computed by
\begin{align}
    \cov([\rhsfo_{\bintkernel}(\bY^{(m,l)})]_i, \intkernele(r^\ast)) &= \frac{1}{N} \sum_{k \neq i} K_\thetae(r_{ik}^{\bX^{(m,l)}}, r^\ast) \mbf{r}_{ij}^{\bX^{(m,l)}},\\
    \cov([\rhsfo_{\bintkernel}(\bY^{(m,l)})]_i, \intkernela(r^\ast)) &= \frac{1}{N} \sum_{k \neq i} K_\thetaa(r_{ik}^{\bX^{(m,l)}}, r^\ast) \mbf{r}_{ij}^{\bV^{(m,l)}}.
\end{align}
Note that $\mbf{m}\bZ^{(m,l)}_{\sigma^2}  = \force_{\mbf{\alpha}}(\bY^{(m,l)}) + \rhsfo_{\intkernel}(\bX^{(m,l)}) + \epsilon^{(m,l)}$ with i.i.d noise $\epsilon^{(m,l)} \sim \mathcal{N}(0, \sigma^2 I_{dN})$ for all $(m,l)$, so we have
  \begin{equation}
    \begin{bmatrix}
    \mbf{m}\bbZ_{\sigma^2,M} - F_{\mbf{\alpha}}(\bbY_M)\\
    \intkernel^{\mathrm{type}}(r^\ast)
    \end{bmatrix}
    \sim \mathcal{N} \left( 0,
    \begin{bmatrix}
    K_{\rhsfo_{\intkernel}}(\bbY_M, \bbY_M) + \sigma^2 I_{dNML} & K_{\rhsfo_\intkernel,\intkernel^{\mathrm{type}}}(\bbY_M, r^\ast)\\
    K_{\intkernel^{\mathrm{type}},\rhsfo_\bintkernel}(r^\ast, \bbY_M) & K_{\theta^{\mathrm{type}}}(r^\ast,r^\ast)
    \end{bmatrix}
    \right),
\end{equation} 
Therefore, based on the properties of the joint Gaussian distribution (see 
Lemma \ref{lemma: conditioning Gaussian}), conditioning on $(\bbY_M, \bbZ_{\sigma^2,M})$, we have that
\begin{equation}
    p(\intkernel^{\mathrm{type}}(r^\ast)|\bbY_M,\bbZ_{\sigma^2,M},r^\ast) \sim \mathcal{N}(\bar{\intkernel}^{\mathrm{type}},var(\bar\intkernel^{\mathrm{type}})),
\end{equation}
where $\bar{\intkernel}^{\mathrm{type}}$ and $var(\bar\intkernel^{\mathrm{type}})$ are defined as in \eqref{apd eq:estimated phi} and \eqref{apd eq:estimated var phi}.
\end{proof}

\section{Psuedocode for Acceleration}\label{secA1}

In this section, we discuss in detail the acceleration of the computations used in our GP framework. We first review the bottleneck in the computation: our likelihood function evaluation is very slow, as it involves inverting the kernel matrix $(K_{\rhsfo_{\bintkernel}}(\bbY,\bbY;\theta) + \sigma^2I)^{-1}$ and computing the log determinant $\log \det( K_{\rhsfo_{\bintkernel}}(\bbY,\bbY;\theta)+\sigma^2I)$. We also require evaluation of the gradient for exact optimization, which further requires evaluation of the trace $\text{Tr}((K_{\rhsfo_{\bintkernel}}(\bbY,\bbY;\theta) + \sigma^2I)^{-1} \frac{\partial K_{\rhsfo_{\bintkernel}}(\bbY,\bbY;\theta)}{\partial \theta_i})$ for each parameter $\theta_i$ as shown in \cref{prop: derivs}.

Our primary goal is to avoid explicit inversion of the kernel matrix entirely by utilizing the Preconditioned Conjugate Gradient (PCG) algorithm, see \cref{Algorithm:PCG}. PCG is an iterative method that can solve systems $Ax = b$ for $x$ without explicitly inverting $A$ through clever choices of update at each step. This algorithm is central for scalability when solving large-scale linear systems with positive definite matrices in the numerical linear algebra literature. Note that the standard CG method is unlikely to work, as our kernel matrix is likely to be very ill-conditioned. An efficient preconditioner will be necessary to avoid extremely slow convergence. As mentioned in the main paper, we recommend the Randomized Gaussian Nystrom preconditioner for improving performance.

\begin{algorithm}[!htb]
\algorithmicrequire\ $\color{black} K_{\rhsfo_{\bintkernel}}(\bbY,\bbY;\theta) + \sigma^2I \color{black}$ (matrix-vector multiplication of kernel), $\color{black} P \color{black}$ (preconditioner), $\bm{b}$ (target vector), $\bm{x_0}$ (initial guess), $\mathrm{errorTol}$ (error tolerance), $t$ (iterations)
\begin{algorithmic} [1]

\STATE $r_0 := \bm{b} - \color{black} (K_{\rhsfo_{\bintkernel}}(\bbY,\bbY;\theta) + \sigma^2I)\color{black} \bm{x_0}$

\STATE $z_0, d_0  := \color{black}P\color{black}r_0$

\STATE while $\norm{r_n} > \mathrm{errorTol}$ and $n < t$

\STATE \hspace{1.5em} $v_n := \color{black}(K_{\rhsfo_{\bintkernel}}(\bbY,\bbY;\theta) + \sigma^2I)\color{black} d_{n-1}$

\STATE \hspace{1.5em} $\alpha_n := \frac{r_{n-1}^T  z_{n-1}}{d_{n-1}^T v_n}$

\STATE \hspace{1.5em} $x_n := x_{n-1} + \alpha_n d_{n-1}$

\STATE \hspace{1.5em} $r_n := r_{n-1} - \alpha_n v_n$

\STATE \hspace{1.5em} $z_n := \color{black} P \color{black} r_n$

\STATE \hspace{1.5em} $\beta_n := \frac{z_n^T r_n}{z_{n-1}^T r_{n-1}}$

\STATE \hspace{1.5em} $d_n := z_n + \beta_n d_{n-1}$

\end{algorithmic}

\algorithmicensure\ $x_n$ (solution to $\color{black}(K_{\rhsfo_{\bintkernel}}(\bbY,\bbY;\theta) + \sigma^2I)\color{black}\bm{x} = \bm{b}$),

\hspace{1.5cm} $\{(\alpha_i, \beta_i) \text{ for all } i \leq n\}$ (exclusively for constructing Lanczos weights)
\caption{{\bf Preconditioned CG for solving $(K_{\rhsfo_{\bintkernel}}(\bbY,\bbY;\theta) + \sigma^2I)\bm{x} = \bm{b}$} \label{Algorithm:PCG}}
\end{algorithm}

Now we can solve the problem of slow likelihood function evaluations. Instead of inversion, our proper preconditioner will allow us to apply PCG and reduce the computational complexity from cubic for inversion to quadratic, see \cref{Algorithm:PCG}. Note that PCG is only limited by the runtime of matrix-vector multiplication, and in the presence of sparsity or other structural features that allow for linear time matrix-vector multiplication, the complexity of PCG will also reduce to linear time. This can be accomplished in the $\nu = \frac 1 2$ case using \cite{gu22}. %

Then we consider the log determinant evaluation. Using stochastic Lanczos quadrature, we can instead compute an estimator for $\text{Tr(log(}\color{black}(K_{\rhsfo_{\bintkernel}}(\bbY,\bbY;\theta) + \sigma^2I)\color{black}))$, see \cref{Algorithm:trace}. This algorithm requires quadrature weights, but these can be efficiently recovered by running the PCG algorithm and arranging $\alpha, \beta$ in a tridiagonal matrix, as seen in \cite{gpy2021}. Then we apply stochastic trace estimation, as developed in \cite{wenger22}. These methods also extend to gradient calculations.

\begin{algorithm}[!htb]
\algorithmicrequire\ $\color{black}(K_{\rhsfo_{\bintkernel}}(\bbY,\bbY;\theta) + \sigma^2I)\color{black}$ (matrix-vector multiplication for kernel matrix), $\color{black} P \color{black}$ (preconditioner), $n$ (number of test vectors), $m$ (number of Lanczos coefficients)
\begin{algorithmic} [1]

\STATE for $i$ from $1$ to $n$

\STATE \hspace{1.5em} $v_n \sim \text{Rademacher}$ \hfill\COMMENT{draw from Rademacher distribution}

\STATE \hspace{1.5em} $T := PCG(\color{black}(K_{\rhsfo_{\bintkernel}}(\bbY,\bbY;\theta) + \sigma^2I)\color{black}, \color{black}P \color{black}, v_n, \ell = m)$ \hfill\COMMENT{get Lanczos coefficients}

\STATE \hspace{1.5em} $[W, \lambda] := eig(T)$

\STATE \hspace{1.5em} for $j$ from $1$ to $m$

\STATE \hspace{3.0em} $\gamma_i := \gamma_i + W_{1,j}^2 \log(\lambda_j)$

\STATE $tr_{est} := \log \det(\color{black}P\color{black}) + \frac{NdML}{n} \sum_{i=1}^n \gamma_i$

\end{algorithmic}
\algorithmicensure\ $tr_{est}$ (estimated trace of $\log(\color{black}(K_{\rhsfo_{\bintkernel}}(\bbY,\bbY;\theta) + \sigma^2I)\color{black})$)
\caption{{\bf Stochastic Trace Estimation with Lanczos Quadrature} \label{Algorithm:trace}}
\end{algorithm}

When choosing a preconditioner, we must have a method for fast and accurate computation of matrix-vector multiplication by $\color{black}P^{-1}\color{black}$ and evaluation of $\log \det (P)$, as these are necessary operations in the above algorithms.

One widely applicable class of preconditioners for positive semi-definite matrices is the low-rank Nystrom approximation. The central idea is to create a low-rank approximation $P$ of a matrix $A$ of interest, with the expectation that $P^{-1}A$ will have a condition number close to $1$. One common implementation is to subsample $r$ columns of the matrix and use these to construct an approximation for the missing entries with rank at most $r$.

The randomized Gaussian Nystrom preconditioner builds on this idea. Written in a general form, we have $P = A \Omega (\Omega^T A \Omega)^{\dagger} \Omega^T A^T$ for a chosen matrix $\Omega \in \mathbb{R}^{NdML \times r}$. Column subsampling is a special case where columns of the matrix $\Omega$ have a single non-zero entry of the unit $1$. However, $\Omega$ can also be populated with randomized Gaussian entries. This idea, developed in \cite{randomnyst}, has resulted in better empirical performance and enjoys theoretical support. For an implementation see \cref{Algorithm:Nystrom}.

\begin{algorithm}[!htb]
\algorithmicrequire\ $\color{black}K_{\rhsfo_{\bintkernel}}(\bbY,\bbY;\theta) \color{black}$ (matrix-vector multiplication for kernel matrix), $\sigma^2$ (noise hyperparameter), $r$ (rank of preconditioner)
\begin{algorithmic} [1]

\STATE $\Omega \sim$ Standard Gaussian $\in \mathbb{R}^{NdML \times r}$

\STATE $R = \text{QR}(\Omega)$ \hfill using economy QR

\STATE $Y = \color{black}K_{\rhsfo_{\bintkernel}}(\bbY,\bbY;\theta) \color{black} R$

\STATE $\nu = \text{eps}(||Y||_{F})$

\STATE $Y_{\nu} = Y + \nu R$

\STATE $C = \text{chol}(R^T Y_{\nu})$

\STATE $B = Y_{\nu} / C$
\STATE $U, \Sigma = \text{svd}(B)$ \hfill using economy svd

\STATE $\Lambda = \text{max}(0, \Sigma^2 - \nu I)$

\STATE $P^{-1} = (\Lambda(-1) + \sigma^2) U (\Lambda + \sigma^2I)^{-1} U^T + I - UU^T$

\STATE $\log \det(P) = \text{sum(sum(}\log(\frac{1}{\Lambda + \sigma^2})))$

\end{algorithmic}
\algorithmicensure\ $P^{-1}$, $\log \det (P)$ (needed preconditioner quantities)
\caption{{\bf Randomized Gaussian Nystrom Preconditioner} \label{Algorithm:Nystrom}}
\end{algorithm}

\section{Proof of the Coercivity condition \JF{}{in \cref{subsec:wellposedness}}} \label{secB2}

\begin{theorem}\label{2ndordersingle:coercivity2}
Consider  $\rho_{\bY}=\begin{bmatrix}\rho_{\bX}\\ \rho_{\bV} \end{bmatrix}$, where $\rho_{\bY}$  is the product of $N$ independent and identical measures with compact support on $\mathbb{R}^d$ and $\rho_{\bV}$ is defined in the same way and is independent of $\rho_{\bX}$.  Then we have 

\begin{align}\label{coercivityidex2}
\|\rhsfo_{\bintkernelvar}\|^2_{L^2(\rho_{\bY})}\geq \frac{N-1}{N^2}\|\intkernelvare\|^2_{L^2( \tilde\rho_{r}^{E})} + \frac{N-1}{N^2}\|\intkernelvara\|^2_{L^2( \tilde\rho_{r}^{A})}
\end{align}
 \end{theorem}
\begin{proof} 

Following the definition of measure $\rho_{\bY}$ and the norm in \eqref{normed}, we have %
\begin{align}
 \|\rhsfo_{\bintkernelvar}\|^2_{L^2(\rho_{\bY})}  &= \frac{1}{N}\sum_{i=1}^N\left\| \sum_{i'=1}^N \frac{1}{N} \Big[\intkernelvar^E (|\bx_{i'} - \bx_i|)(\bx_{i'} - \bx_i) + \intkernelvar^A( |{\bx_{i'} - \bx_i}|)(\bv_{i'} - \bv_i)\Big] \right\|^2_{L^2(\rho_{\bY})} \nonumber \\
 &=\frac{1}{N^3}\sum_{i=1}^{N} \left( \left(\sum_{j=k=1}^N + \sum_{j\neq k=1}^N\right) C_{i,j,k}^{E}+C_{i,j,k}^{A}+D_{i,j,k}\right)    \nonumber \\
&=\frac{N-1}{N^2}( \|\intkernelvar^{E} \|_{L^2(\tilde\rho_r^{E})}^2+ \|\intkernelvar^{A}\|_{L^2(\tilde\rho_r^{A})}^2) + \mathcal{R}
\end{align}
where
\begin{align*}
C_{i,j,k}^{E}&=   \langle \intkernelvar^{E}(\|\bx_j-\bx_i\|)(\bx_j-\bx_i), \intkernelvar^E(\|\bx_k-\bx_i\|)(\bx_k-\bx_i)\big\rangle_{L^2(\rho_{\bY})}, \\
C_{i,j,k}^{A}&=  \langle \intkernelvar^{A}(\|\bx_j-\bx_i\|)(\bv_j-\bv_i), \intkernelvar^A(\|\bx_k-\bx_i\|)(\bv_k-\bv_i)\big\rangle_{L^2(\rho_{\bY})},\\
D_{i,j,k}&=   \langle \intkernelvar^{E}(\|\bx_j-\bx_i)\|(\bx_j-\bx_i), \intkernelvar^A(\|\bx_k-\bx_i\|)(\bv_k-\bv_i)\big\rangle_{L^2(\rho_{\bY})}\\&\quad\quad+ \langle \intkernelvar^{A}(\|\bx_{j}-\bx_{i}\|)(\bv_j-\bv_i), \intkernelvar^E(\|\bx_k-\bx_i\|)(\bx_k-\bx_i)\big\rangle_{L^2(\rho_{\bY})}=0,\\
 \mathcal{R} &= \frac{1}{N^3} \sum_{i=1}^N\sum_{j\neq k, j\neq i, k\neq i} (C_{ijk}^{A}+C_{ijk}^{E}).
 \end{align*}
 
By the property of $\rho_{\bY}$, when $i,j,k$ are distinct, we have 
 \begin{align*}
 C_{ijk}^{E}&=\mathbb{E}\big[ \intkernelvar^{E}(\|X_1-X_2\|) \intkernelvar^{E}(\|X_1-X_3\|) \left\langle X_2-X_1,  X_3-X_1 \right\rangle\big]\\
C_{ijk}^{A}&=\mathbb{E}\big[ \intkernelvar^{A}(\|X_1-X_2\|) \intkernelvar^{A}(\|X_1-X_3\|)\big]\mathbb{E}\big[\left\langle V_2-V_1,  V_3-V_1 \right\rangle\big],
\end{align*} for all $(i,j,k)$, where $X_i$s and $V_i$s are identical copies of the position and velocity variables $\bx_i,\bv_i$s.  From the \cref{lemma1B} below,
\begin{align*} 
C_{ijk}^{E}\geq 0, C_{ijk}^{A}\geq 0
\end{align*}
 and we used the fact
$$\mathbb{E}\big[\left\langle V_2-V_1,  V_3-V_1 \right\rangle\big]=\mathbb{E}(\|V_1\|^2)-\|\mathbb{E}(V_1)\|^2\geq 0$$  
Therefore,  
\begin{align*}
 \|\rhsfo_{\bintkernelvar}\|^2_{L^2(\rho_{\bY})} \geq   \frac{N-1}{N^2}( \|\intkernelvar^{E} \|_{L^2(\tilde\rho_r^{E})}^2+ \|\intkernelvar^{A}\|_{L^2(\tilde\rho_r^{A})}^2)
\end{align*} 

 \end{proof}

 The proof of \cref{2ndordersingle:coercivity2} uses the following lemma. 

\begin{lemma}\label{lemma1B}If $X,Y,Z$ are i.i.d random vectors, then for any measurable function $g$ on $\mathbb{R}^d$, we have that 
\[\mathbb{E}[g(X-Y)g(X-Z)\inp{X-Y,X-Z}]\geq 0,\]
\[\mathbb{E}[g(X-Y)g(X-Z)] \geq 0,\]
provided the expectation exists.
\end{lemma}
\begin{proof} Without loss of generality, suppose the probability density function of $X$ is $p(x)$. (The discrete distribution case follows from the same argument). Let $(U,V)=(X-Y,X-Z)$. By the independence of $X,Y,Z$, the pdf of $(U,V)$ is 

\[p(u,v)=\int p(x)p(x-u)p(x-v)dx.\]
Since
\[\sum_{i=1}^N\sum_{j=1}^N c_i\bar{c}_jp(u_i,u_j)=\int p(x)|\sum_{i=1}^N c_ip(x-u_i)|^2dx\ge0,\]
which means $p(u,v)$ is positive definite (p.d.) As $\inp{u,v}$ is p.d and $g(u)g(v)$ is p.d. \cite{li2021identifiability},  we get $g(u)g(v)\inp{u,v}p(u,v)$ is p.d.. Note that
\begin{align}
\mathbb{E}[g(X-Y)g(X-Z)\inp{X-Y,X-Z}]=\int_{\mathbb{R}^{2d}} g(u)g(v)\inp{u,v} p(u,v) du dv,
\end{align}
if the function $g(u)g(v)\inp{u,v}p(u,v)$ is measurable and integrable. Then the inequality holds by p.d. property. Similarly, one can prove the second inequality. 
\end{proof}

\section{Proof of Representer Theorem} \label{secA2}

We prove \JF{}{the Representer Theorem (\cref{representerthm} in main text \cref{subsec:representerthm})} by using an operator-theoretic approach.  

\begin{proposition}\label{eoperator}  Given the empirical noisy trajectory data $(\bbY_M,\bbZ_{\sigma^2,M})=\{\bbX_M,\bbV_M,\bbZ_{\sigma^2,M}\}$. We define the sampling operator
$A_{M}: \mHe\times\mHa \rightarrow \mathbb{R}^{dNML}$ by
\begin{align}\label{finiterank}
A_{M}\bintkernelvar=\rhsfo_{\bintkernelvar}(\bbX_M)&:=\mathrm{Vec}(\{\rhsfo_{\bintkernelvar}(\bY^{(m,l)})\}_{m,l=1}^{M,L}) =\mathrm{Vec}(\{\rhsfo_{\intkernelvare}(\bY^{(m,l)}) + \rhsfo_{\intkernelvara}(\bY^{(m,l)})\}_{m,l=1}^{M,L}),
\end{align} where $\mathbb{R}^{dNML}$ is equipped with the inner product defined in \eqref{winnerp}. 

\begin{itemize}
 \item [1.]The adjoint operator  $A_{M}^*$ is a finite rank operator. For any noise vector $\mathbb{W}$ in  $\mathbb{R}^{dNML}$, let $\mathbb{W}_{m,l,i} \in \mathbb{R}^d$ denote the $i$-th component of $(m,l)$th block of $\mathbb{W}$ as the same way in $\mathbb{Y}_{M}$, then we have 
\begin{align}
A^*_{M}\mathbb{W}=&\bigg( \frac{1}{LM}\sum_{l,m=1}^{L,M}\sum_{i=1,i'\neq i}^{N}\frac{1}{N^2}K_{r_{ii'}^{\bX^{(m,l)}}}^E \langle \mbf{r}_{ii'}^{\bX^{(m,l)}}, \mathbb{W}_{m,l,i}\rangle,\notag\\
& \qquad \qquad \frac{1}{LM}\sum_{l,m=1}^{L,M}\sum_{i=1,i'\neq i}^{N}\frac{1}{N^2}K_{r_{ii'}^{\bX^{(m,l)}}}^A \langle \mbf{r}_{ii'}^{\bV^{(m,l)}}, \mathbb{W}_{m,l,i}\rangle\bigg).
\end{align}
For any function $\bintkernelvar \in \mHe\times\mHa$, we have that 
\begin{align}
B_{M}\bintkernelvar:=A^*_{M}A_{M}\bintkernelvar=&\bigg(\frac{1}{LM}\sum_{l,m=1}^{L,M}\sum_{i=1,i', i'' \neq i}^{N}\frac{1}{N^3}K_{r_{ii'}^{\bX^{(m,l)}}}^E (\langle \intkernelvare ,K_{r_{ii''}^{\bX^{(m,l)}}}^E \rangle_{\mHe} \langle \mbf{r}_{ii'}^{\bX^{(m,l)}},\mbf{r}_{ii''}^{\bX^{(m,l)}}\rangle \\
& \hspace{1in} + \langle \intkernelvara ,K_{r_{ii''}^{\bX^{(m,l)}}}^A \rangle_{\mHa} \langle \mbf{r}_{ii'}^{\bX^{(m,l)}},\mbf{r}_{ii''}^{\bV^{(m,l)}}\rangle),\notag\\
&\frac{1}{LM}\sum_{l,m=1}^{L,M}\sum_{i=1,i', i'' \neq i}^{N}\frac{1}{N^3}K_{r_{ii'}^{\bX^{(m,l)}}}^A (\langle \intkernelvara ,K_{r_{ii''}^{\bX^{(m,l)}}}^A \rangle_{\mHa} \langle \mbf{r}_{ii'}^{\bV^{(m,l)}},\mbf{r}_{ii''}^{\bV^{(m,l)}}\rangle \\
& \hspace{1in} + \langle \intkernelvare ,K_{r_{ii''}^{\bX^{(m,l)}}}^E \rangle_{\mHe} \langle \mbf{r}_{ii'}^{\bV^{(m,l)}},\mbf{r}_{ii''}^{\bX^{(m,l)}}\rangle) \bigg).
\end{align}
\item[2.]  If $\mbf{\lambda}=(\lambda^E,\lambda^A)>0$, a unique minimizer  $\phi_{\mHe\times\mHa}^{\mbf{\lambda},M}$ that solves 
$$
\argmin{\bintkernelvar \in \mHe\times\mHa}\mE^{\mbf{\lambda},M}(\bintkernelvar)
:=\|A_M\bintkernelvar -\bbZ_{\sigma^2,M}\|^2+ \|\sqrt{\mbf{\lambda}}\cdot\bintkernelvar\|_{\mHe\times\mHa}^2$$
exists and is given by 
\begin{align}\label{em}
\phi_{\mHe\times\mHa}^{\mbf{\lambda},M}=(B_M+\mbf{\lambda})^{-1}A_{M}^{*}\bbZ_{\sigma^2,M}.
\end{align} where we interpret the map $\mbf{\lambda}(\bintkernel)$ by $\mbf{\lambda} \cdot \bintkernel=(\lambda^E\intkernele, \lambda^A{\intkernela})$.
\end{itemize}
\end{proposition}

\begin{proof} The part 1 of \cref{eoperator} can be derived by using the identity $\langle A_M \bintkernelvar, \mbf{w}\rangle=\langle  \bintkernelvar, A_M^*\mbf{w} \rangle_{\mHe\times\mHa}$.  Part 2 of \cref{eoperator} is straightforward by  solving the normal equation. 
\end{proof}

Now we derive a basis representation formula for the empirical minimizer of  \eqref{em}

\begin{theorem}\label{representerthm2}
If $\mbf{\lambda}>0$, then the minimizer of the regularized empirical risk functional $\mE^{\mbf{\lambda},M}(\cdot)$ has the form
\begin{equation}
  \phi_{\mHe\times\mHa}^{\lambda,M}= (\sum_{r^x \in r_{\bbX_M}} \hat c_{r^x} K_{r^x}^E, \sum_{(r^x,r^v) \in (r_{\bbX_M} \times r_{\bbV_M})} \hat c_{r^v} K_{r^x}^A),
\end{equation}
where $r_{\bbX_M} \in \mathbb{R}^{MLN^2}$ is the set contains all the pair distances in $\bbX_{M}$, i.e. 
\begin{equation}
r_{\bbX_M} = \begin{bmatrix}r_{11}^{(1,1)},\dots,r_{1N}^{(1,1)},\dots,r_{N1}^{(1,1)},\dots,r_{NN}^{(1,1)}, \dots, r_{11}^{(M,L)},\dots,r_{1N}^{(M,L)},\dots, r_{N1}^{(M,L)},\dots,r_{NN}^{(M,L)}\end{bmatrix}^T,
\end{equation}
and  $r_{\bbX_M}\times r_{\bbV_M} \in \mathbb{R}^{MLN^2\times MLN^2}$ is the set contains all the pair distances in $\bbX_{M}$ and their associated pair distances in $\bbV_{M}$.\\
Moreover, we have 
\begin{eqnarray}\label{solution1}
    \hat c_{r^x}= \frac{1}{N}\mbf{r}_{\bbX_M}^T \cdot  (K_{\rhsfo_\bintkernel}(\bbY_M,\bbY_M) + \lambda^E N ML I)^{-1}\bbZ_{\sigma^2,M},\notag\\
    \hat c_{r^v}= \frac{1}{N}\mbf{r}_{\bbV_M}^T \cdot  (K_{\rhsfo_\bintkernel}(\bbY_M,\bbY_M) + \lambda^A N ML I)^{-1}\bbZ_{\sigma^2,M},
\end{eqnarray}
where the block-diagonal matrix $\mbf{r}_{\bbX_M} = \mathrm{diag}(\mbf{r}_{\bX^{(m,l)}}) \in \mathbb{R}^{MLdN \times MLN^2}$ and $\mbf{r}_{\bX^{(m,l)}} \in \mathbb{R}^{dN\times N^2}$ defined by
\begin{equation}
\mbf{r}_{\bX^{(m,l)}} = 
    \begin{bmatrix}
     \mbf{r}_{11}^{(m,l)}, \dots, \mbf{r}_{1N}^{(m,l)} & \mbf{0} & \cdots & \mbf{0}\\
     \mbf{0} & \mbf{r}_{21}^{(m,l)}, \dots, \mbf{r}_{2N}^{(m,l)} & \cdots & \mbf{0}\\
     \vdots & \vdots & \ddots & \vdots\\
    \mbf{ 0} & \mbf{0} & \cdots & \mbf{r}_{N1}^{(m,l)}, \dots, \mbf{r}_{NN}^{(m,l)}
    \end{bmatrix}
    \ ,
\end{equation}
and same for $\mbf{r}_{\bbV_M}$.
\end{theorem}

\begin{proof} Let  $\mathcal{H}_{K^E,M}$ be the subspace of $\mHe$ spanned by the set of functions $\{K_{r}^E: r\in r_{\bbX_M}\}$, and similarly for $\mathcal{H}_{K^A,M}$. By Proposition \cref{eoperator}, we know that $B_M(\mathcal{H}_{K,M}^E \times \mathcal{H}_{K,M}^A) \subset \mathcal{H}_{K,M}^E \times \mathcal{H}_{K,M}^A$. Since $B_M$ is self-adjoint and compact, by the spectral theory of self-adjoint compact operator (see \cite{blank2008hilbert}),  $\mathcal{H}_{K,M}^E \times \mathcal{H}_{K,M}^A$ is also an invariant subspace for the operator $(B_M+\lambda I)^{-1}$.  Then by \eqref{em}, there exists vectors
$\hat c_{r^x}$, $\hat c_{r^v}$ such that 
\begin{equation}\label{brep}
    \phiH^{\lambda,M} = (\sum_{r^x \in r_{\bbX_M}} \hat c_{r^x} K_{r^x}^E, \sum_{(r^x,r^v) \in (r_{\bbX_M}\times r_{\bbX_M})} \hat c_{r^v} K_{r^x}^A).
\end{equation}

Then, multiplying $(B_M+\mbf{\lambda})$ on both sides of \eqref{em} and plugging in \eqref{brep}, we can obtain 
\begin{equation}\label{mtxid}
\begin{cases}
\big(\mbf{r}_{\bbX_M}^T\mbf{r}_{\bbX_M}K^E(r_{\bbX_M}, {r_{\bbX_M}})+\lambda^E N^3ML I \big)\hat c_{r^x} + \mbf{r}_{\bbX_M}^T\mbf{r}_{\bbV_M}K^A(r_{\bbX_M}, {r_{\bbX_M}})\hat c_{r^v}&=N\mbf{r}_{\bbX_M}^T \bbZ_{\sigma^2,M}\notag\\
\big(\mbf{r}_{\bbV_M}^T\mbf{r}_{\bbV_M}K^A(r_{\bbX_M}, {r_{\bbX_M}})+\lambda^A N^3ML I \big)\hat c_{r^v} + \mbf{r}_{\bbV_M}^T\mbf{r}_{\bbX_M}K^E(r_{\bbX_M}, {r_{\bbX_M}})\hat c_{r^x}&=N\mbf{r}_{\bbV_M}^T \bbZ_{\sigma^2,M}
\end{cases}
\end{equation}
using the matrix representation of $(B_M+\mbf{\lambda} )$ with respect to the spanning sets $\{K_{r}^E: r\in r_{\bbX_M}\}$ and $\{K_{r}^A: r\in r_{\bbX_M}\}$.

Recall that we have $K^E(r_{\bbX_M}, {r_{\bbX_M}})=(K^E(r_{ij}, r_{i'j'}))_{r_{ij},r_{i'j'} \in r_{\bbX_M}}$, $K^A(r_{\bbX_M}, {r_{\bbX_M}})=$\\$(K^A(r_{ij}, r_{i'j'}))_{r_{ij},r_{i'j'} \in r_{\bbX_M}}$ and $K_{\rhsfo_\bintkernel}(\bbY_M,\bbY_M)=\mathrm{Cov}(\rhsfo_{\bintkernel}(\bbY_M), 
\rhsfo_{\bintkernel}(\bbY_M))$, so using the identity 
\begin{align}\label{id}
\mbf{r}_{\bbX_M}K^E(r_{\bbX_M}, {r_{\bbX_M}}) \mbf{r}_{\bbX_M}^T + \mbf{r}_{\bbV_M}K^A(r_{\bbX_M}, {r_{\bbX_M}}) \mbf{r}_{\bbV_M}^T=N^2K_{\rhsfo_\bintkernel}(\bbY_M,\bbY_M)
\end{align}
and the fact that the matrices $\big(\mbf{r}_{\bbX_M}^T\mbf{r}_{\bbX_M}K^E(r_{\bbX_M}, {r_{\bbX_M}})+\lambda^E N^3ML I \big)$, $\big(\mbf{r}_{\bbV_M}^T\mbf{r}_{\bbV_M}K^A(r_{\bbX_M}, {r_{\bbX_M}})+\lambda^A N^3ML I \big)$ are invertible, one can verify that 
\begin{eqnarray}
\begin{cases}
    \hat c_{r^x}&= \frac{1}{N}\mbf{r}_{\bbX_M}^T \cdot  (K_{\rhsfo_\bintkernel}(\bbY_M,\bbY_M) + \lambda^E N ML I)^{-1}\bbZ_{\sigma^2,M},\\
    \hat c_{r^v}&= \frac{1}{N}\mbf{r}_{\bbV_M}^T \cdot  (K_{\rhsfo_\bintkernel}(\bbY_M,\bbY_M) + \lambda^A N ML I)^{-1}\bbZ_{\sigma^2,M},
\end{cases}
\end{eqnarray}
is the solution. 

\end{proof}

Now we are ready to finish the proof of the Representer theorem. 

\begin{proof} 
Let $\tilde K^E=\frac{\sigma^2 K^E}{MNL\lambda^E}$, $\tilde K^A=\frac{\sigma^2 K^A}{MNL\lambda^A}$.

  Since $\intkernele \sim \mathcal{GP}(0,\tilde K^E)$, $\intkernela \sim \mathcal{GP}(0,\tilde K^A)$, the posterior mean in $\eqref{eq:estimated phi} $ will then become
\begin{align*}
\bar{\intkernel}_M^E(r^{\ast})&= \tilde K_{\intkernele,\rhsfo_{\bintkernel}}(r^\ast,\bbX_M)(\tilde K_{\rhsfo_\bintkernel}(\bbY_M,\bbY_M) + \sigma^2I)^{-1}\bbZ_{\sigma^2,M}\\
&=\frac{1}{N}\tilde K_{r_{\bbX_M}^T}^E(r^{\ast})\mbf{r}_{\bbX_M}^T(\tilde K_{\rhsfo_\bintkernel}(\bbY_M,\bbY_M) + \sigma^2I)^{-1}\bbZ_{\sigma^2,M}\\
&=\frac{1}{N} K_{r_{\bbX_M}^T}^E(r^{\ast})\mbf{r}_{\bbX_M}^T( K_{\rhsfo_\bintkernel}(\bbY_M,\bbY_M) + NML\lambda^E I)^{-1}\bbZ_{\sigma^2,M}\\
&= K_{\intkernele,\rhsfo_{\bintkernel}}(r^\ast,\bbX_M)( K_{\rhsfo_\bintkernel}(\bbY_M,\bbY_M) + NML\lambda^E I)^{-1}\bbZ_{\sigma^2,M}\\
&= \sum_{r \in r_{\bbX_M}} \hat c_r K_{r}^E, 
\end{align*} 
where $\hat c_r$ is defined in \eqref{solution1} and  we used  the identity $ K_{\intkernele,\rhsfo_{\bintkernel}}(r^\ast,\bbX_M)=\frac{1}{N} K_{r_{\bbX_M}^T}^E(r^{\ast})\mbf{r}_{\bbX_M}^T$ (also for $\tilde K$) in the proof. Similarly, we can get the posterior mean for $\bar{\intkernel}_M^A(r^{\ast})$ .
\end{proof}

\begin{lemma}\label{lemma: conditioning Gaussian}
Let $\bx$ and $\by$ be jointly Gaussian random vectors
 \begin{equation}
 \begin{bmatrix}
 \bx\\ \by
 \end{bmatrix}
 \sim \mathcal{N} (
 \begin{bmatrix}
 \mu_{\bx}\\ \mu_{\by}
 \end{bmatrix}
 , 
 \begin{bmatrix}
 A & C\\
 C^T & B
 \end{bmatrix}
 ),
\end{equation}
then the marginal distribution of $\bx$ and the conditional distribution of $\bx$ given $\by$ are
\begin{equation}
    \bx \sim \mathcal{N}(\mu_{\bx},A), \quad \textrm{and } \bx|\by \sim \mathcal{N}(\mu_{\bx} + CB^{-1}(\by - \mu_{\by}), A - CB^{-1}C^T).
\end{equation}
\end{lemma}
\begin{proof}
See, e.g. \cite{williams2006gaussian}, Appendix A.
\end{proof}

\end{appendices}

\bibliographystyle{unsrt}
\bibliography{references}

\end{document}